%% file: main.tex
\documentclass{article} 
\usepackage{iclr2024_conference_preprint,times}

\usepackage{amsmath}
\usepackage{amssymb}
\usepackage{mathtools}
\usepackage{thmtools,thm-restate}
\usepackage{amsthm}
\usepackage{enumerate}
\usepackage{algpseudocode}
\usepackage{wrapfig}
\usepackage{algorithm}
\usepackage{comment}
\usepackage{multirow}
\usepackage{cellspace}
\usepackage{makecell}
\usepackage{lipsum}  
\usepackage{stmaryrd} 
\usepackage{booktabs}

\theoremstyle{plain}
\newtheorem{theorem}{Theorem}[section]

\newtheorem{lemma}[theorem]{Lemma}

\theoremstyle{definition}

\theoremstyle{remark}

\makeatletter
\newtheorem*{rep@theorem}{\rep@title}
\newcommand{\newreptheorem}[2]{%
\newenvironment{rep#1}[1]{%
 \def\rep@title{\textbf{#2} \ref{##1}}%
 \begin{rep@theorem}}%
 {\end{rep@theorem}}}
\makeatother

\newreptheorem{theorem}{Theorem}
\newreptheorem{proposition}{Proposition}
\newreptheorem{lemma}{Lemma}
\newreptheorem{corollary}{Corollary}

\input{math_commands}

\newcommand{\odestep}{\text{step}}
\newcommand{\aux}{\text{\tiny aux}}
\newcommand{\nextt}{{\tiny \text{next}}}
\usepackage{xcolor}
\definecolor{linkcolor}{RGB}{74, 102, 146}
\usepackage[colorlinks=true,allcolors=linkcolor,pageanchor=true,plainpages=false,pdfpagelabels,bookmarks,bookmarksnumbered]{hyperref}

\usepackage{url}
\usepackage{multirow}

\newcommand{\juan}[1]{\textcolor{violet}{[\textbf{Juan:}] #1} }

\newcommand\cincludegraphics[2][]{\raisebox{-0.4\height}{\includegraphics[#1]{#2}}}

\title{Bespoke Solvers for Generative Flow Models}

\author{N. Shaul$^1$\ \ J. Perez$^2$\ \ R. T. Q. Chen$^3$\ \ A. Thabet$^2$\ \ A. Pumarola$^2$\ \ Y. Lipman$^{1,3}$ \\
    $^1$Weizmann Institute of Science  \ \
    $^2$GenAI, Meta \ \ 
    $^3$FAIR, Meta \\
}

\iclrfinalcopy 
\begin{document}

\maketitle

\begin{abstract}
Diffusion or flow-based models are powerful generative paradigms that are notoriously hard to sample as samples are defined as solutions to high-dimensional Ordinary or Stochastic Differential Equations (ODEs/SDEs) which require a large Number of Function Evaluations (NFE) to approximate well. Existing methods to alleviate the costly sampling process include model distillation and designing dedicated ODE solvers. However, distillation is costly to train and sometimes can deteriorate quality, while dedicated solvers still require relatively large NFE to produce high quality samples. In this paper we introduce \emph{``Bespoke solvers''}, a novel framework for constructing custom ODE solvers tailored to the ODE of a given pre-trained flow model. Our approach optimizes an order consistent and parameter-efficient solver (\eg, with 80 learnable parameters), is trained for roughly 1\% of the GPU time required for training the pre-trained model, and significantly improves approximation and generation quality compared to dedicated solvers. For example, a Bespoke solver for a CIFAR10 model produces samples with Fréchet Inception Distance (FID) of 2.73 with 10 NFE, and gets to 1\% of the Ground Truth (GT) FID (2.59) for this model with only 20 NFE. On the more challenging ImageNet-64$\times$64, Bespoke samples at 2.2 FID with 10 NFE, and gets within 2\% of GT FID (1.71) with 20 NFE.\vspace{-5pt}
\end{abstract}

\section{Introduction}\vspace{-5pt}
Diffusion models \citep{sohl2015deep,ho2020denoising}, and more generally flow-based models \citep{song2020score,lipman2022flow,albergo2022si}, have become prominent in generation of images \citep{dhariwal2021diffusion,rombach2021highresolution}, audio \citep{kong2020diffwave,le2023voicebox}, and molecules \citep{kong2020diffwave}. While training flow models is relatively scalable and efficient, sampling from a flow-based model entails solving a Stochastic or Ordinary Differential Equation (SDE/ODE) in high dimensions, tracing a velocity field defined with the trained neural network. Using off-the-shelf solvers to approximate the solution of this ODE to a high precision requires a large Number~(\ie, 100s) of Function Evaluations~(NFE), making sampling one of the main standing challenges in flow models. Improving the sampling complexity of flow models, without degrading sample quality, will open up new applications that require fast sampling, and will help reducing the carbon footprint and deployment cost of these models.

Current approaches for efficient sampling of flow models divide into two main groups: (i) \emph{Distillation}: where the pre-trained model is fine-tuned to predict either the final sampling \citep{luhman2021knowledge} or some intermediate solution steps \citep{salimans2022progressive} of the ODE. Distillation does not guarantee sampling from the pre-trained model's distribution, but, when given access to the training data during distillation training, it is shown to empirically generate samples of comparable quality to the original model \citep{salimans2022progressive,meng2023distillation}. Unfortunately, the GPU time required to distill a model is comparable to the training time of the original model \cite{salimans2022progressive}, which is often considerable. (ii) \emph{Dedicated solvers}: where the specific structure of the ODE is used to design a more efficient solver \citep{song2020denoising,lu2022dpm,lu2022dpm-pp} and/or employ a suitable solver family from the literature of numerical analysis \citep{zhang2022fast,zhang2023improved}. The main benefit of this approach is two-fold: First, it is \emph{consistent}, \ie, as the number of steps (NFE) increases, the samples converge to those of the pre-trained model. Second, it does not require further training/fine-tuning of the pre-trained model, consequently avoiding long additional training times and access to training data. Related to our approach, some works have tried to learn an ODE solver within a certain class \citep{watson2021learning,duan2023optimal}; however, they do not guarantee consistency and usually introduce moderate improvements over generic dedicated solvers. 


\begin{wrapfigure}[20]{r}{0.3\textwidth}
  \begin{center}
    \includegraphics[width=0.29\textwidth]{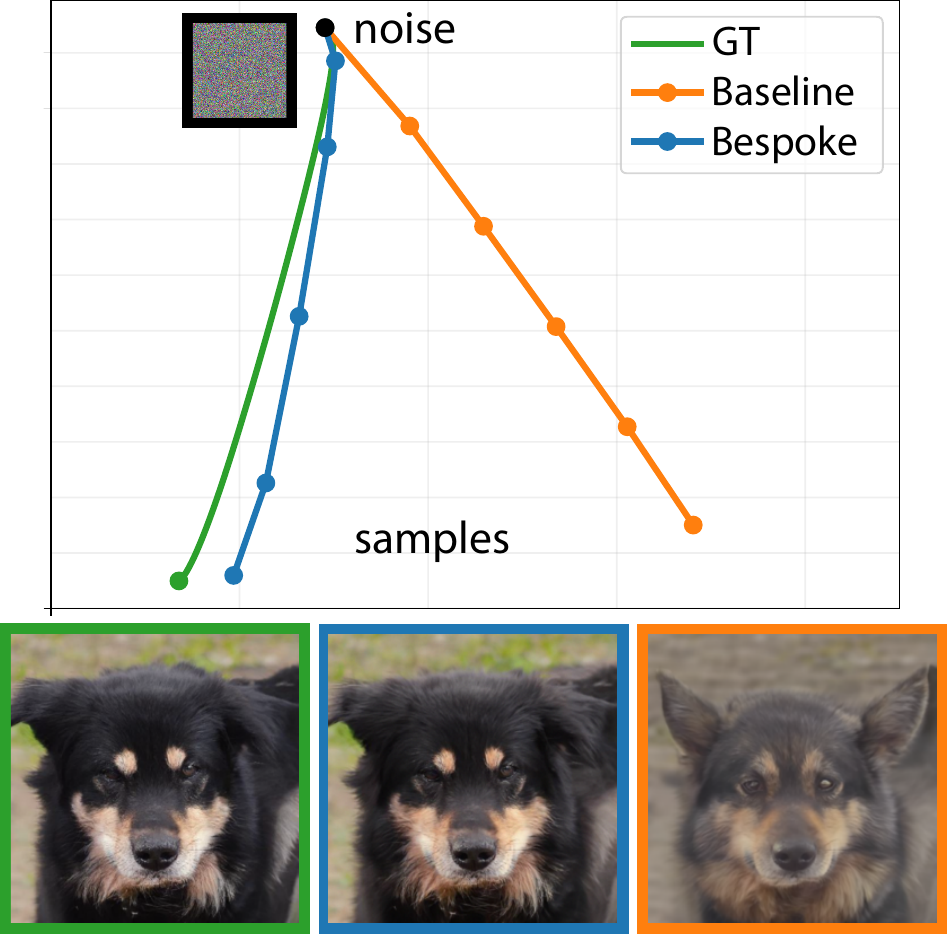}
  \end{center}\vspace{-10pt}
  \caption{Using 10 NFE to sample using our Bespoke solver improves fidelity w.r.t.~the baseline (RK2) solver. Visualization of paths was done with the 2D PCA plane approximating the noise and end sample points. }\label{fig:3_paths}
\end{wrapfigure}
In this paper, we introduce \emph{Bespoke solvers}, a framework for learning consistent ODE solvers \emph{custom-tailored} to pre-trained flow models. 
The main motivation for Bespoke solvers is that different models exhibit sampling paths with different characteristics, leading to local truncation errors that are specific to each instance of a trained model. 
A key observation of this paper is that optimizing a solver for a particular model can significantly improve quality of samples for low NFE compared to existing dedicated solvers. Furthermore, Bespoke solvers use a very small number of learnable parameters and consequently are efficient to train. 
For example, we have trained $n\in\set{5,8,10}$ steps Bespoke solvers for a pre-trained ImageNet-64$\times$64 flow model with $\set{40,64,80}$ learnable parameters (resp.)~producing images with Fréchet Inception Distances (FID) of $2.2$, $1.79$, $1.71$ (resp.), where the latter is within 2\% from the Ground Truth (GT) FID (1.68) computed with $\sim180$ NFE. The Bespoke solvers were trained (using a rather naive implementation) for roughly 1\% of the GPU time required for training the original model. Figure \ref{fig:3_paths} compares sampling at 10 NFE from a pre-trained AFHQ-256$\times$256 flow model with order 2 Runge-Kutta (RK2) and its Bespoke version (RK2-Bes), along with the GT sample that requires $\sim180$ NFE.
%
Our work brings the following contributions:
\begin{enumerate}
    \item A differentiable parametric family of consistent ODE solvers. 
    \item A tractable loss that bounds the global truncation error while allowing parallel computation. 
    \item An algorithm for training a Bespoke $n$-step solver for a specific pre-trained model. 
    \item Significant improvement over dedicated solvers in generation quality for low NFE. \vspace{-5pt}
\end{enumerate}

\section{Bespoke solvers}\vspace{-5pt}
We consider a pre-trained flow model taking some prior distribution~(noise) $p$ to a target~(data) distribution $q$ in data space $\Real^d$. The flow model~\citep{chen2018neural} is represented by a time-dependent Vector Field (VF) $u:[0,1]\times\Real^d\too \Real^d$ that transforms a noise sample $x_0\sim p(x_0)$ to a data sample $x_1\sim q(x_1)$ by solving the ODE 
\begin{equation}\label{e:ode}
    \dot{x}(t) = u_t(x(t)),
\end{equation}
with the initial condition $x(0)=x_0\sim p(x_0)$, from time $t=0$ until time $t=1$, and $\dot{x}(t):=\tfrac{d}{dt}x(t)$. The solution at time $t=1$, \ie, $x(1)\sim q(x(1))$, is the generated target sample. 


\input{tables/ode_solver}
\textbf{Numerical ODE solvers.} Solving \eqref{e:ode} is done in practice with numerical ODE solvers. A numerical solver is defined by an update step:
\begin{equation}\label{e:step}    (t_{\nextt},x_{\nextt}) = \odestep(t,x;u_t).
\end{equation}
The update step takes as input current time $t$ and approximate solution $x$, and outputs the next time step $t_\nextt$ and the corresponding approximation $x_\nextt$ to the true solution $x(t_\nextt)$ at time $t_{\nextt}$. To approximate the solution at some desired end time, \ie, $t=1$, one first initializes the solution at $t=0$ and repeatedly applies the update step in \eqref{e:step} $n$ times, as presented in Algorithm \ref{alg:odesolve}. 
The $\odestep$ is designed so that $t_n=1$. 

An ODE solver ($\odestep$) is said to be of \emph{order} $k$ if its local truncation error is 
\begin{equation}
    \norm{x(t_{\nextt}) - x_{\nextt}} = O\parr{(t_{\nextt}-t)^{k+1}}, 
\end{equation}
asymptotically as $t_{\nextt}\too t$, where $t\in [0,1)$ is arbitrary but fixed and $t_\nextt, x_\nextt$ are defined by the solver, \eqref{e:step}. A popular family of solvers that offers a wide range of orders is the Runge-Kutta (RK) family~\citep{iserles2009first}. Two of the most popular members of the RK family are (set $h=n^{-1}$): 
\begin{align}\label{e:euler}
    &\text{RK1 (Euler - order 1):} \qquad \ \odestep(t,x; u_t) = \parr{t+h \, ,\,  x + h u_{t}(x)}, \\ \label{e:midpoint}
    &\text{RK2 (Midpoint - order 2):} \ \ \  \odestep(t,x; u_t) = \parr{t+h\,  ,\,  x + h u_{t+\frac{h}{2}}\parr{x + \frac{h}{2}u_{t}(x)}}.
\end{align}

\textbf{Approach outline.} Given a pre-trained $u_t$ and a target number of time steps $n$ our goal is to find a custom (Bespoke) solver that is optimal for approximating the samples $x(1)$ defined via \eqref{e:ode} from initial conditions sampled according to $x(0)=x_0\sim p(x_0)$. To that end we develop two components: (i) a differentiable parametric family of solvers $\odestep^\theta$, with parameters $\theta\in \Real^p$ (where $p$ is \textit{very} small); and (ii) a tractable loss bounding the \emph{global truncation error}, \ie, the Root Mean Square Error~(RMSE) between the approximate sample ${x}^\theta_n$ and the GT sample $x(1)$, 
\begin{equation}\label{e:rmse}
    \text{Global truncation error:} \quad \gL_{\text{RMSE}}(\theta) = \E_{x_0\sim p(x_0)}\norm{x(1)-{x}_n^{\theta}},
\end{equation} 
where ${x}_n^\theta$ is the output of Algorithm  \ref{alg:odesolve} using the candidate solver $\odestep^\theta$ and $\norm{x}=\sqrt{\frac{1}{d}\sum_{j=1}^d [x^{(j)}]^2}$.

\subsection{Parametric family of ODE solvers through transformed sampling paths}

Our strategy for defining the parametric family of solvers $\odestep^\theta$ is using a generic base ODE solver, such as RK1 or RK2, applied to a parametric family of \textit{transformed paths}. 

\begin{wrapfigure}{r}{0.33\textwidth}
\centering
\vspace{-10pt}
\includegraphics[width=0.3\textwidth]{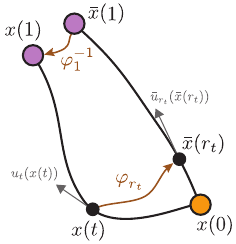}
\caption{Transformed paths.}
\label{fig:paths}
\end{wrapfigure}
\textbf{Transformed sampling paths.}\label{ss:transformed_paths} We transform the sample trajectories $x(t)$ by applying two components: a time reparametrization and an arbitrary invertible transformation. That is,
\begin{equation}\label{e:bar_x_r_from_x_t}
    \bar{x}(r) = \varphi_r(x(t_r)),\quad r\in [0,1],
\end{equation}
where $t_r,\varphi_r(x)$ are arbitrary functions in a family $\gF$ defined by the following conditions: (i) \emph{Smoothness}: $t_r:[0,1]\too[0,1]$ is a diffeomorphism\footnote{A diffeomorphism is a $C^1$ continuously differentiable function with a $C^1$ continuous differentiable inverse.}, and $\varphi:[0,1]\times \Real^d \too \Real^d$ is $C^1$ and a diffeomorphism in $x$. We also assume $r_t$ and $\varphi_r^{-1}$ are Lipschitz continuous with a constant $L>0$. (ii) \emph{Boundary conditions}: $t_r$ satisfies $t_0=0$ and $t_1=1$, and $\varphi_0(\cdot)$ is the identity function, \ie, $\varphi_0(x)=x$ for all $x\in \Real^d$. Figure \ref{fig:paths} depicts a transformation of a path, $x(t)$. Note that $\bar{x}(0)=x(0)$, however the end point $\bar{x}(1)$ does not have to coincide with $x(1)$. Furthermore, as $t_r:[0,1]\too [0,1]$ is a diffeomorphism, 
$t_r$ is strictly monotonically increasing. 

The motivation behind the definition of the transformed trajectories is that it allows reconstructing $x(t)$ from $\bar{x}(r)$. Indeed, denoting $r=r_t$ the inverse function of $t=t_r$ we have 
\begin{equation}\label{e:x_t_from_bar_x_r}
x(t) = \varphi^{-1}_{r_t}(\bar{x}({r_t})). 
\end{equation}
Our hope is to find a transformation that simplifies sampling paths and allows the base solver to provide better approximations of the GT samples. The transformed trajectory $\bar{x}_r$ is defined by a VF $\bar{u}_r(x)$ that can be given an explicit form as follows (proof in Appendix \ref{a:transformed_paths}):
\begin{restatable}{proposition}{barut}\label{prop:bar_u_t}
    Let $x(t)$ be a solution to \eqref{e:ode}. Denote $\dot{\varphi}_r := \tfrac{d}{dr}\varphi_r$ and $\dot{t}_r := \tfrac{d}{dr}t_r$. Then $\bar{x}(r)$ defined in \eqref{e:bar_x_r_from_x_t} is a solution to the ODE (\eqref{e:ode}) with the VF 
    \begin{equation}\label{e:bar_u_t}
        \bar{u}_r(x) = \dot{\varphi}_r(\varphi^{-1}_r(x)) + \dot{t}_r\partial_x \varphi_r(\varphi^{-1}_r(x))u_{t_r}(\varphi^{-1}_r(x)).
    \end{equation}
\end{restatable}
\textbf{Solvers via transformed paths.   }
We are now ready to define our parametric family of solvers $\odestep^\theta(t,x;u_t)$: First we transform the input sample $(t,x)$ according to \eqref{e:bar_x_r_from_x_t} to 
\begin{equation}\label{e:step_parametric_1}
    (r,\bar{x})=(r_t, \varphi_{r_t}(x)).
\end{equation}
Next, we perform a step with the base solver of choice, denoted here by $\odestep$, \eg, RK1 or RK2, 
\begin{equation}\label{e:step_parametric_2}
    (r_{\nextt},\bar{x}_{\nextt}) = \odestep(r, \bar{x}; \bar{u}_r),
\end{equation}
and lastly, transform back using \eqref{e:x_t_from_bar_x_r} to define the parametric solver $\odestep^\theta$ via
\begin{equation}\label{e:step_parametric}
    (t_\nextt, x_\nextt) = \odestep^\theta(x,t;u_t) =  \parr{t_{r_{\nextt}}, \varphi_{r_{\nextt}}^{-1}  (\bar{x}_{\nextt})}.
\end{equation}
The parameters $\theta$ denote the parameterized transformations $t_r$ and $\varphi_r$ satisfying the properties of $\gF$ and the choice of a base solver $\odestep$. In Section \ref{sec:use_cases} we derive the explicit rules we use in this paper. 

\textbf{Consistency of solvers.} \label{ss:consistency} An important property of the parametric solver $\odestep^\theta$ is \emph{consistency}. Namely, due to the properties of $\gF$, regardless of the particular choice of $t_r,\varphi_r\in \gF$, the solver $\odestep^\theta$ has the same local truncation error as the base solver. 
\begin{restatable}{theorem}{consistency}(Consistency of parametric solvers)\label{thm:consistency}
    Given arbitrary $t_r,\varphi_r$ in the family of functions $\gF$ and a base ODE solver of order $k$, the corresponding ODE solver $\odestep^\theta$ is also of order $k$, \ie, \vspace{-5pt}
    \begin{equation}
        \norm{x(t_\nextt)  - x_{\nextt}} = O((t_\nextt-t)^{k+1}).\vspace{-5pt}
    \end{equation}
\end{restatable}
The proof is provided in Appendix \ref{a:consistency}.
%
Therefore, as long as $t_r,\varphi_r(x)$ are in $\gF$, decreasing the base solver's step size $h\too 0$ will result in our approximated sample $x_n^\theta$ converging to the exact sample $x(1)$ of the trained model in the limit, \ie, $x^\theta_n \too x(1)$ as $n\too \infty$. 



\vspace{-5pt}
\subsection{Two use cases}\label{ss:two_use_cases}\vspace{-5pt}\label{sec:use_cases}
We instantiate the Bespoke solver framework for two cases of interest  (a full derivation is in Appendix \ref{a:parametric_solver_step}), and later prove that our choice of transformations in fact covers all ``noise scheduler'' configurations used in the standard diffusion model literature.
In our use cases, we consider a time-dependent scaling as our invertible transformation $\varphi_r$,
\begin{equation}\label{e:scale}
    \varphi_r(x)=s_r x, \text{ and its inverse }\varphi_r^{-1}(x) =  x/s_r,
\end{equation}
where $s:[0,1]\too \Real_{>0}$ is a strictly positive $C^1$ scaling function such that $s_0=1$ (\ie, satisfying the boundary condition of $\varphi$). The transformation of trajectories, \ie, equations \ref{e:bar_x_r_from_x_t} and \ref{e:x_t_from_bar_x_r}, take the form
\begin{equation}\label{e:scale_time}
    \bar{x}(r)=s_r x({t_r}),\text{ and } x(t) = \bar{x}(r_t)/{s_{t_r}}, 
\end{equation}
and we name this transformation: \emph{scale-time}. The transformed VF $\bar{u}_r$ (\eqref{e:bar_u_t}) is thus \vspace{-5pt}
\begin{equation}\label{e:bar_u_t_scale}
    \bar{u}_r(x) = \frac{\dot{s}_r}{s_r}x + \dot{t}_r s_r u_{t_r}\parr{\frac{x}{s_r}}.\vspace{-5pt}
\end{equation}

\textbf{Use case I: RK1-Bespoke.} We consider RK1 (Euler) method (\eqref{e:euler}) as the base solver $\odestep$ and denote $r_i=ih$, $i\in [n]$, where $[n]=\set{0,1,\ldots,n}$ and $h=n^{-1}$. Substituting \eqref{e:euler} in  \eqref{e:step_parametric_2}, we get from \eqref{e:step_parametric} that
\begin{equation}\label{e:step_theta_euler}
    \odestep^\theta(t_i,x_i;u_t) \vcentcolon= \parr{t_{i+1} , \frac{s_i+h\dot{s}_i}{s_{i+1}}x_i + h \dot{t}_i\frac{s_i}{s_{i+1}}u_{t_i}(x_i)  },
\end{equation}
where we denote $t_i=t_{r_i}$, $\dot{t}_i=\frac{d}{dr}\vert_{r=r_i}t_r$, $s_i=s_{r_i}$, $\dot{s}_i=\frac{d}{dr}\vert_{r=r_i}s_r$, and $i\in [n-1]$. The learnable parameters $\theta\in\Real^p$ and their constraints are derived from the fact that the functions $t_r,\varphi_r$ are members of $\gF$. There are $p=4n-1$ parameters in total: $\theta = (\theta^t,\theta^s)$, where  
\begin{equation}    \label{e:theta_euler}
\theta^t:\begin{cases} 0=t_0 < t_1 < \cdots < t_{n-1} < t_n=1 \\
 \dot{t}_0, \ldots, \dot{t}_{n-1} > 0
\end{cases}, \quad
\theta^s:\begin{cases}
    s_1, \ldots, s_n > 0 \ \ , \ \ s_0=1\\
    \dot{s}_0,\ldots,\dot{s}_{n-1} 
\end{cases}.
\end{equation}
Note that we ignore the Lipschitz constant constraints in $\gF$ when deriving the constraints for $\theta$. 

\textbf{Use case II: RK2-Bespoke.} 
Here we choose the RK2 (Midpoint) method (\eqref{e:midpoint}) as the base solver $\odestep$. Similarly to the above, substituting \eqref{e:midpoint} in \eqref{e:step_parametric_2}, we get
\begin{equation}\label{e:step_theta_midpoint}
    \odestep^\theta(t_i,x_i;u_t) \vcentcolon= \parr{t_{i+1}, \frac{s_i}{s_{i+1}}x_i + \frac{h}{s_{i+1}}\set{\frac{\dot{s}_{i+\frac{1}{2}}}{s_{i+\frac{1}{2}}}z_i + \dot{t}_{i+\frac{1}{2}}s_{i+\frac{1}{2}}u_{t_{i+\frac{1}{2}}}\parr{\frac{z_i}{s_{i+\frac{1}{2}}}}  } },
\end{equation}
where we set $r_{i+\frac{1}{2}} = r_i+\frac{h}{2}$, and accordingly $t_{i+\frac{1}{2}}$, $\dot{t}_{i+\frac{1}{2}}$, $s_{i+\frac{1}{2}}$, and $\dot{s}_{i+\frac{1}{2}}$ are defined, and
\begin{equation}\label{e:z_i}
    z_i = \parr{s_i + \frac{h}{2}\dot{s}_i}x_i + \frac{h}{2}s_i \dot{t}_i u_{t_i}(x_i).
\end{equation}
In this case there are $p=8n-1$ learnable parameters, $\theta=(\theta^t,\theta^s)\in\Real^{p}$, where 
\begin{equation}\label{e:theta_midpoint}    
\theta^t:\begin{cases} 0=t_0 < t_{\frac{1}{2}} < \cdots < t_n=1 \\
 \dot{t}_0, \dot{t}_{\frac{1}{2}}, \ldots, \dot{t}_{n-1}, \dot{t}_{n-\frac{1}{2}} > 0 
\end{cases}, \quad
\theta^s:\begin{cases}
 s_{\frac{1}{2}}, s_1, \ldots, s_n > 0 \ \ , \ \ s_0=1\\
\dot{s}_0,\dot{s}_{\frac{1}{2}},\ldots,\dot{s}_{n-\frac{1}{2}} 
\end{cases}.
\end{equation}

\textbf{Equivalence of scale-time transformations and Gaussian Paths.}\label{ss:equivalence}
We note that our \textit{scale-time} transformation covers \textit{all} possible trajectories used by diffusion and flow models trained with Gaussian distributions.
Denote by $p_t(x)$ the probability density function  of the random variable $x(t)$, where $x(t)$ is defined by a random initial sampling $x(0)=x_0\sim p(x_0)$ and solving the ODE in \eqref{e:ode}. 

When training a Diffusion or Flow Matching models, $p_t$ has the form
$p_t(x) = \int p_t(x|x_1)q(x_1)dx_1$, where $p_t(x|x_1)=\gN(x|\alpha_t x_1, \sigma_t^2 I)$. A pair of functions $\alpha,\sigma:[0,1]\too [0,1]$ satisfying  
\begin{equation}\label{e:def_scheduler}
    \alpha_0=0=\sigma_1,\ \  \alpha_1=1=\sigma_0, \ \text{ and strictly monotonic } \mathrm{snr}(t)=\alpha_t/\sigma_t
\end{equation}
is called a \emph{scheduler}\footnote{We use the convention of noise at time $t=0$ and data at time $t=1$.}. We use the term \emph{Gaussian Paths} for the collection of probability paths $p_t(x)$ achieved by different schedulers. The velocity vector field that generates $p_t(x)$ and results from zero Diffusion/Flow Matching training loss is
\begin{equation}\label{e:u_t}
    u_t(x) = \int u_t(x|x_1) \frac{p_t(x|x_1)q(x_1)}{p_t(x)}dx_1,
\end{equation}
where $u_t(x|x_1) = \frac{\dot{\sigma}_t}{\sigma_t}x + \brac{\dot{\alpha}_t - \dot{\sigma}_t\frac{\alpha_t}{\sigma_t}}x_1$, as derived in \cite{lipman2022flow}. Next, we generalize a result by \cite{kingma2021variational} and \cite{karras2022elucidating} to consider marginal sampling paths $x(t)$ defined by $u_t(x)$, and show that any two such paths are related by a scale-time transformation: 
\begin{restatable}{theorem}{equivalence}(Equivalence of Gaussian Paths and scale-time transformation)\label{thm:equivalence}
Consider a Gaussian Path defined by a scheduler $(\alpha_t,\sigma_t)$, and let $x(t)$ denote the solution of \eqref{e:ode} with $u_t$ defined in \eqref{e:u_t} and initial condition $x(0)=x_0$. Then, \vspace{-3pt}
\begin{enumerate}[(i)]
    \item For every other Gaussian Path defined by a scheduler $(\bar{\alpha}_r,\bar{\sigma}_r)$ with trajectories $\bar{x}(r)$ there exists a scale-time transformation with $s_1=1$ such that $\bar{x}(r)=s_r x(t_r)$.
    \item For every scale-time transformation with $s_1=1$ there exists a Gaussian Path defined by a scheduler $(\bar{\alpha}_r,\bar{\sigma}_r)$ with trajectories $\bar{x}(r)$ such that $s_r x(t_r)=\bar{x}(r)$.\vspace{-3pt}
\end{enumerate}
\end{restatable}
(Proof in Appendix \ref{a:equivalence}.) Assuming an ideal velocity field (\eqref{e:u_t}), \ie, the pre-trained model is optimal, this theorem implies that searching over the scale-time transformations is equivalent to searching over all possible Gaussian Paths. 
Note, that in practice we allow $s_1\ne 1$, expanding beyond the standard space of Gaussian Paths. 
Another interesting consequence of Theorem \ref{thm:equivalence} (simply plug in $t=1$) is that all ideal velocity fields in \eqref{e:u_t} define the \emph{same} coupling, \ie, joint distribution, of noise $x_0$ and data $x_1$.\vspace{-0pt}

\input{tables/bespoke_pseudocode}\vspace{-3pt}
\subsection{RMSE upper bound loss}\label{ss:rmse_bound}
The second component of our framework is a tractable loss that bounds the RMSE loss in \eqref{e:rmse} while enabling parallel computation of the loss over each step of the Bespoke solver. 
To construct the bound, let us fix an initial condition $x_0\sim p(x_0)$ and denote as before $x(1)$ to be the exact solution of the sample path (\eqref{e:ode}). Furthermore, consider a candidate solver $\odestep^\theta$, and denote its $t$ and $x$ coordinate updates by $\odestep^\theta = (\odestep_t^\theta, \odestep_x^\theta)$. Applying Algorithm \ref{alg:odesolve} with $\odestep^\theta$ and initial condition $x_0$ produces a series of approximations $x^\theta_i$, each corresponds to a  time step $t_i$, $i\in [n]$. Lastly, we denote by 
\begin{equation}\label{e:ei_di}
    e^\theta_i = \norm{x(t_i)-x^\theta_i}, \quad d^\theta_i=\norm{x(t_i) - \odestep_x^\theta(t_{i-1},x(t_{i-1}); u_t)}
\end{equation}
the \emph{global} and \emph{local} truncation errors at time $t_i$, respectively. Our goal is to bound the global error at the final time  step $t_n=1$, \ie, $e^\theta_n$. Using the update step definition (\eqref{e:step}) and triangle inequality we can bound 
\begin{align*}
     e^\theta_{i+1}\hspace{-1pt}\leq\hspace{-1pt} \norm{x(t_{i+1})-\odestep^\theta_x(t_{i},x(t_{i});u_t)} \hspace{-2pt}+\hspace{-2pt} \norm{\odestep_x^\theta(t_{i},x(t_{i});u_t) - \odestep^\theta_x(t_{i},x^\theta_{i};u_t)}  \leq d^\theta_{i+1} + L^\theta_{i} e^\theta_{i},
\end{align*}
where $L^\theta_{i}$ is defined to be the Lipschitz constant of the function $\odestep_x^\theta(t_i,\cdot\,;u_t)$. To simplify notation we set by definition $L^\theta_n=1$ (this is possible since $L^\theta_n$ does not actually participate in the bound). Using the above bound $n$ times and noting that $e^\theta_0=0$ we get 
\begin{equation}\label{e:bound}
    e^\theta_n \leq \sum_{i=1}^n M^\theta_i d^\theta_i, \text{ where } M^\theta_i = \prod_{j=i}^n L^\theta_j.
\end{equation}
Motivated by this bound we define our RMSE-Bespoke loss: 
\begin{center}			
    \colorbox{mygray} {		
      \begin{minipage}{0.977\linewidth} 	
       \centering
       \vspace{-8pt}
        \begin{equation}\label{e:loss_bes}
\gL_{\text{bes}}(\theta) = \E_{x_0\sim p(x_0)}  \sum_{i=1}^n M^\theta_i d^\theta_i,   \end{equation}
      \end{minipage}}			
\end{center}
where $d^\theta_i$ is defined in \eqref{e:ei_di} and $M^\theta_i$ defined in \eqref{e:bound}. The constants $L^\theta_i$ depend both on the parameters $\theta$ and the Lipschitz constant $L_u$ of the network $u_t$. As $L_u$ is difficult to estimate, we treat $L_u$ as a hyper-parameter, denoted $L_\tau$ (in all experiments we use $L_\tau=1$), and compute $L^\theta_i$ in terms of $\theta$ and $L_\tau$ for our two Bespoke solvers, RK1 and RK2, in Appendix \ref{a:lipschitz_of_step}. 
Assuming that $L_\tau\geq L_u$, an immediate consequence of the bound in \eqref{e:bound} is that the RMSE-bespoke loss bounds the RMSE loss, \ie, the global truncation error defined in \eqref{e:rmse},
\begin{equation}
    \gL_{\text{RMSE}}(\theta) \leq \gL_{\text{bes}}(\theta).
\end{equation}

%
%
\textbf{Implementation of the RMSE-Bespoke loss.}\label{ss:implementaiton_of_the_rmse_bespoke_loss}
We provide pseudocode for Bespoke training and sampling in Algorithms \ref{alg:bes_training} and \ref{alg:bes_sampling}, respectively. During training, we need to have access to the GT path $x(t)$ at times $t_i$, $i\in [n]$, which we compute with a generic solver. The Bespoke loss is constructed by plugging $\odestep^\theta$ (equations \ref{e:step_theta_euler} or \ref{e:step_theta_midpoint}) into $d_i$ (\eqref{e:ei_di}). The gradient $\nabla_\theta \gL_{\text{bes}}(\theta)$ requires the derivatives $\partial x(t_i) / \partial t_i$. Computing the derivatives of $x(t_i)$ can be done using the ODE it obeys, \ie, $\dot{x}(t_i)=u_{t_i}(x_i)$. Therefore, a simple way to write the loss ensuring correct gradients w.r.t.~$t_i$ is replace $x(t_i)$ with $x^{\aux}_i(t_i)$ where 
\begin{equation}\label{e:x_aux}
   \hspace{-2pt} x^{\aux}_i(t)\hspace{-1pt} = \hspace{-1pt} x(\llbracket t_i \rrbracket) + u_{\llbracket t_{i}\rrbracket}\hspace{-1pt}\left(x(\llbracket t_i \rrbracket)\right)(t - \llbracket t_i \rrbracket),
\end{equation}
where $\llbracket\cdot\rrbracket$ denotes the stop gradient operator; \ie, $x_i^{\aux}(t)$ is linear in $t$ and its value and derivative w.r.t.~$t$ coincide with that of $x(t_i)$ at time $t=t_i$. Full details are provided in Appendix \ref{a:impl_details}. 



\input{sections/previous_work}

\input{sections/experiments}
\section{Conclusions}\vspace{-5pt}
This paper develops an algorithm for finding low-NFE ODE solvers custom-tailored to general pre-trained flow models. Through extensive experiments we found that different models can benefit greatly from their own optimized solvers in terms of global truncation error (RMSE) and generation quality (FID). Currently, training a Bespoke solver requires roughly 1\% of the original model's training time, which can probably be still be made more efficient (\eg, by using training data samples and/or pre-processing sampling paths). Lastly, considering more elaborate models of $\varphi_r(\cdot),t_r$ could provide further benefits in fast sampling of pre-trained models. 




\pagebreak \newpage 
\bibliography{main}
\bibliographystyle{iclr2024_conference}


\newpage

\appendix

\section{Transformed paths}
\label{a:transformed_paths}
(Appendix to Section \ref{ss:transformed_paths}.)\\

\barut*

\begin{proof}
     Differentiating $\bar{x}(r)$ in \eqref{e:bar_x_r_from_x_t}, \ie, $\bar{x}(r)=\varphi_r(x(t_r))$ w.r.t.~$r$ and using the chain rule gives
\begin{align*}
    \dot{\bar{x}}(r) &= \frac{d}{dr}(\varphi_r(x(t_r)))\\
    &= \dot{\varphi}_r(x(t_r)) + \partial_x \varphi_r(x(t_r))\dot{x}(t_r) \dot{t}_r\\
    &= \dot{\varphi}_r(x(t_r)) + \partial_x \varphi_r(x(t_r))u_{t_r}(x(t_r)) \dot{t}_r\\
    &= \dot{\varphi}_r(\varphi^{-1}_r(\bar{x}(r))) + \partial_x \varphi_r(\varphi^{-1}_r(\bar{x}(r)))u_{t_r}(\varphi^{-1}_r(\bar{x}(r)))\dot{t}_r
\end{align*}
where in the third equality we used the fact that $x(t)$ solves the ODE in \eqref{e:ode} and therefore $\dot{x}(t)=u_t(x(t))$; and in the last equality we applied $\varphi^{-1}_r$ to both sides of \eqref{e:bar_x_r_from_x_t}, \ie,  $x(t_r)=\varphi_r^{-1}(\bar{x}(r))$. The above equation shows that 
\begin{equation}
    \dot{\bar{x}}(r) = u_r(\bar{x}(r)),
\end{equation}
where $\bar{u}_r(x)$ is defined in \eqref{e:bar_u_t}, as required.
\end{proof}

\section{Consistency of solvers}
\label{a:consistency}
(Appendix to Section \ref{ss:consistency}.) \\

\begin{figure}[ht]
\centering
\vspace{-10pt}
\includegraphics[width=0.4\textwidth]{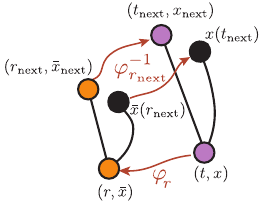}
\caption{Proof notations and setup.}
\label{fig:notations}
\end{figure}

\consistency*

\begin{proof} 
    Here $(t,x)$ is our input sample $x\in\Real^d$ at time $t\in[0,1]$. By definition $r=r_t$, $r_\nextt = r + h$, and $t_\nextt = t_{r_\nextt}$. Furthermore, by definition $\bar{x}=\varphi_r(x)$ is a sample at time $r$; $\bar{x}(r_\nextt)$ is the solution to the ODE defined by $\bar{u}_r$ starting from $(r,\bar{x})$; $\bar{x}_\nextt$ is an approximation to $\bar{x}(r_\nextt)$ as generated from the base ODE solver $\odestep$. Lastly, $x_\nextt = \varphi_{r_\nextt}^{-1}(\bar{x}_{r_\nextt})$ and $x(t_\nextt)= \varphi_{r_\nextt}^{-1}(\bar{x}(r_\nextt))$. See Figure \ref{fig:notations} for an illustration visualizing this setup.
    
    Now, since $\odestep$ is of order $k$ we have that 
    \begin{align} \nonumber
        \bar{x}(r_\nextt) - \bar{x}_\nextt &= 
        \bar{x}(r_\nextt) - \odestep(\bar{x},r;\bar{u}_r) \\ \label{ar:consistency_1} &= O((r_\nextt-r)^{k+1}). 
    \end{align}
    Now,
    \begin{align*}
    x(t_\nextt) - x_\nextt &=
    x(t_\nextt) - \varphi^{-1}_{r_\nextt}(\bar{x}_\nextt) \\
    &= x(t_\nextt) - \varphi^{-1}_{r_\nextt}(\bar{x}(r_\nextt) + O((r_\nextt-r)^{k+1}) ) \\
    &= x(t_\nextt) - \varphi^{-1}_{r_\nextt}(\bar{x}(r_\nextt)) + O((r_\nextt-r)^{k+1})\\
    &= O((r_\nextt-r)^{k+1}) \\
    &= O((t_\nextt-t)^{k+1}),
    \end{align*}
    where in the first equality we used the definition of $x_\nextt$; in the second equality we used \eqref{ar:consistency_1}; in the third equality we used the fact that $\varphi_r^{-1}$ is Lipschitz with constant $L$ (for all $r$); in the fourth equality we used the definition of the path transform, $x(t_\nextt)= \varphi_{r_\nextt}^{-1}(\bar{x}(r_\nextt))$ as mentioned above; and in the last equality we used the fact that $r_t$ is also Lipschitz with a constant $L$ and therefore $r_\nextt - r = r_{t_\nextt} - r_t = O(t_\nextt - t)$.    
\end{proof}

\section{Equivalence of Gaussian Paths and scale-time transformations}\label{a:equivalence}
(Appendix to Section \ref{ss:equivalence}.) \\ 
\equivalence*

\begin{proof}[Proof of theorem \ref{thm:equivalence}.]

Consider two arbitrary schedulers $(\alpha_t,\sigma_t)$ and $(\bar{\alpha}_r,\bar{\sigma}_r)$. 
We can find $s_r,t_r$ such that 
\begin{align}\label{e:scheduler_relation}
    \bar{\alpha}_r = s_r \alpha_{t_r}, \qquad \bar{\sigma}_r = s_r \sigma_{t_r}.
\end{align}
Indeed, one can check the following are such $s_r,t_r$:
\begin{align}\label{e:scale-time_relation}
t_r &= \mathrm{snr}^{-1}(\overline{\mathrm{snr}}(r)), \qquad s_r = \frac{\bar{\sigma}_r}{\sigma_{t_r}},
\end{align}
where we remember $\mathrm{snr}$ is strictly monotonic as defined in \eqref{e:def_scheduler}, hence invertible. On the other hand, given an arbitrary scheduler $(\alpha_t,\sigma_t)$ and an arbitrary scale-time transformation $(t_r,s_r)$ with $s_1=1$, we can define a scheduler $(\bar{\alpha}_r,\bar{\sigma}_r)$ via \eqref{e:scheduler_relation}. 

For case (i), we are given another scheduler $\bar{\alpha}_r, \bar{\sigma}_r$ and define a scale-time transformation $s_r,t_r$ with \eqref{e:scale-time_relation}. For case (ii), we are given a scale-time transformation $s_r,t_r$ and define a scheduler $\bar{\alpha}_r, \bar{\sigma}_r$ by \eqref{e:scheduler_relation}. 

Now, the scheduler $\bar{\alpha}_r, \bar{\sigma}_r$ defines sampling paths $\bar{x}(r)$ given by the solution of the ODE in \eqref{e:ode} with the marginal VF $\bar{u}_r^{(1)}(x)$ defined in \eqref{e:u_t}, \ie,
\begin{equation}
    \bar{u}^{(1)}_r(x) = \int \bar{u}_r(x|x_1) \frac{\bar{p}_r(x|x_1)q(x_1)}{\bar{p}_r(x)}dx_1,
\end{equation} 
where $\bar{u}_r(x|x_1) = \frac{\dot{\bar{\sigma}}_r}{\bar{\sigma}_r}x + \brac{\dot{\bar{\alpha}}_r - \dot{\bar{\sigma}}_r\frac{\bar{\alpha}_r}{\bar{\sigma}_r}}x_1$. 

The scale-time transformation $s_r,t_r$ gives rise to a second VF $\bar{u}^{(2)}_r(x)$ as in \eqref{e:bar_u_t_scale},
\begin{equation}
    \bar{u}^{(2)}_r(x) = \frac{\dot{s}_r}{s_r}x + \dot{t}_r s_r u_{t_r}\parr{\frac{x}{s_r}},
\end{equation}
where $u_t$ is the VF defined by the scheduler $(\alpha_t, \sigma_t)$ and \eqref{e:u_t}. 

By uniqueness of ODE solutions, the theorem will be proved if we show that 
\begin{equation}\label{ea:need_to_prove}
    \bar{u}^{(1)}_r(x) = \bar{u}^{(2)}_r(x), \quad \forall x\in\Real^d, r\in[0,1].
\end{equation}

For that end, we use the notation of determinants to express  
\begin{equation}
    \bar{u}_r(x|x_1) = \frac{1}{\bar{\sigma}_{r}}\abs{
    \begin{matrix} 0 & x & x_1 \\
    \bar{\sigma}_{r}& \bar{\alpha}_{r} & 1 \\
    \dot{\bar{\sigma}}_{r} & \dot{\bar{\alpha}}_{r}  & 0
    \end{matrix}
    },
\end{equation}
where $x,\ x_1 \in \R^d$ and  $\bar{\alpha}_{r},\ \bar{\sigma}_{r},\ \dot{\bar{\alpha}}_{r},\ \dot{\bar{\sigma}}_{r} \in \R$ as in vector cross product.
Differentiating $\bar{\alpha}_r,\bar{\sigma}_r$ w.r.t.~$r$ gives
\begin{align}\label{e:dot_scheduler_relation}
    \dot{\bar{\alpha}}_r=\dot{s}_r \alpha_{t_r} + s_r\dot{\alpha}_{t_r}\dot{t}_r, \qquad \dot{\bar{\sigma}}_r=\dot{s}_r \sigma_{t_r} + s_r\dot{\sigma}_{t_r}\dot{t}_r.
\end{align}
Using the bi-linearity of determinants shows that:
\begin{align*}
    \bar{u}_r(x|x_1) &= \frac{1}{\bar{\sigma}_{r}}\abs{
    \begin{matrix} 0 & x & x_1 \\
    \bar{\sigma}_{r}& \bar{\alpha}_{r} & 1 \\
    \dot{\bar{\sigma}}_{r} & \dot{\bar{\alpha}}_{r}  & 0
    \end{matrix}
    }\\
    &= \frac{1}{s_r\sigma_{t_r}}\abs{
    \begin{matrix} 0 & x & x_1 \\
    s_r\sigma_{t_r}& s_r\alpha_{t_r} & 1 \\
    \dot{s}_r \sigma_{t_r} + s_r\dot{\sigma}_{t_r}\dot{t}_r & \dot{s}_r \alpha_{t_r} + s_r\dot{\alpha}_{t_r}\dot{t}_r  & 0
    \end{matrix}
    }\\
    &= \frac{1}{s_r\sigma_{t_r}}\abs{
    \begin{matrix} 0 & x & x_1 \\
    s_r\sigma_{t_r}& s_r\alpha_{t_r} & 1 \\
    \dot{s}_r \sigma_{t_r} & \dot{s}_r \alpha_{t_r} & 0
    \end{matrix}
    } +
    \frac{1}{s_r\sigma_{t_r}}\abs{
    \begin{matrix} 0 & x & x_1 \\
    s_r\sigma_{t_r}& s_r\alpha_{t_r} & 1 \\
     s_r\dot{\sigma}_{t_r}\dot{t}_r &  s_r\dot{\alpha}_{t_r}\dot{t}_r  & 0
    \end{matrix}
    }\\
    &= \frac{\dot{s}_r}{s_r}x + \frac{s_r\dot{t}_r}{\sigma_{t_r}}\abs{
    \begin{matrix} 0 & \frac{x}{s_r} & x_1 \\
    \sigma_{t_r} & \alpha_{t_r} & 1 \\
     \dot{\sigma}_{t_r} &  \dot{\alpha}_{t_r}  & 0
    \end{matrix}
    }\\
    &= \frac{\dot{s}_r}{s_r}x + s_r\dot{t}_r u_{t_r}\parr{\frac{x}{s_r}\Big\vert x_1},
\end{align*}
where in the second equality we substitute $\dot{\bar{\sigma}}_{r},\ \dot{\bar{\alpha}}_{r}$ as in \eqref{e:dot_scheduler_relation}, in the third and fourth equality we used the bi-linearity of determinants, and in the last equality we used the definition of $u_t(x|x_1) = \frac{\dot{\sigma}_t}{\sigma_t}x + \brac{\dot{\alpha}_t - \dot{\sigma}_t\frac{\alpha_t}{\sigma_t}}x_1$ expressed in determinants notation.
Furthermore, since 
\begin{equation}
    \bar{p}_r(x|x_1)=\gN(x\vert s_r \alpha_{t_r}x_1, s^2_r \sigma^2_{t_r}I) \propto \gN\parr{\frac{x}{s_r} \Big \vert \alpha_{t_r} x_1, \sigma^2_{t_r}I}=p_{t_r}\parr{\frac{x}{s_r}\Big\vert x_1}
\end{equation}
we have that 
\begin{equation}
    \bar{p}_r(x_1|x) = p_{t_r}\parr{x_1 \Big| \frac{x}{s_r}}.
\end{equation}
Therefore,
\begin{align*}
    \int \bar{u}_r(x|x_1) \frac{\bar{p}_r(x|x_1)q(x_1)}{\bar{p}_r(x)}dx_1 &= \E_{\bar{p}_r(x_1|x)}\bar{u}_r(x|x_1)\\
    & = \E_{p_{t_r}\parr{x_1 |\frac{x}{s_r}}}\brac{\frac{\dot{s}_r}{s_r}x + s_r\dot{t}_r u_{t_r}\parr{\frac{x}{s_r}\Big\vert x_1}}\\
    &= \frac{\dot{s}_r}{s_r}x + s_r\dot{t}_r\E_{p_{t_r}\parr{x_1 | \frac{x}{s_r}}}u_{t_r}\parr{\frac{x}{s_r}\Big\vert x_1}\\
    & = \frac{\dot{s}_r}{s_r}x + s_r\dot{t}_ru_{t_r}\parr{\frac{x}{s_r}},
\end{align*}
where in the first equality we used Bayes rule, in the second equality we substitute $\bar{u}_r(x|x_1)$ and $\bar{p}_r(x_1|x)$ as above, and in the last equality we used the definition of $u_t$ as in \eqref{e:u_t}. We have proved \eqref{ea:need_to_prove} and that concludes the proof. 
\end{proof}

\newpage
\section{Lipschitz constants of $\odestep^\theta$.}
\label{a:lipschitz_of_step}
    (Appendix to Section \ref{ss:rmse_bound}.) \\    
We are interested in computing $L^\theta_i$, a Lipschitz constant of the bespoke solver step function $\odestep_x^\theta(t_i,\cdot\,;u_t)$. Namely, $L^\theta_i$ should satisfy 
\begin{equation}
    \norm{\odestep_x^\theta(t_i,x\,;u_t) - \odestep_x^\theta(t_i,y\,;u_t)} \leq L^\theta_i \norm{x-y}, \qquad \forall x,y\in \Real^d.
\end{equation}
We remember that $\odestep_x^\theta(t_i,\cdot\,;u_t)$ is defined using a base solver and the VF $\bar{u}_{r_i}(\cdot)$; hence, we begin by computing a Lipschitz constant for $\bar{u}_{r_i}$ denoted $L_{\Bar{u}}(r_i)$ in an auxiliary lemma:
\begin{lemma}\label{lem:lipshitz_u_bar}
    Assume that the original velocity field $u_t$ has a Lipschitz constant $L_u>0$. Then for every $r_i\in[0,1]$, $L_{\tau}\geq L_u$,   and $x,y \in \R^d$
    \begin{equation}
        \norm{\bar{u}_{r_i}(x) - \bar{u}_{r_i}(y)} \le L_{\Bar{u}}(r_i)\norm{x-y},
    \end{equation}
    where
    \begin{equation}
        L_{\Bar{u}}(r_i) = \frac{\abs{\dot{s}_i}}{s_i} + \dot{t}_iL_{\tau}
    \end{equation}
\end{lemma}
\begin{proof}[Proof of lemma \ref{lem:lipshitz_u_bar}]
    Since the original velocity field $u$ has a Lipshitz constant $L_{u}>0$, for every $t\in[0,1]$ and $x,y \in \R^d$
    \begin{equation}
        \norm{u_t(x) - u_t(y)} \le L_u\norm{x-y}.
    \end{equation}
    Hence
    \begin{align}
        \norm{\bar{u}_{r_i}(x) - \bar{u}_{r_i}(y)} 
        & = \norm{\frac{\dot{s}_i}{s_i}x + \dot{t}_is_iu_{t_i}\parr{\frac{x}{s_i}} - \parr{\frac{\dot{s}_i}{s_i}y + \dot{t}_is_iu_{t_i}\parr{\frac{y}{s_i}}}}\\
        &= \norm{\frac{\dot{s}_i}{s_i}(x-y) + \dot{t}_is_i \parr{u_{t_i}\parr{\frac{x}{s_i}}-u_{t_i}\parr{\frac{y}{s_i}}}}\\
        & \le \frac{\abs{\dot{s}_i}}{s_i}\norm{x-y} + \dot{t}_is_i\norm{ u_{t_i}\parr{\frac{x}{s_i}}-u_{t_i}\parr{\frac{y}{s_i}}}\\
        & \le \parr{\frac{\abs{\dot{s}_i}}{s_i} + \dot{t}_iL_u}\norm{x-y}\\
        & \le \parr{\frac{\abs{\dot{s}_i}}{s_i} + \dot{t}_iL_\tau}\norm{x-y}.
    \end{align}    
\end{proof}
We first apply the auxiliary lemma \ref{lem:lipshitz_u_bar} to compute a Lipschitz constant of $\odestep_x^\theta(t_i,\cdot\,;u_t)$ with RK1 (Euler method) as the base solver in lemma \ref{lem:bound_L_euler} and for RK2 (Midpoint method) as the base solver in lemma \ref{lem:bound_L_midpoint}.
\begin{restatable}{lemma}{lipeuler}\label{lem:bound_L_euler}
    (RK1 Lipschitz constant) Assume that the original velocity field $u_t$ has a Lipschitz constant $L_u>0$. Then, for every $L_\tau \geq L_u$,
    \begin{equation}
        L_i^{\theta} = \frac{s_i}{s_{i+1}}\parr{1+ h  L_{\Bar{u}}(r_i) },
    \end{equation}
    is a Lipschitz constant of RK1-Bespoke update step, where
    \begin{equation}
        L_{\Bar{u}}(r_i) = \frac{\abs{\dot{s}_i}}{s_i} + \dot{t}_iL_{\tau}.
    \end{equation}        
\end{restatable}
\begin{proof}[Proof of lemma \ref{lem:bound_L_euler}]
    We begin with writing an explicit expression of $\odestep_x^\theta(t_i,x,;u_t)$ for Euler solver in terms of the transformed velocity field $\bar{u}_r$. That is, 
    \begin{align}
        \odestep_x^\theta(t_i,x,;u_t) &= \frac{1}{s_{i+1}}\brac{s_i x + h\bar{u}_{r_i}(s_i x)}.
    \end{align}
    So that applying the triangle inequality and lemma \ref{lem:lipshitz_u_bar} gives
    \begin{align*}
        \norm{\odestep_x^\theta(t_i,x;u_t) - \odestep_x^\theta(t_i,y;u_t)}
        & = \frac{1}{s_{i+1}}\norm{s_i x + h\bar{u}_{r_i}(s_i x) -
            \brac{s_i y + h\bar{u}_{r_i}(s_i y)}}\\
        & \le  \frac{s_i}{s_{i+1}}\norm{x-y} 
        + \frac{h}{s_{i+1}} \norm{ 
        \bar{u}_{r_i}(s_i x) - \bar{u}_{r_i}(s_i y)
        }\\
        & \le \frac{s_i}{s_{i+1}}\norm{x - y}
        + \frac{h}{s_{i+1}} \parr{\frac{\abs{\dot{s}_i}}{s_i} + \dot{t}_iL_\tau}\norm{s_i x -s_i y}
        \\
        & = \frac{s_i}{s_{i+1}}\parr{1+ h\parr{\frac{\abs{\dot{s}_i}}{s_i} + \dot{t}_iL_\tau}}\norm{x - y}.
    \end{align*}
\end{proof}
\begin{lemma}\label{lem:bound_L_midpoint}(RK2 Lipschitz constant)
    Assume that the original velocity field $u_t$ has a Lipschitz constant $L_u>0$. Then for every $L_\tau \geq L_u$ 
    \begin{equation}
        L_i^{\theta} =  \frac{s_i}{s_{i+1}}\brac{
    1 + h L_{\bar{u}}(r_{i+\frac{1}{2}})\parr{ 1+ \frac{h}{2}L_{\bar{u}}(r_i) }
    }
    \end{equation}
    is a Lipschitz constant of RK2-Bespoke update step, where
    \begin{equation}
        L_{\Bar{u}}(r_i) = \frac{\abs{\dot{s}_i}}{s_i} + \dot{t}_iL_{\tau}.
    \end{equation}    
\end{lemma}
\begin{proof}[Proof of lemma \ref{lem:bound_L_midpoint}]
    We begin by writing explicit expression of $\odestep_x^\theta(t_i,x;u_t)$ for RK2 (Midpoint) method in terms of the transformed velocity field $\bar{u}_r$. We set 
    \begin{align}
        z = s_i x + \frac{h}{2}\bar{u}_{r_i}(s_i x), \quad w = s_i y + \frac{h}{2}\bar{u}_{r_i}(s_i y).
    \end{align}
    then 
    \begin{equation}
        \odestep_x^\theta(t_i,x,;u_t) = \frac{1}{s_{i+1}}\brac{s_i x + h\bar{u}_{r_{i+\frac{1}{2}}}(z)}, 
    \end{equation}
    and
    \begin{equation}
        \odestep_x^\theta(t_i,y;u_t) = \frac{1}{s_{i+1}}\brac{s_i y + h\bar{u}_{r_{i+\frac{1}{2}}}(w)}.
    \end{equation}
    So that applying the triangle inequality and lemma \ref{lem:lipshitz_u_bar} gives
    \begin{align} \nonumber
        \norm{\odestep_x^\theta(t_i,x;u_t) - \odestep_x^\theta(t_i,y;u_t)}
        & \le  \frac{s_i}{s_{i+1}}\norm{x-y} 
        + \frac{h}{s_{i+1}} \norm{ 
        \bar{u}_{r_{i+\frac{1}{2}}}\parr{z} - \bar{u}_{r_{i+\frac{1}{2}}}\parr{w}}\\ \label{ea:lipschitz_bound_RK2}
        & \le \frac{s_i}{s_{i+1}}\norm{x-y} 
        + \frac{h}{s_{i+1}} L_{\bar{u}}(r_{i+\frac{1}{2}})\norm{z - w}.
    \end{align}
    We apply the triangle inequality and the lemma\ref{lem:lipshitz_u_bar} again to $\norm{z-w}$. That is,
    \begin{align*}
        \norm{z-w}
        & = \norm{ s_i x + \frac{h}{2}\bar{u}_{r_i}(s_i x) - \parr{s_i y + \frac{h}{2}\bar{u}_{r_i}\parr{s_i y}}}\\
        & \le s_i \norm{x-y} + \frac{h}{2}\norm{\bar{u}_{r_i}(s_i x) - \bar{u}_{r_i}\parr{s_i y}}\\
        & \le s_i \norm{x-y} + \frac{h}{2}L_{\bar{u}}(r_i)s_i\norm{x-y}\\
        & = s_i \parr{1+\frac{h}{2}L_{\bar{u}}(r_i)}\norm{x-y}.
    \end{align*}
Substitute back in \eqref{ea:lipschitz_bound_RK2} gives
\begin{align*}
    \frac{s_i}{s_{i+1}}\norm{x-y} 
        + \frac{h}{s_{i+1}} L_{\bar{u}}(r_{i+\frac{1}{2}})\norm{z-w}
    & \le \frac{s_i}{s_{i+1}}\brac{
    1 + h L_{\bar{u}}(r_{i+\frac{1}{2}})\parr{ 1+ \frac{h}{2}L_{\bar{u}}(r_i) }
    }\norm{x-y}.\\
    %
\end{align*}
\end{proof}



\section{Derivation of parametric solver $\odestep^\theta$}
\label{a:parametric_solver_step}
    (Appendix to Section \ref{ss:two_use_cases}.)\\    
This section presents a derivation of $n$-step parametric solver 
\begin{equation}
 \odestep^\theta(t,x\,;u_t) = \parr{\odestep_t^\theta(t,x\,;u_t),\odestep_x^\theta(t,x\,;u_t)} \end{equation}
for scale-time transformation (\eqref{e:scale_time}) with two options for a base solver: (i) RK1 method (Euler) as the base solver; and $(ii)$ RK2 method (Midpoint). We do so by following equation \ref{e:step_parametric_1}-\ref{e:step_parametric}. We begin with RK1 and derive \eqref{e:step_theta_euler}. Given $(t_i,x_i)$, \eqref{e:step_parametric_1} for the scale time transformation is,
\begin{equation}\label{e:drive_param_step_1}
    \bar{x}_i = s_ix_i.
\end{equation}
Then according to \eqref{e:step_parametric_2},
\begin{align}
    \bar{x}_{i+1} &= \odestep_x(r_i, \bar{x}_{i}, \bar{u}_{r_i})\\
    & = \bar{x}_{i} +h\bar{u}_{r_i}(\bar{x}_{i})\\
    & = \bar{x}_{i} +h \parr{\frac{\dot{s}_i}{s_i}\bar{x}_{i} + \dot{t}_i s_i u_{t_i}\parr{\frac{\bar{x}_{i}}{s_i}}}\\
    & = s_ix_i + h \parr{\dot{s}_ix_i + \dot{t}_i s_i u_{t_i}(x_i)},
\end{align}
where in the second equality we apply an RK1 step (\eqref{e:euler}), in the third equality we substitute $\bar{u}_{r_i}$ using \eqref{e:bar_u_t_scale}, and in the fourth equality we substitute $\bar{x}_i$ as in \eqref{e:drive_param_step_1}. According to RK1 step (\eqref{e:euler}) we also have $r_{i+1}=r_i+h$. Finally, \eqref{e:step_parametric} gives,
\begin{align}
    \odestep_t^\theta(t_i,x_i\,;u_t) &= t_{i+1} \\
    \odestep_x^\theta(t_i,x_i\,;u_t) &= \frac{s_i +h\dot{s_i}}{s_{i+1}}x_i + \frac{h}{s_{i+1}}\dot{t}_i s_i u_{t_i}(x_i),
\end{align}
as in \eqref{e:step_theta_euler}.\\ 

Regarding the second case, \eqref{e:step_parametric_2} for the RK2 method (\eqref{e:midpoint}) is,
\begin{align}
    \bar{x}_{i+1} &= \odestep_x(r_i, \bar{x}_{i}, \bar{u}_{r_i})\\
    & = \bar{x}_{i} +h\bar{u}_{r_{i+\frac{1}{2}}}(\bar{x}_{i+\frac{1}{2}}),
\end{align}
where 
\begin{equation}
    \bar{x}_{i+\frac{1}{2}} = \bar{x}_i + \frac{h}{2}\bar{u}_{r_i}(\bar{x}_i)
\end{equation}
is the RK1 step from $(r_i,\bar{x}_i)$ with step size $h/2$. Now substituting $\bar{x}_i$ as defined \eqref{e:drive_param_step_1} and $\bar{u}_r$ as defined in \eqref{e:bar_u_t_scale} we get
\begin{align}
    \bar{x}_{i+1} = s_i x_i  + h \parr{\frac{\dot{s}_{i+\frac{1}{2}}}{s_{i+\frac{1}{2}}}\bar{x}_{{i+\frac{1}{2}}} + \dot{t}_{i+\frac{1}{2}} s_{i+\frac{1}{2}} u_{t_{i+\frac{1}{2}}}\parr{\frac{\bar{x}_{{i+\frac{1}{2}}}}{s_{i+\frac{1}{2}}}}}. 
\end{align}
where 
\begin{equation}
    \bar{x}_{i+\frac{1}{2}} = \parr{s_i + \frac{h}{2}\dot{s}_i}x_i + \frac{h}{2}s_i \dot{t}_i u_{t_i}(x_i).
\end{equation}
Lastly, according to \eqref{e:step_parametric} we have
\begin{equation}
    \odestep_x^\theta(t_i,x\,;u_t) = \frac{s_i}{s_{i+1}}x_i +\frac{h}{s_{i+1}}\parr{\frac{\dot{s}_{i+\frac{1}{2}}}{s_{i+\frac{1}{2}}}\bar{x}_{{i+\frac{1}{2}}} + \dot{t}_{i+\frac{1}{2}} s_{i+\frac{1}{2}} u_{t_{i+\frac{1}{2}}}\parr{\frac{\bar{x}_{{i+\frac{1}{2}}}}{s_{i+\frac{1}{2}}}}},
\end{equation}
as in \eqref{e:step_theta_midpoint} where $z_i=\bar{x}_{i+\frac{1}{2}}$.


\section{Implementation Details}
\label{a:impl_details}
(Appendix to Section \ref{ss:implementaiton_of_the_rmse_bespoke_loss}.)\\
This section presents further implementation details, complementing the main text. Our parametric family of solvers $\odestep^\theta$ is defined via a base solver $\odestep$ and a transformation $(t_r,\ \varphi_r)$ as defined in \eqref{e:step_parametric}.  We consider the RK2 (Midpoint, \eqref{e:midpoint}) method as the base solver with $n$ steps and $(t_r,\ \varphi_r)$ the scale-time transformation (\eqref{e:scale_time}). That is, $\varphi_r(x)=s_rx$, where $s:[0,1]\rightarrow \R_{>0}$, as in \eqref{e:scale}, which is our primary use case.

\textbf{Parameterization of $t_i$.} Remember that $t_r$ is a strictly monotonic, differentiable, increasing function $t:[0,1]\rightarrow [0,1]$. Hence, $t_i$ must satisfy the  constraints as in \eqref{e:theta_midpoint}, \ie, 
\begin{align}
    & 0=t_0 < t_{\frac{1}{2}} < \cdots < t_n=1\\
    & \dot{t}_0, \dot{t}_{\frac{1}{2}}, \ldots, \dot{t}_{n-1}, \dot{t}_{n-\frac{1}{2}} > 0.
\end{align}
To satisfy these constrains, we model $t_i$ and $\dot{t}_i$ via 
\begin{equation}
    t_i = \frac{\sum_{j=0}^i|\theta^t_j|}{\sum_{k=0}^n|\theta^t_k|}, \quad \dot{t}_i = |\theta^{\dot{t}}_i|,
\end{equation}
where $\theta^t_i$ and $\theta^{\dot{t}}_i$, $i=0,\frac{1}{2},...,n$ are free learnable parameters. 

\textbf{Parameterization of $s_i$.} Since $s_r$ is a strictly positive, differentiable function satisfying a boundary condition at $r=0$, the sequence $s_i$ should satisfy the constraints as in \eqref{e:theta_midpoint}, \ie, 
\begin{align}
     s_{\frac{1}{2}}, s_1, \ldots, s_n > 0 \ \ , \ \ s_0=1,
\end{align}
and $\dot{s}_i$ are unconstrained. Similar to the above, we model $s_i$ and $\dot{s}_i$ by 
\begin{equation}
    s_i = \begin{cases}
        0  &  i=0\\
        \exp{\theta^s_i} &  \text{otherwise}
    \end{cases},\quad \dot{s}_i = \theta_i^{\dot{s}},
\end{equation}
where $\theta^s_i$ and $\theta^{\dot{s}}_i$, $i=0,\frac{1}{2},...,n$ are free learnable parameters.  

\textbf{Bespoke training.} The pseudo-code for training a Bespoke solver is provided in Algorithm \ref{alg:bes_training}. Here we add some more details on different steps of the training algorithm. We initialize the parameters $\theta$ such that the scale-transformation is the Identity transformation. That is, for every $i=0,\frac{1}{2},...,n$,
\begin{align}
    & t_i = \frac{i}{n},\quad \dot{t}_i = 1,\\
    & s_i  =1,\quad \dot{s}_i =0.
\end{align}
Explicitly, in terms of the learnable parameters, for every $i=0,\frac{1}{2},...,n$,
\begin{align}
    & \theta_i^t =1, \quad \theta_i^{\dot{t }}=1\\
    &\theta^s_i=0,\quad \theta_i^{\dot{s}}=0.
\end{align}
To compute the GT path $x(t)$, we solve the ODE in \eqref{e:ode} with the pre-trained model $u_t$ and DOPRI5 method, then use linear interpolation to extract $x(t_i)$, $i=0,1,...,n$~\citep{torchdiffeq}. Then, apply $x_i^{\text{aux}}(t)$ (\eqref{e:x_aux}) to correctly handle the gradients w.r.t.~to $\theta^t_i$. To compute the loss $\mathcal{L}_{\text{bes}}$ (\eqref{e:loss_bes}) we compute $x_{i+1}=\odestep_x^{\theta}\parr{x^{\aux}_i(t_{i}),t_{i};u_t}$ with equations \ref{e:step_theta_midpoint},\ref{e:z_i}, and compute $M_i$ via lemma \ref{lem:bound_L_midpoint} with $L_\tau=1$. Finally, we use Adam optimizer \cite{kingma2017adam} with a learning rate of $2e^{-3}$.

\newpage
\section{Bespoke RK1 versus RK2}
\label{a:rk1_vs_rk2}
\input{tables/rk1_vs_rk2_appendix}

\newpage
\section{CIFAR10}
\label{a:image_gen_cifar10}

\input{tables/cifar10_table_appendix}
\begin{figure}[h!]
    \centering
    \begin{tabular}{@{\hspace{0pt}}c@{\hspace{0pt}}c@{\hspace{0pt}}c@{\hspace{0pt}}} \quad CIFAR10: $\eps$-VP & \quad  CIFAR10: FM/$v$-CS & \quad CIFAR10: FM-OT \\
    \includegraphics[width=0.335\textwidth]{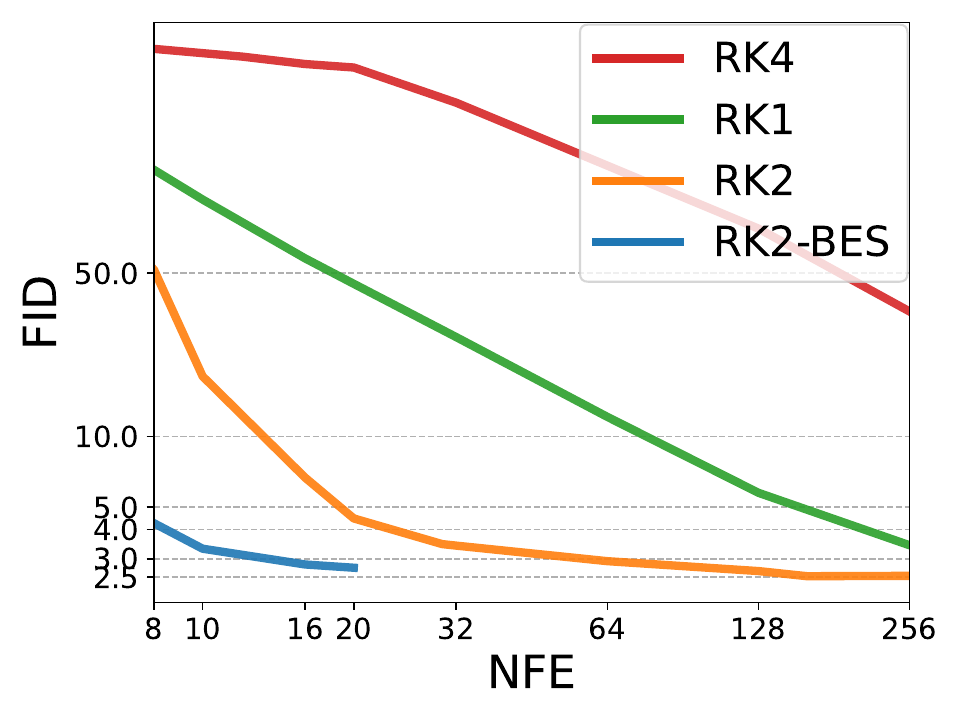} & \includegraphics[width=0.335\textwidth]{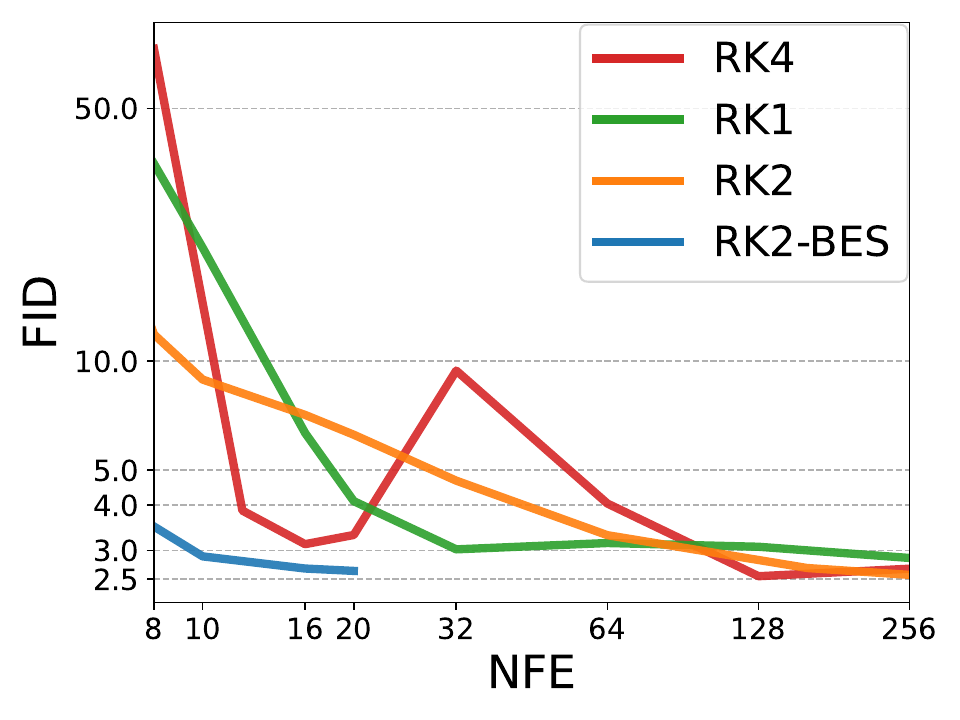} & \includegraphics[width=0.335\textwidth]{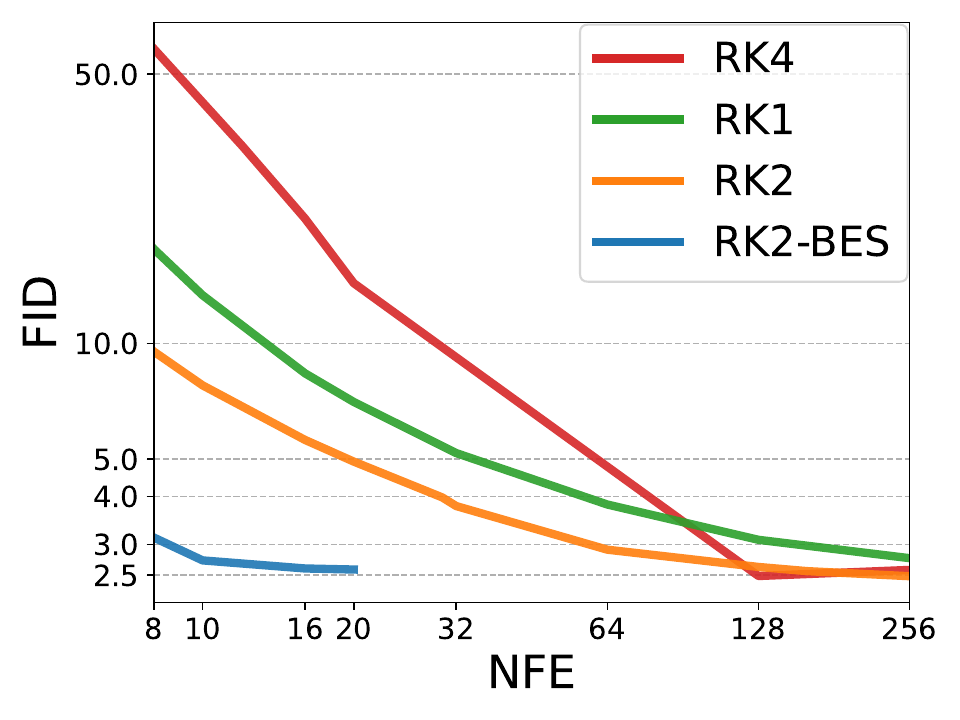} \\
     \includegraphics[width=0.335\textwidth]{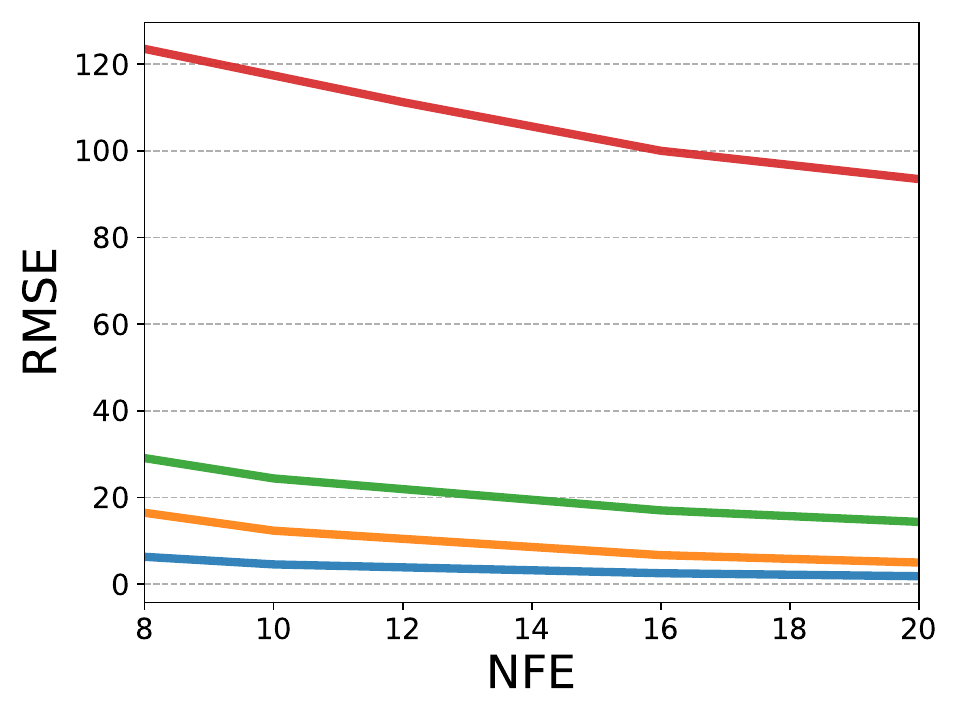}  & \includegraphics[width=0.335\textwidth]{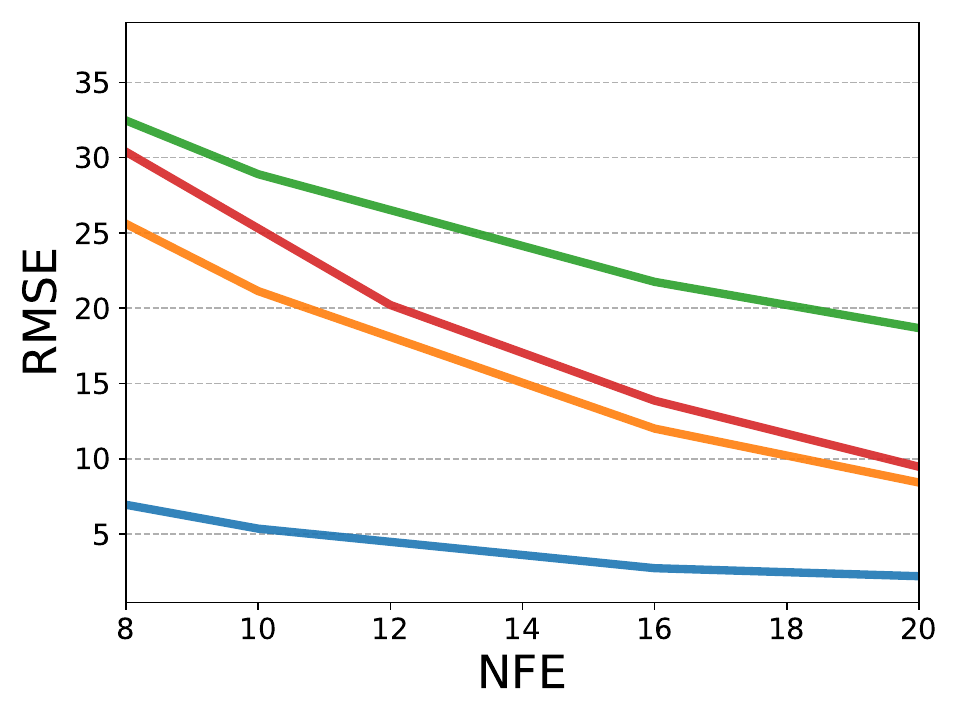} & \includegraphics[width=0.335\textwidth]{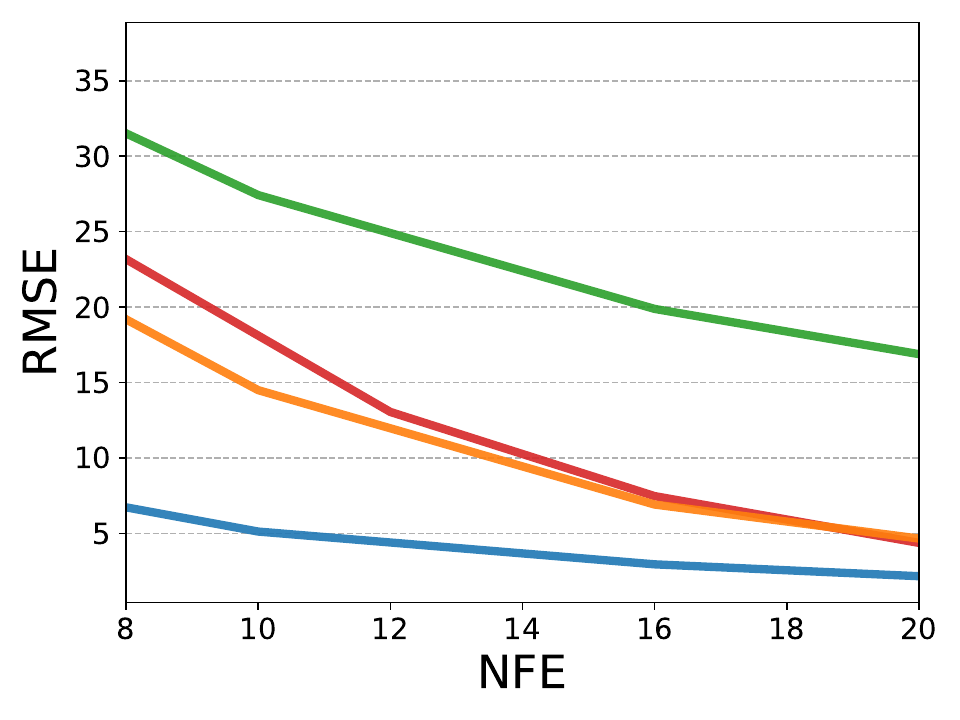} \\
     \includegraphics[width=0.335\textwidth]{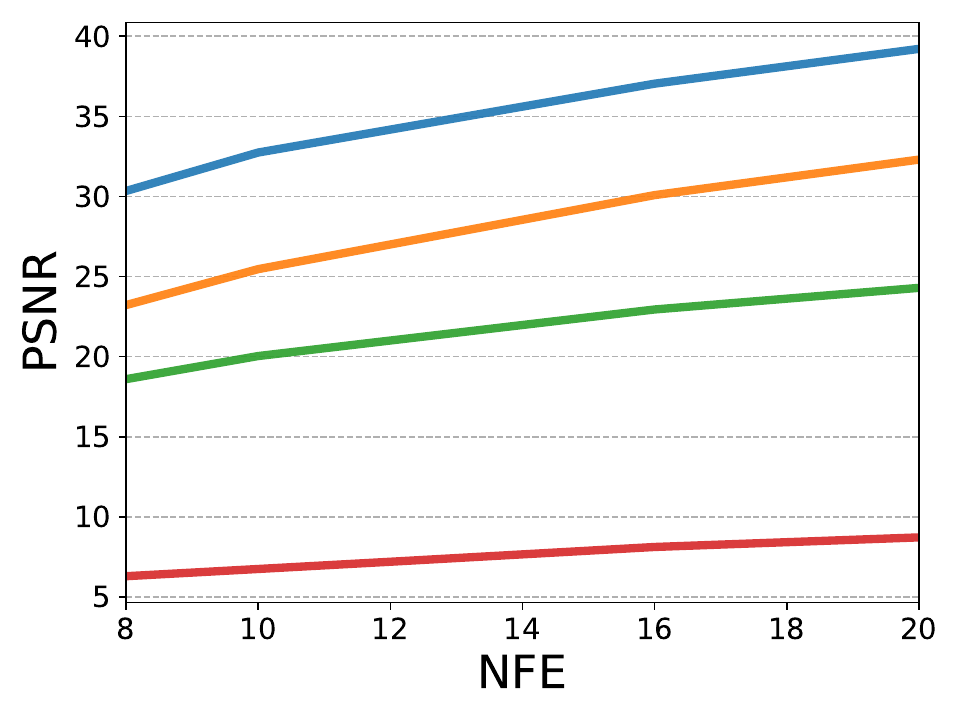}  & \includegraphics[width=0.335\textwidth]{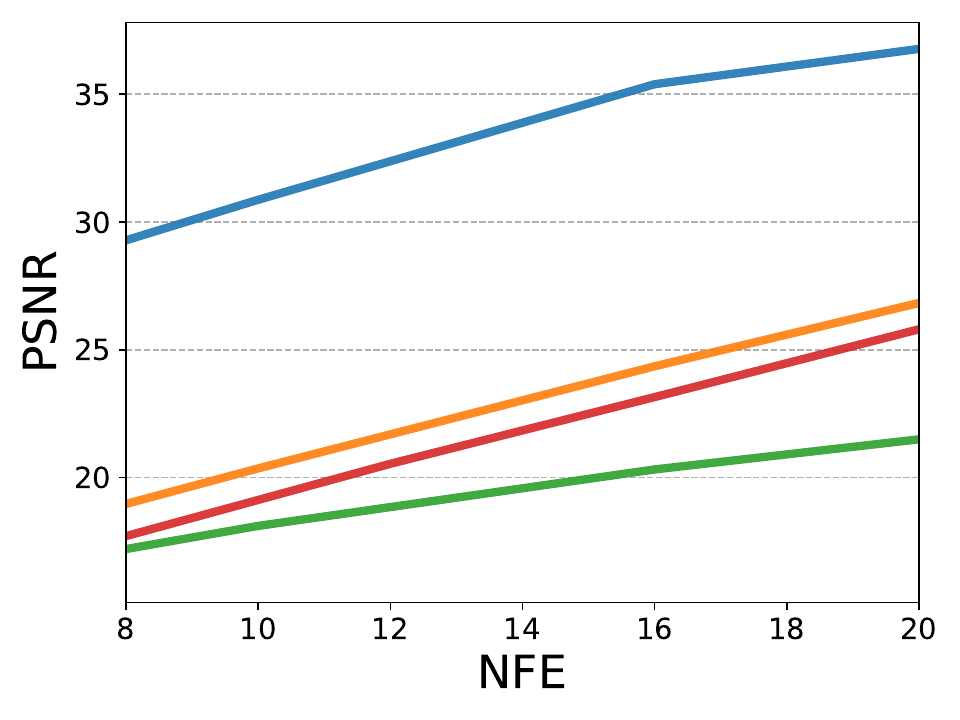} & \includegraphics[width=0.335\textwidth]{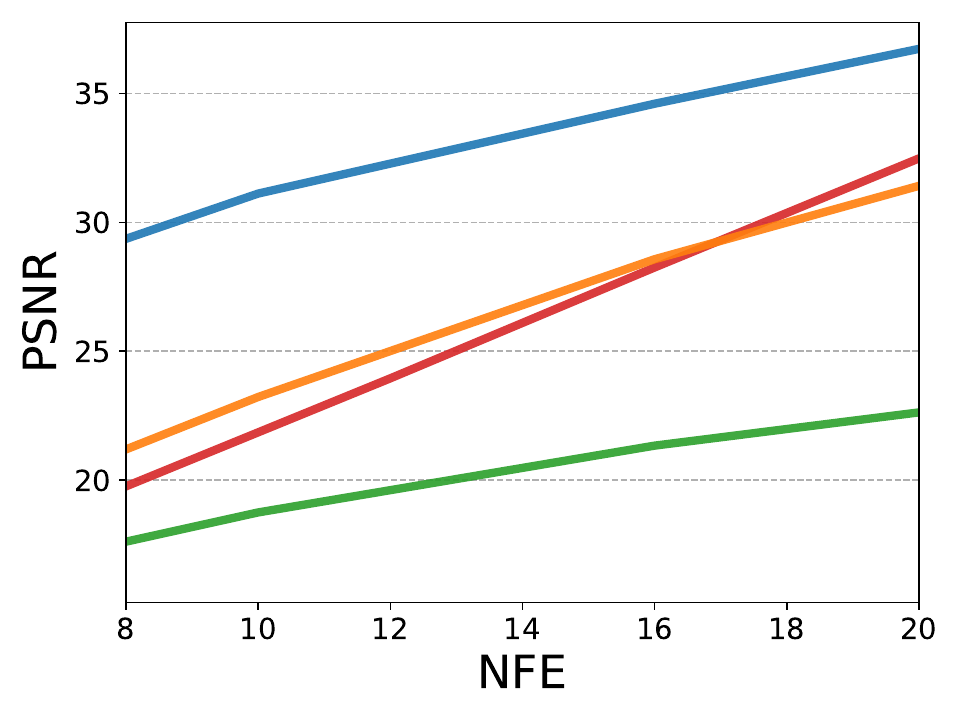} 
\end{tabular}    
    \caption{CIFAR10 sampling with Bespoke RK2 solvers vs.~RK1,RK2,RK4: FID vs.~NFE (top row), RMSE vs.~NFE (middle row), and PSNR vs.~NFE (bottom row).}
    \label{fig:nfe_vs_fid_cifar10}    
\end{figure}

\newpage
\section{ImageNet-64/128}
\label{a:image_gen_imagenet}
\input{tables/graphs_imagenet_training}
\input{tables/graphs_imagenet_appendix}

\section{AFHQ-256}
\label{a:image_gen_afhq}
\input{tables/graphs_afhq_appendix}

\newpage
\section{Ablations}
\label{a:ablations}

\input{tables/scale_time_ablation}
\input{tables/transferred_ablation}

\newpage
\section{Trained Bespoke solvers}
In this section we present the trained Bespoke solvers by visualizing their respective parameters $\theta$. 
\begin{figure}[h!]
    \centering
    \begin{tabular}{@{\hspace{0pt}}c@{\hspace{0pt}}c@{\hspace{0pt}}c@{\hspace{0pt}}c@{\hspace{0pt}}}
    {\quad \ \scriptsize NFE=8} & {\quad \ \scriptsize NFE=10} & {\quad \ \scriptsize NFE=16}  & {\quad \ \scriptsize NFE=20} \\
    \raisebox{3\height}{\rotatebox[origin=c]{90}{\scriptsize FM-OT}}\includegraphics[width=0.25\textwidth]{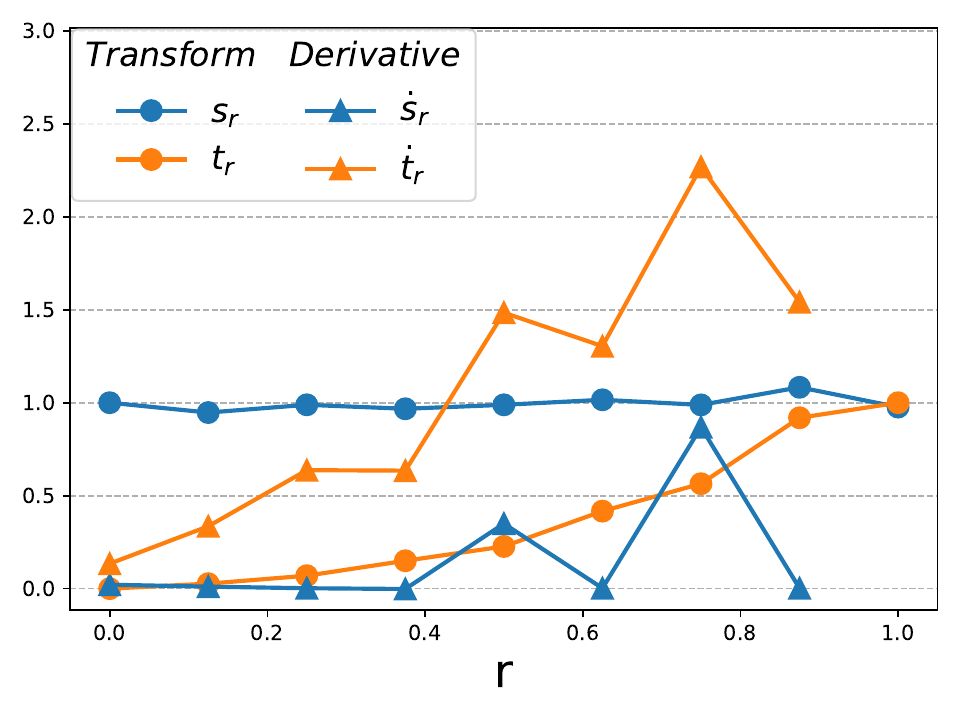} & \includegraphics[width=0.25\textwidth]{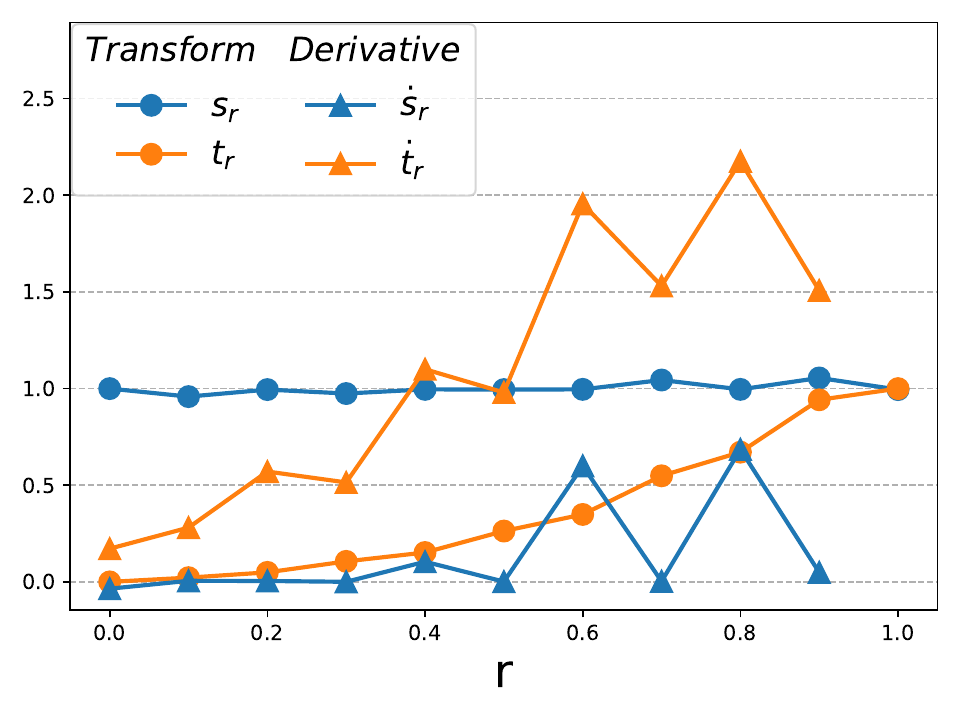} & \includegraphics[width=0.25\textwidth]{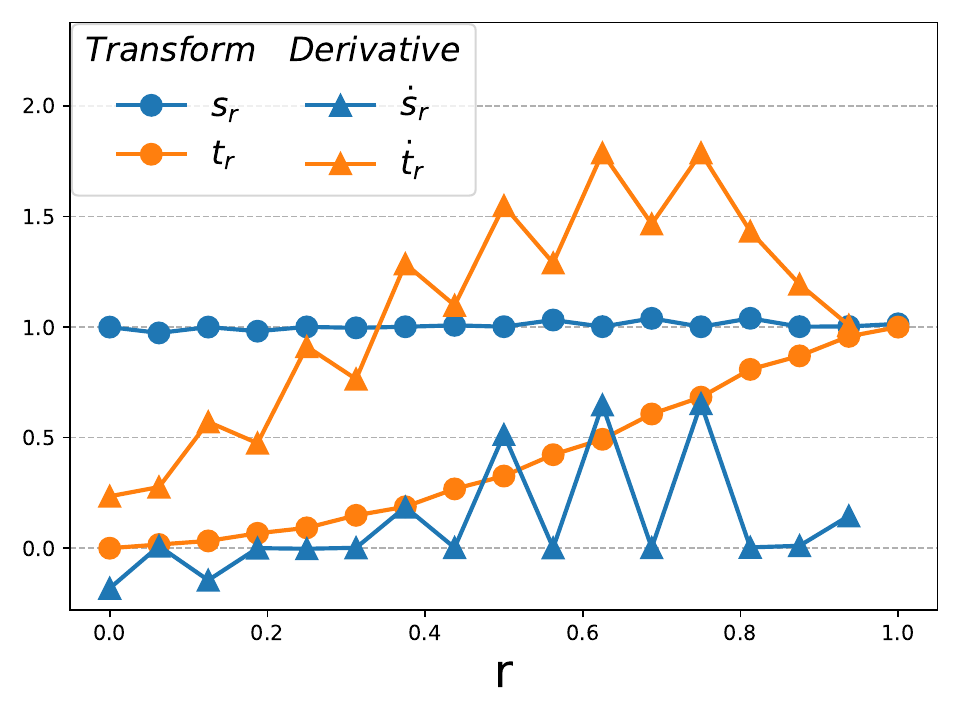} &
    \includegraphics[width=0.25\textwidth]{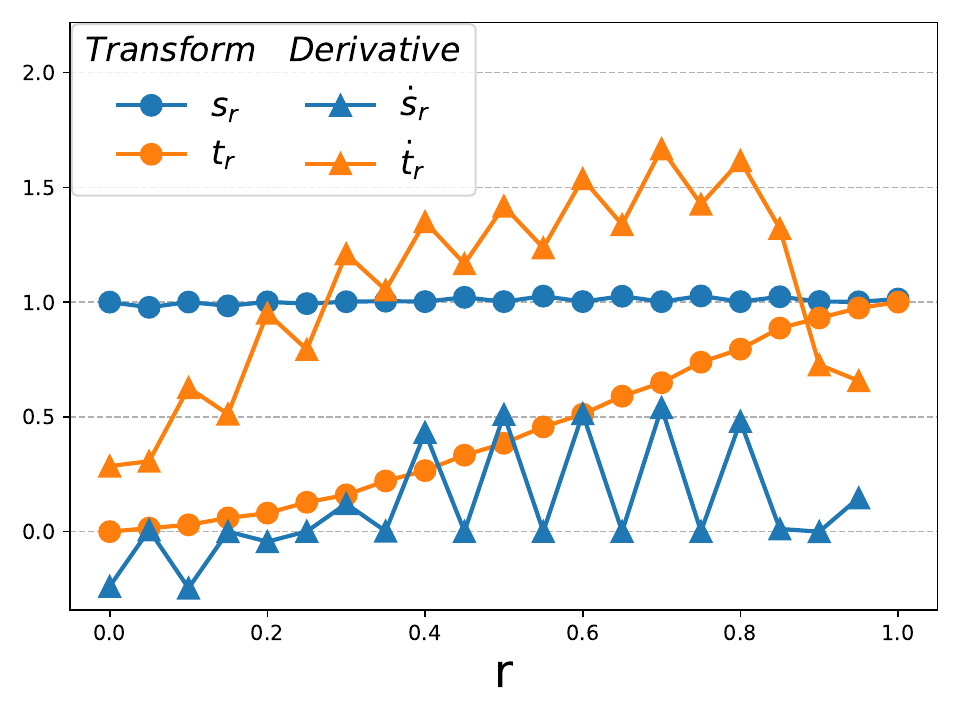}\\
\end{tabular}    
    \caption{Trained $\theta$ of Bespoke-RK2 solvers on ImageNet-128 FM-OT for NFE 8/10/16/20.}
    \label{fig_a:scheme_imagenet128}
\end{figure}
\begin{figure}[h!]
    \centering
    \begin{tabular}{@{\hspace{0pt}}c@{\hspace{0pt}}c@{\hspace{0pt}}c@{\hspace{0pt}}c@{\hspace{0pt}}}
    {\quad \ \scriptsize NFE=8} & {\quad \ \scriptsize NFE=10} & {\quad \ \scriptsize NFE=16}  & {\quad \ \scriptsize NFE=20} \\
    \raisebox{3.5\height}{\rotatebox[origin=c]{90}{\scriptsize $\eps$-pred}}\includegraphics[width=0.25\textwidth]{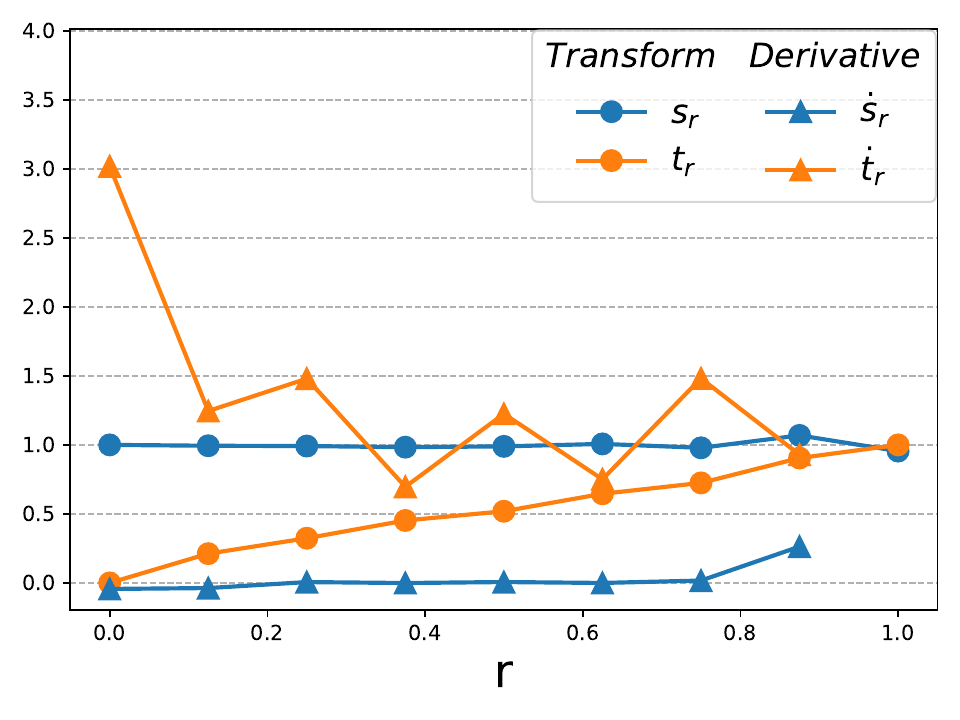} & \includegraphics[width=0.25\textwidth]{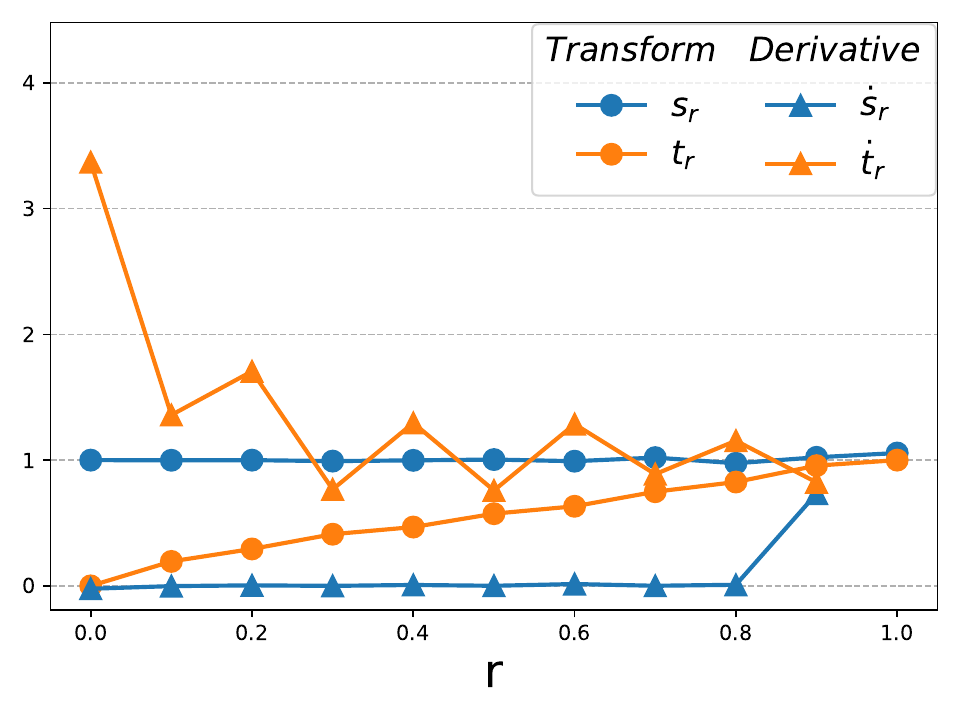} & \includegraphics[width=0.25\textwidth]{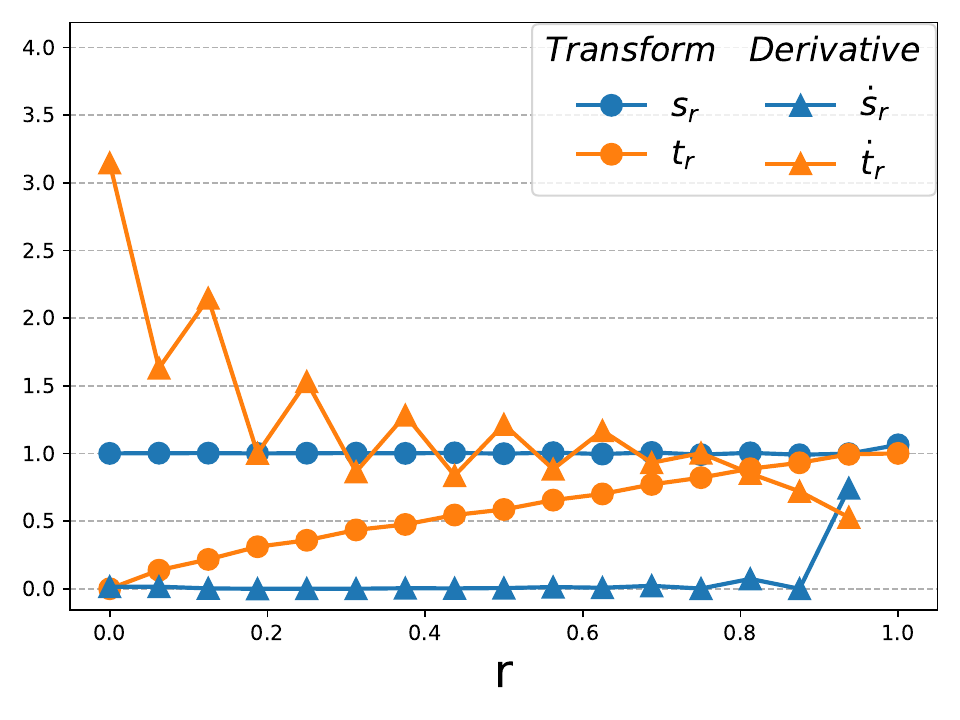} &
    \includegraphics[width=0.25\textwidth]{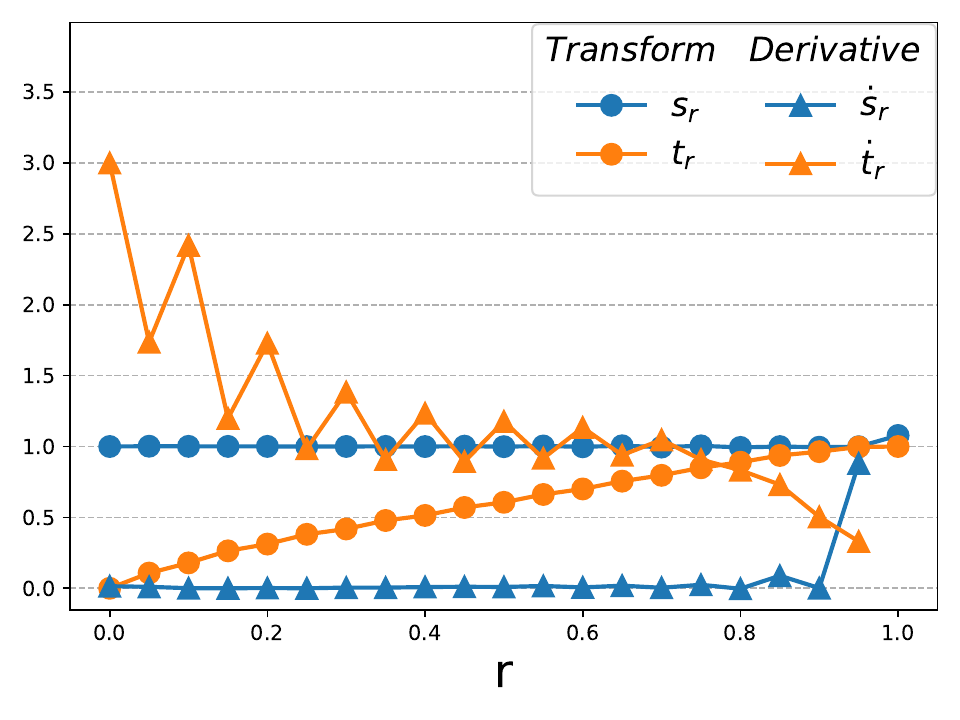}\\
     \raisebox{2.5\height}{\rotatebox[origin=c]{90}{\scriptsize FM/$v$-CS}}\includegraphics[width=0.25\textwidth]{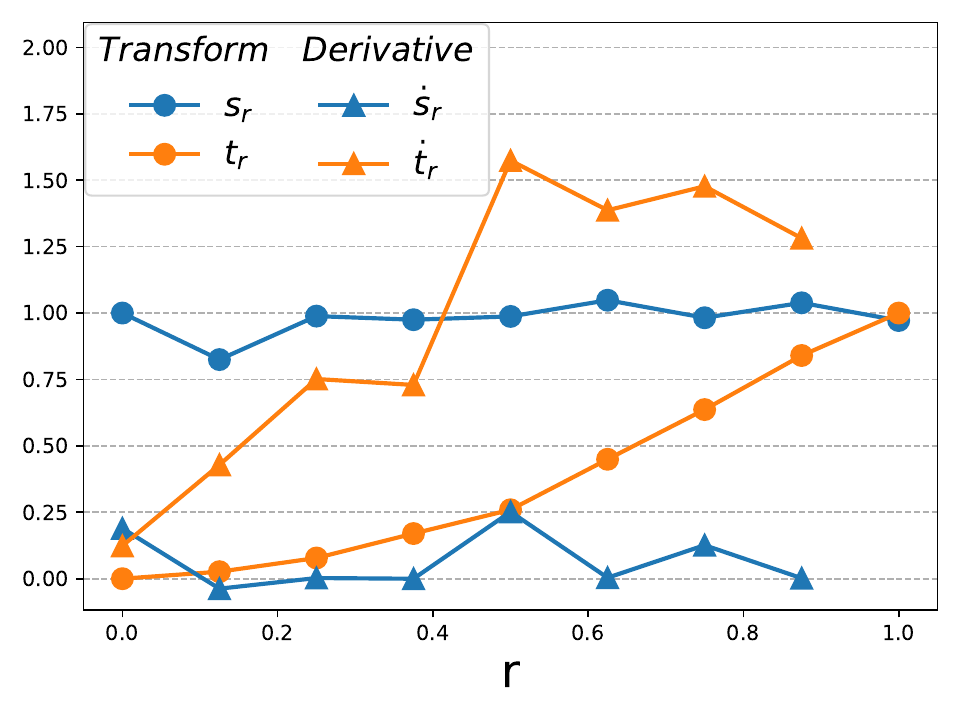} & \includegraphics[width=0.25\textwidth]{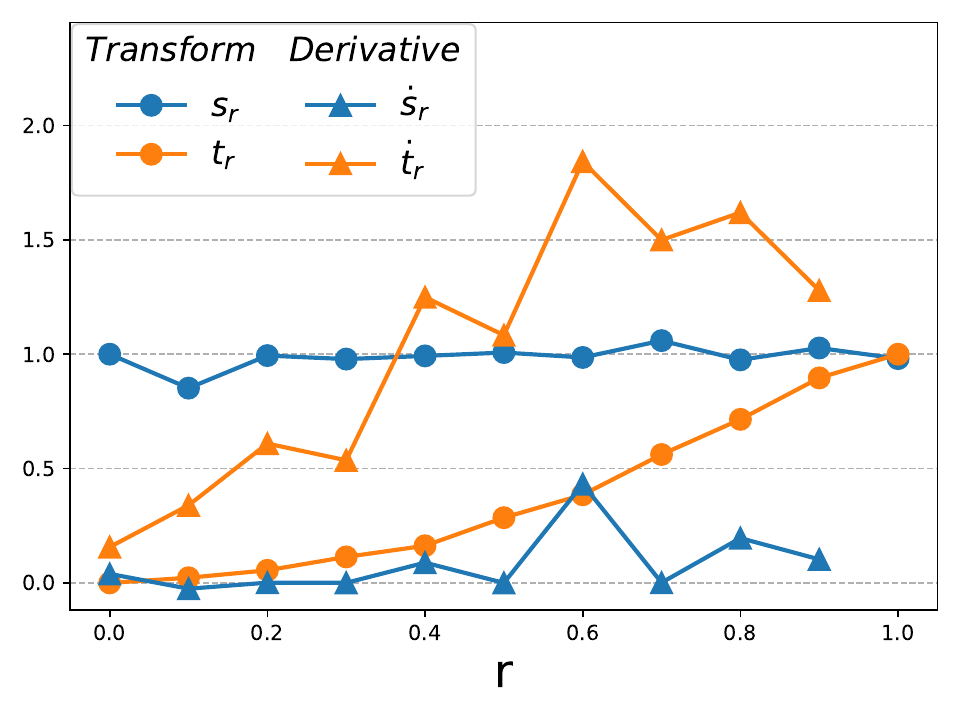} & \includegraphics[width=0.25\textwidth]{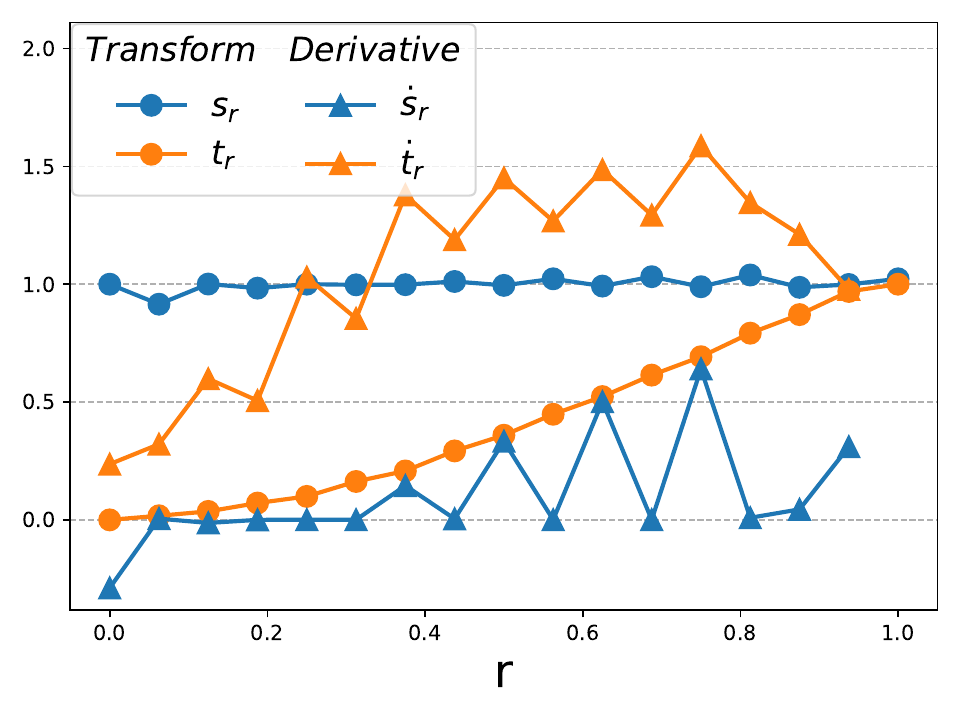} &
    \includegraphics[width=0.25\textwidth]{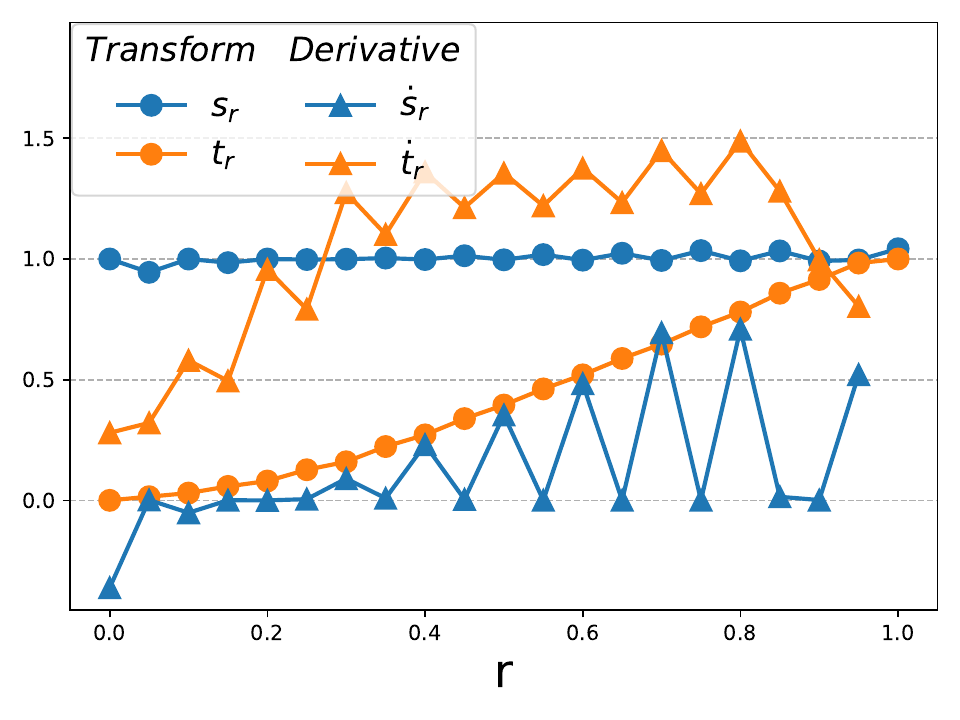}\\
    \raisebox{3\height}{\rotatebox[origin=c]{90}{\scriptsize FM-OT}}\includegraphics[width=0.25\textwidth]{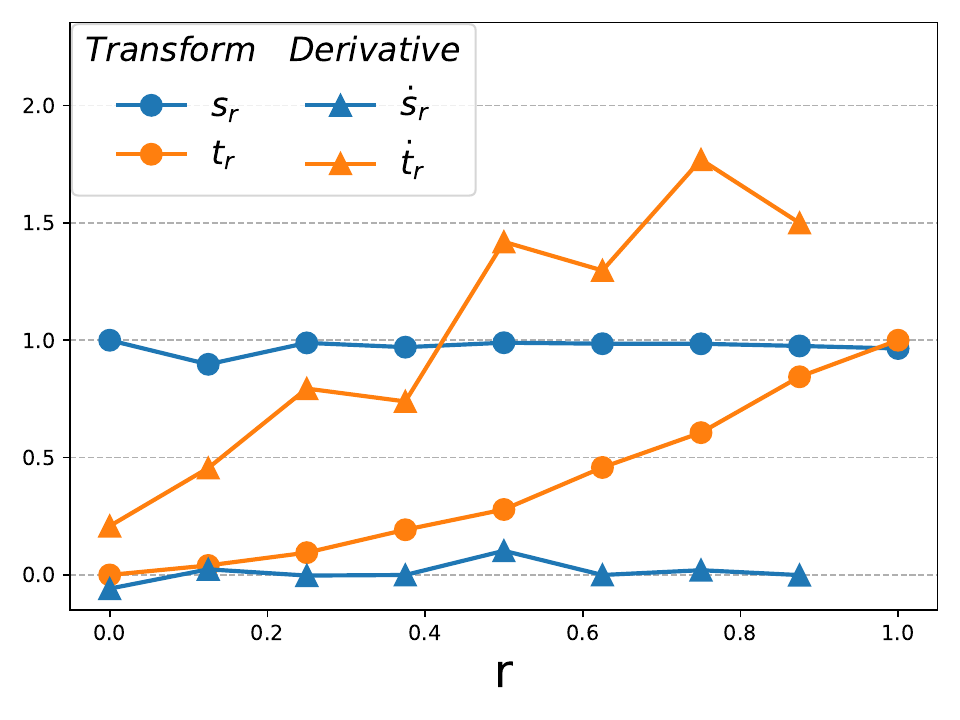} & \includegraphics[width=0.25\textwidth]{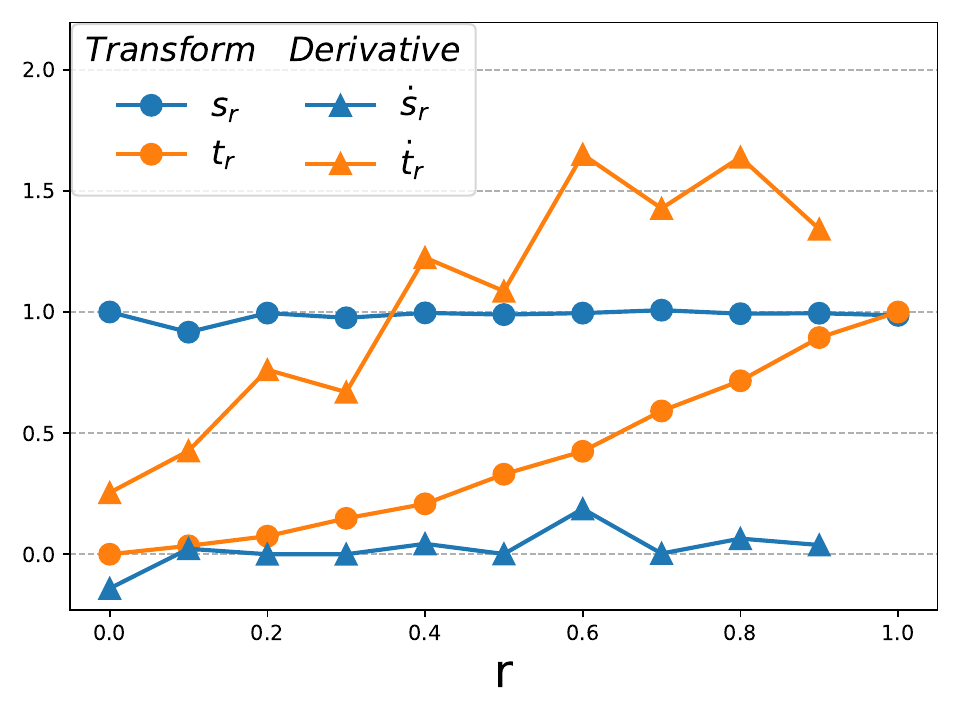} & \includegraphics[width=0.25\textwidth]{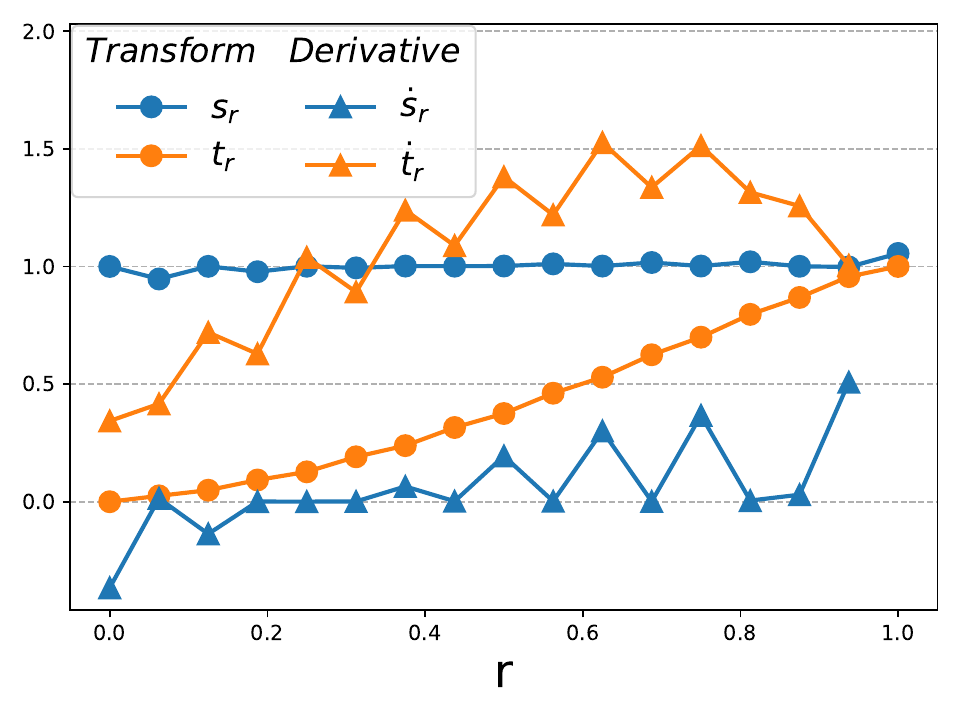} &
    \includegraphics[width=0.25\textwidth]{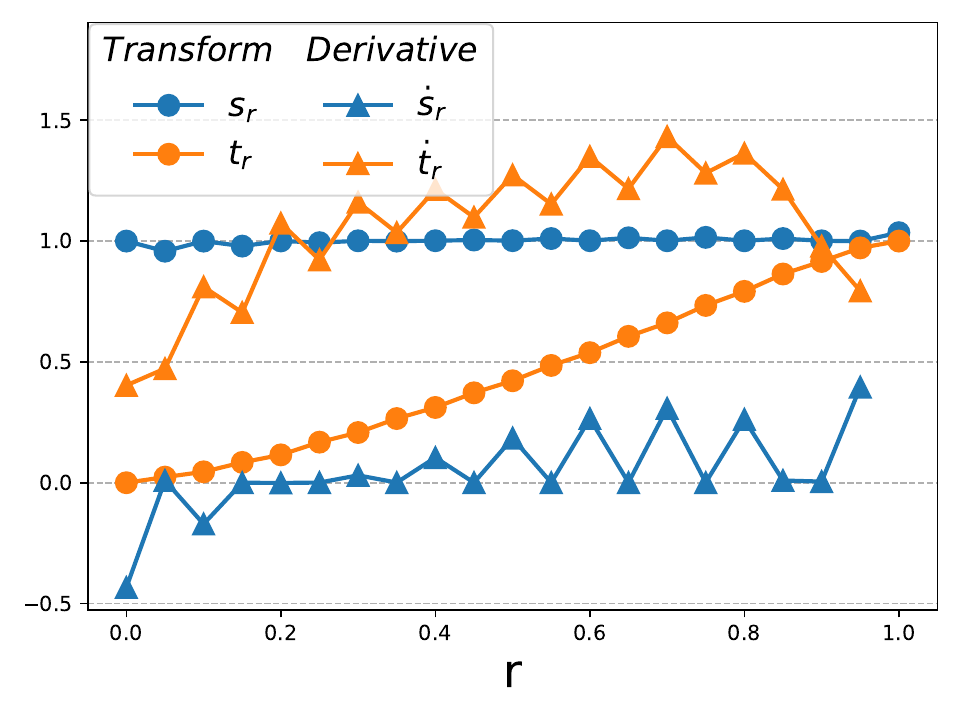}\\
\end{tabular}    
    \caption{Trained $\theta$ of Bespoke-RK2 solvers on ImageNet-64 for NFE 8/10/16/20; $\eps$-VP (top), FM/$v$-CS (middle), and FM-OT (bottom).}
    \label{fig_a:scheme_imagenet64}
\end{figure}
\begin{figure}[h!]
    \centering
    \begin{tabular}{@{\hspace{0pt}}c@{\hspace{0pt}}c@{\hspace{0pt}}c@{\hspace{0pt}}c@{\hspace{0pt}}}
    {\quad \ \scriptsize NFE=8} & {\quad \ \scriptsize NFE=10} & {\quad \ \scriptsize NFE=16}  & {\quad \ \scriptsize NFE=20} \\
    \raisebox{3.5\height}{\rotatebox[origin=c]{90}{\scriptsize $\eps$-pred}}\includegraphics[width=0.25\textwidth]{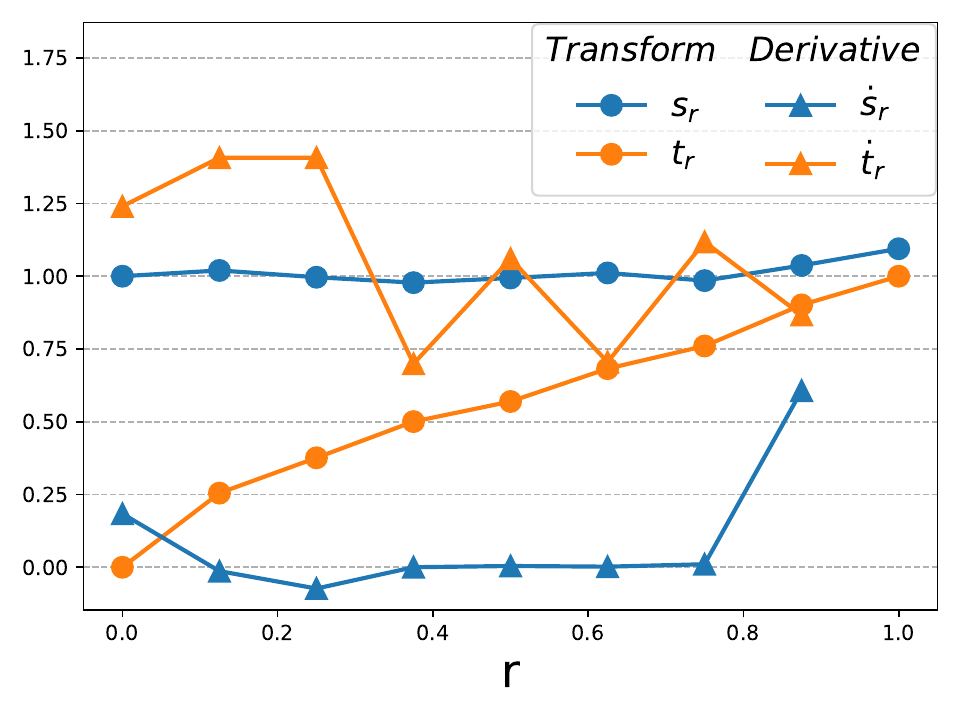} & \includegraphics[width=0.25\textwidth]{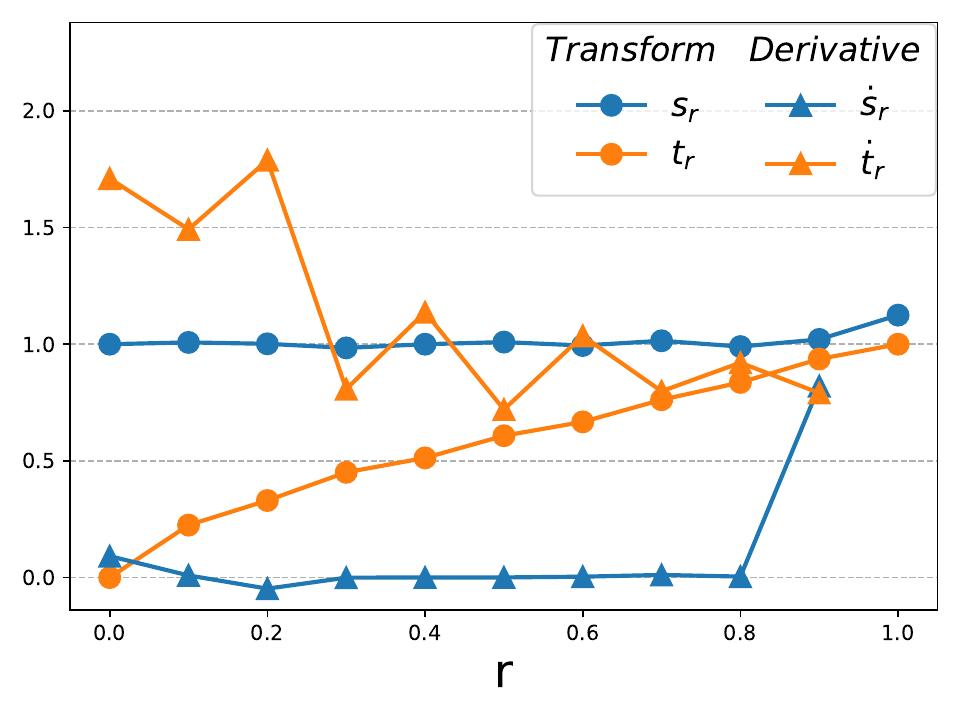} & \includegraphics[width=0.25\textwidth]{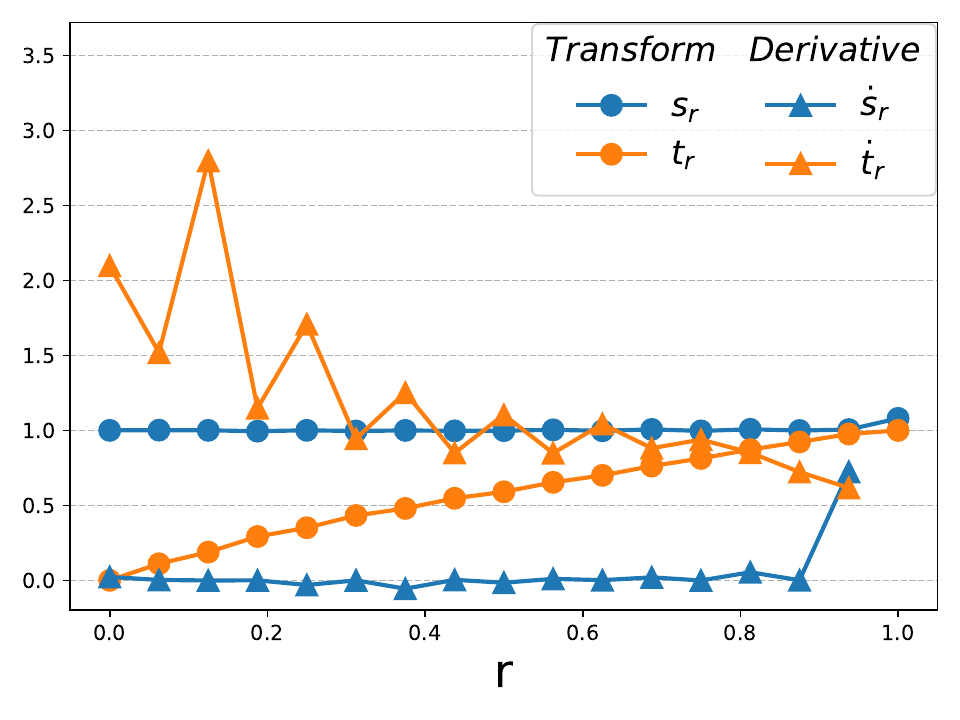} &
    \includegraphics[width=0.25\textwidth]{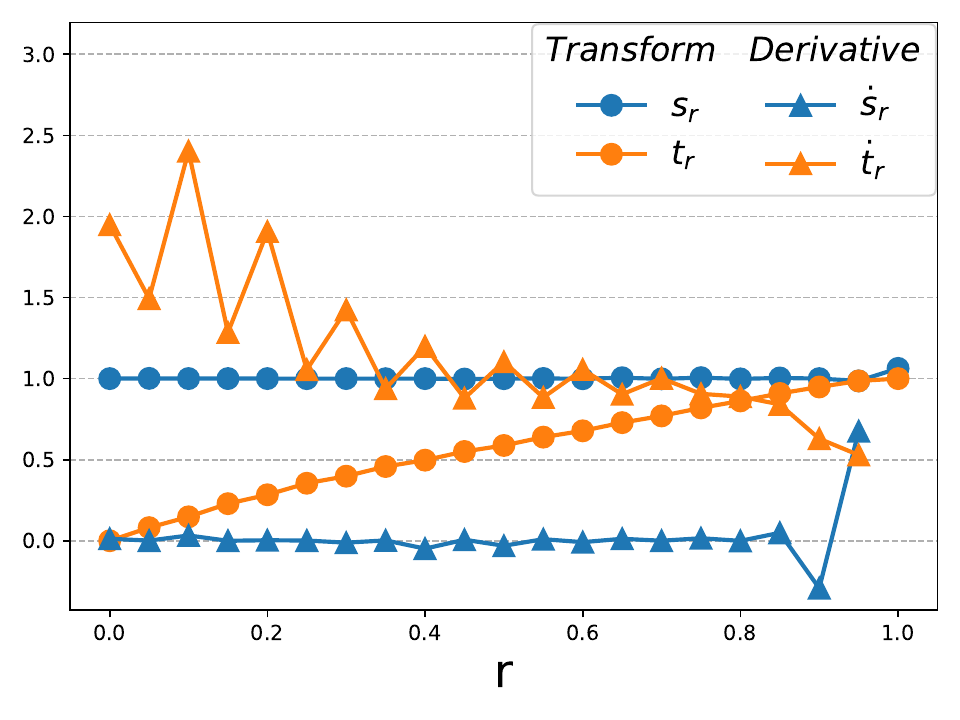}\\
     \raisebox{2.5\height}{\rotatebox[origin=c]{90}{\scriptsize FM/$v$-CS}}\includegraphics[width=0.25\textwidth]{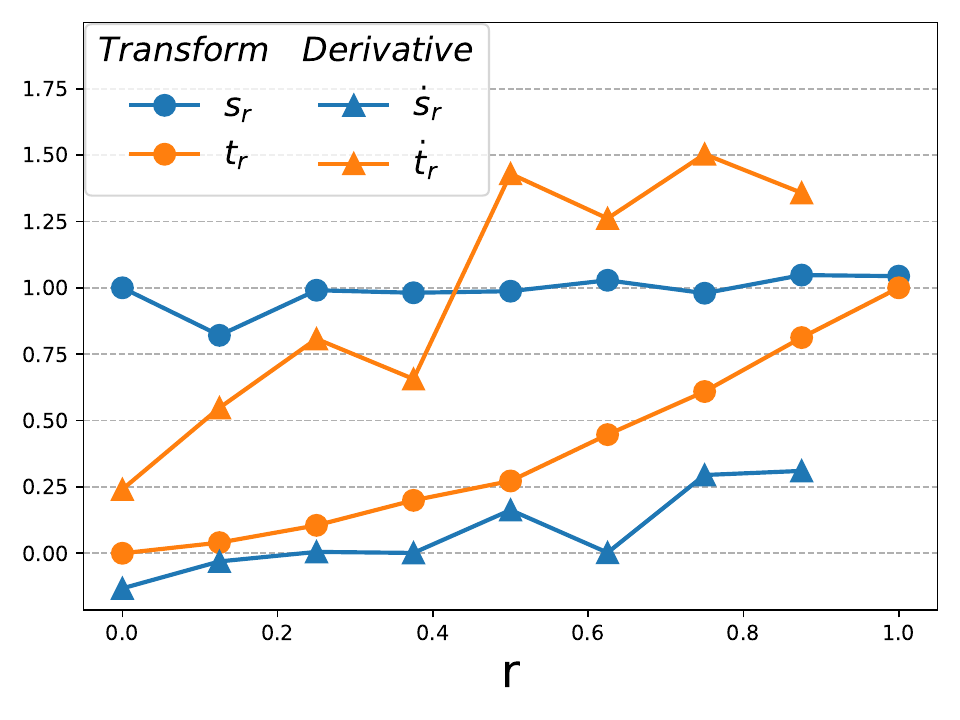} & \includegraphics[width=0.25\textwidth]{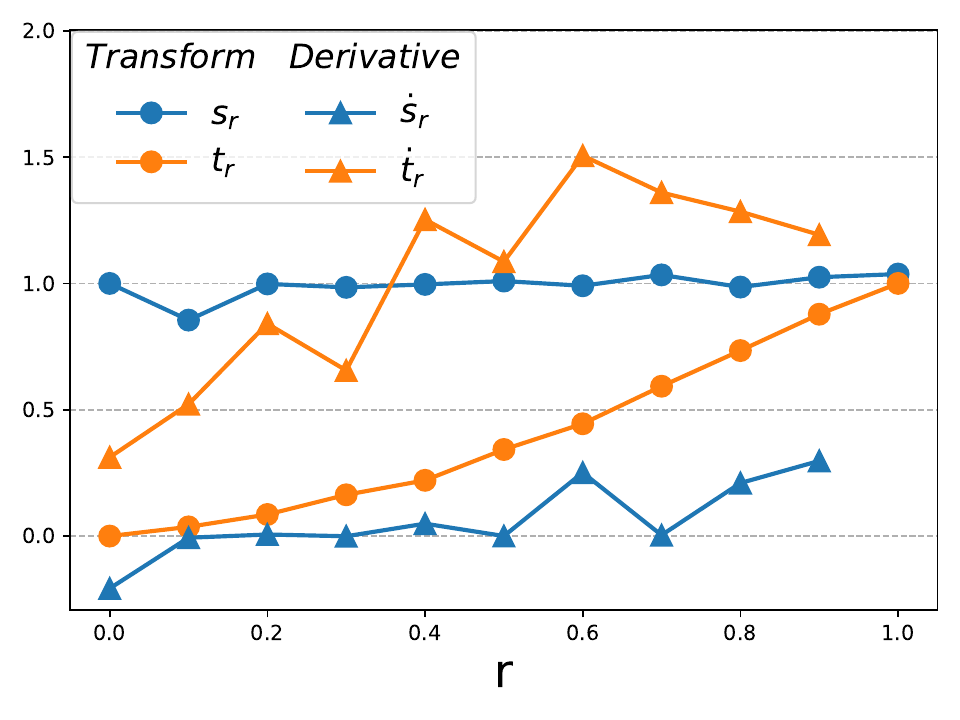} & \includegraphics[width=0.25\textwidth]{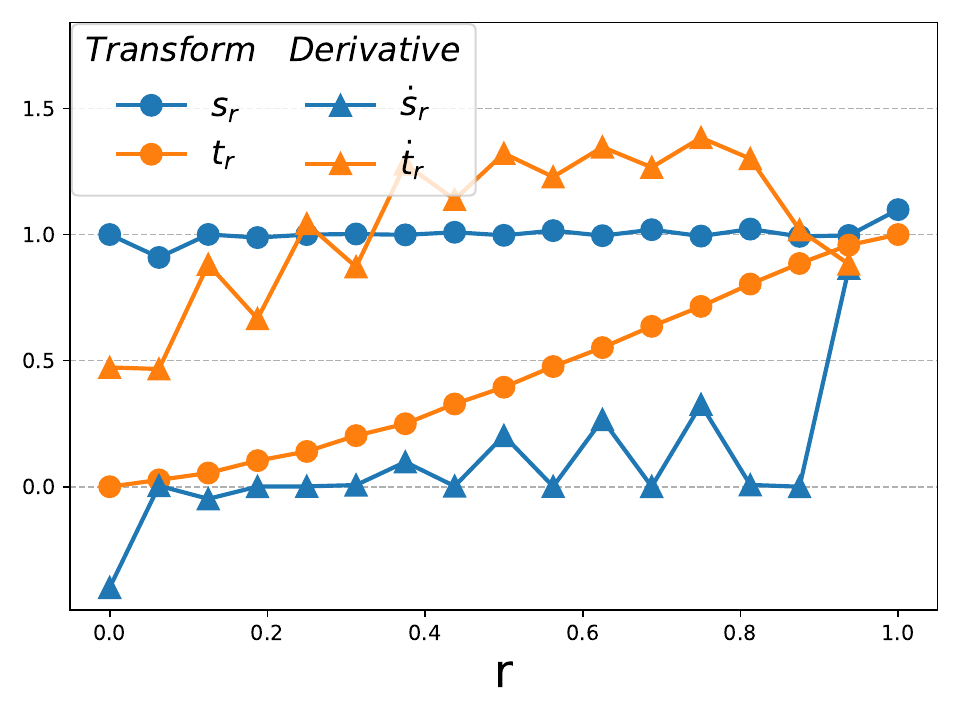} &
    \includegraphics[width=0.25\textwidth]{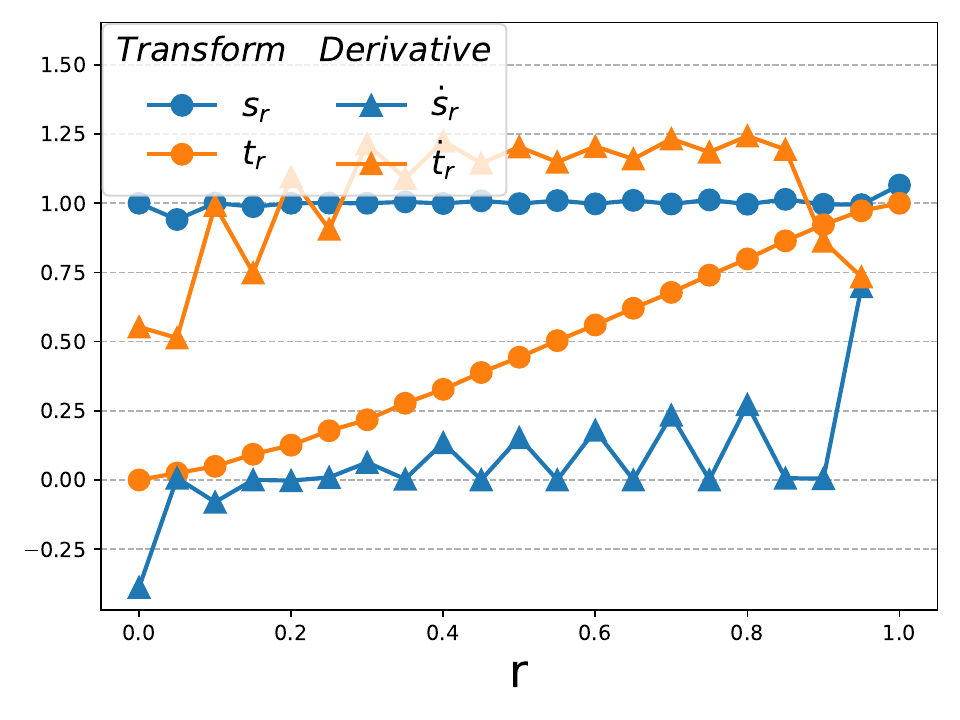}\\
    \raisebox{3\height}{\rotatebox[origin=c]{90}{\scriptsize FM-OT}}\includegraphics[width=0.25\textwidth]{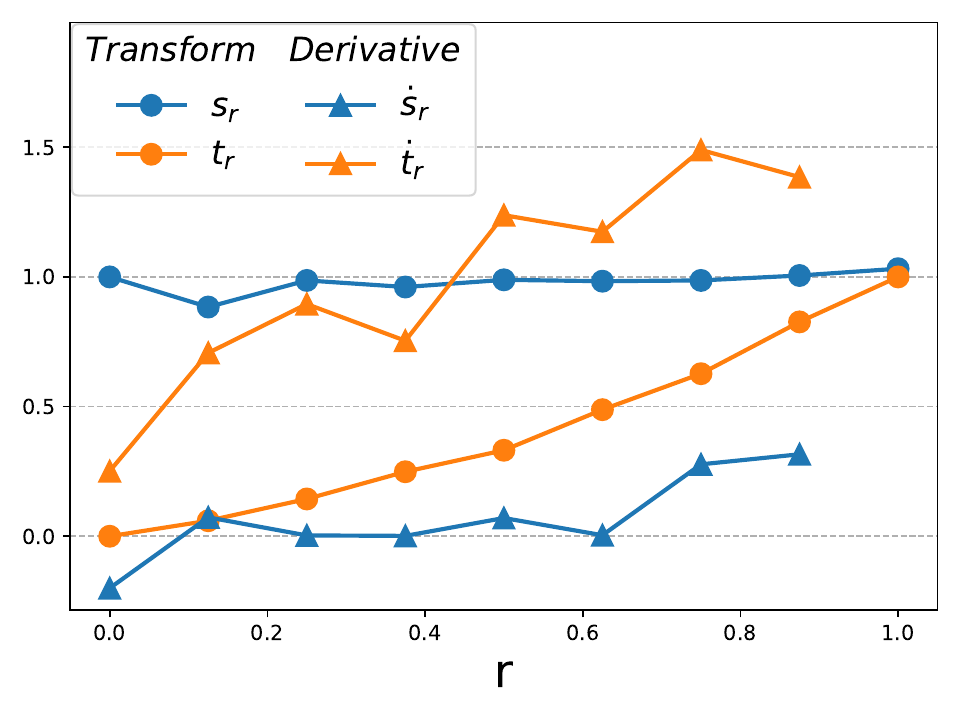} & \includegraphics[width=0.25\textwidth]{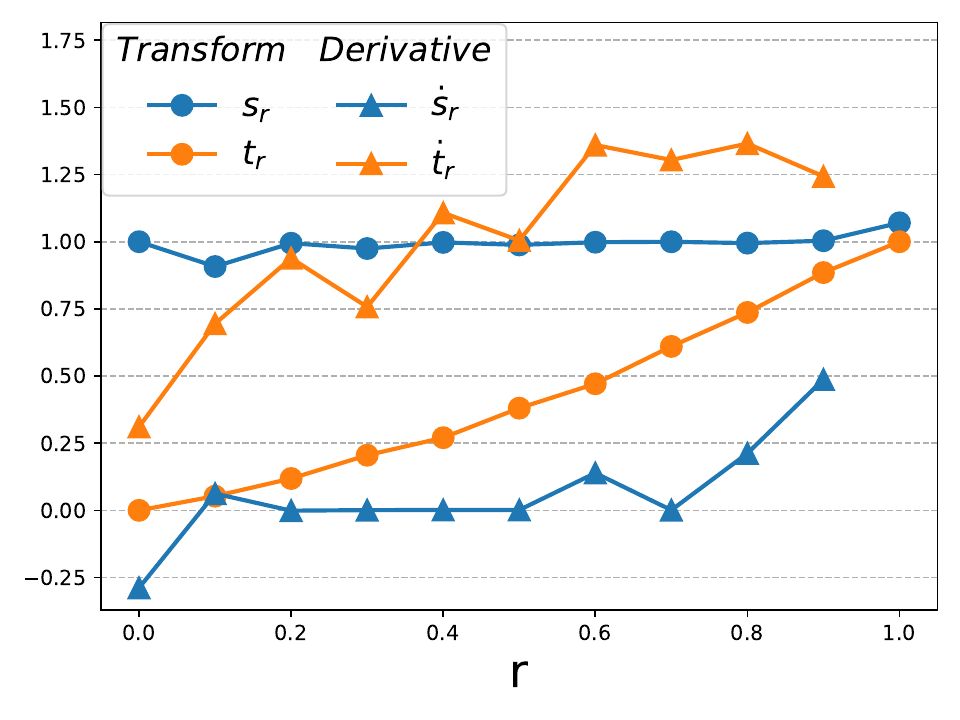} & \includegraphics[width=0.25\textwidth]{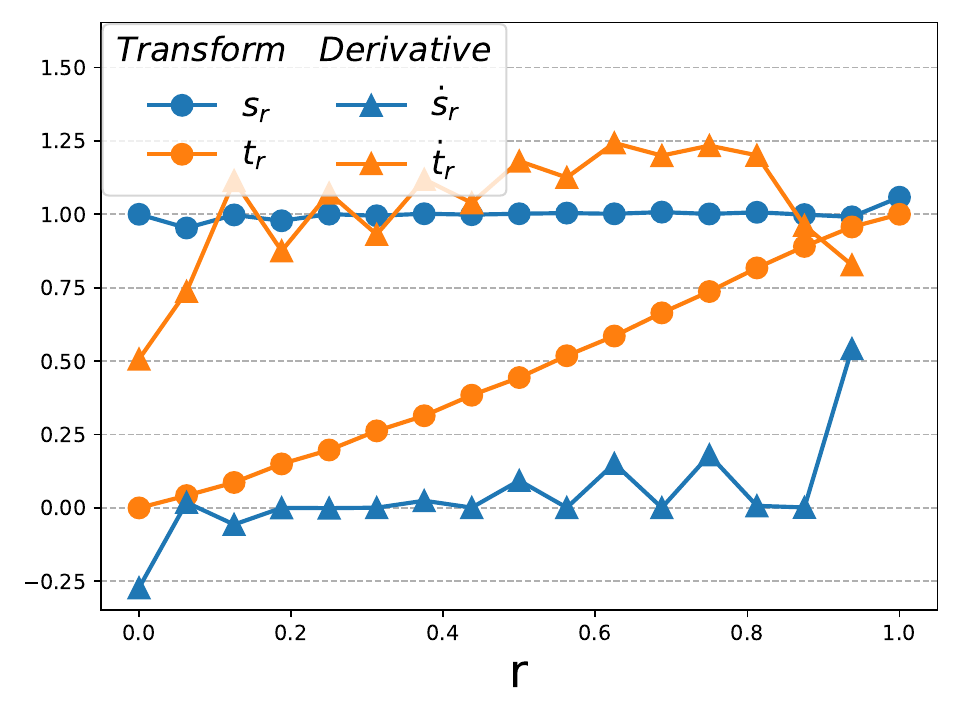} &
    \includegraphics[width=0.25\textwidth]{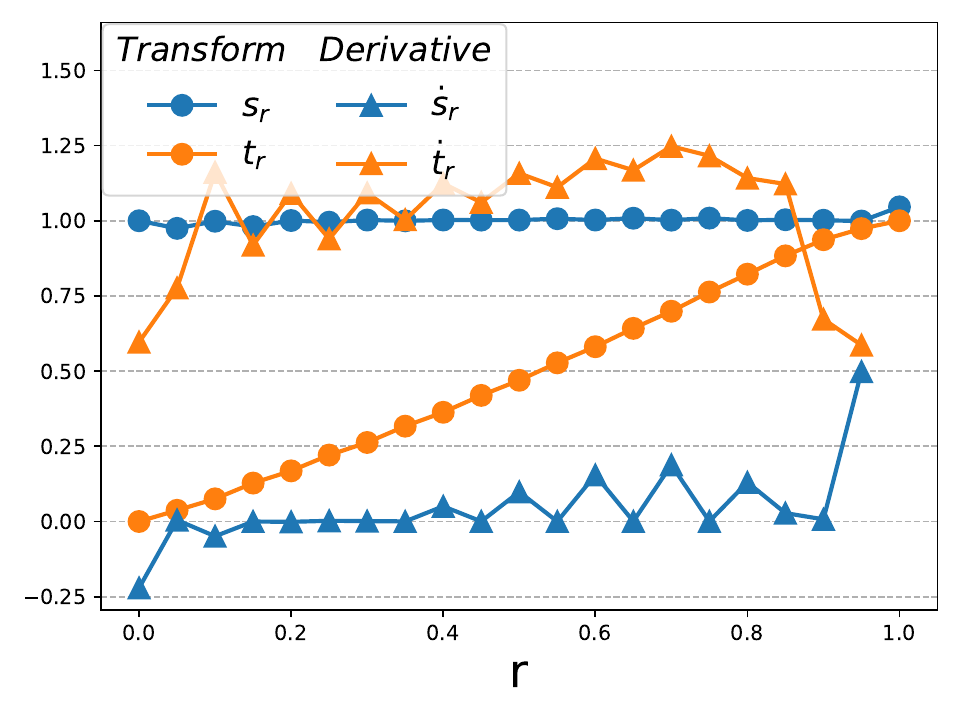}\\
\end{tabular}    
    \caption{Trained $\theta$ of Bespoke-RK2 solvers on CIFAR10 for NFE 8/10/16/20; $\eps$-pred (top), FM/$v$-CS (middle), and FM-OT (bottom).}
    \label{fig_a:scheme_cifar10}
\end{figure}

\section{Pre-trained Models}
All our FM-OT and FM/$v$-CS models were trained with Conditional Flow Matching (CFM) loss derived in \cite{lipman2022flow},
\begin{equation}
    \gL_{\text{CFM}}(\theta) = \E_{t, p_0(x_0), q(x_1)}\norm{v_t(x_t;\theta) - (\dot{\sigma}_tx_0 + \dot{\alpha_t}x_1)}^2,
\end{equation}
where $t \sim \gU([0,1])
$, $p_0(x_0)=\gN\parr{x_0|0,I}$, $q(x_1)$ is the data distribution, $v_t(x_t;\theta)$ is the network, $(\alpha_t, \sigma_t)$ is the scheduler as defined in \eqref{e:def_scheduler}, and $x_t=\sigma_tx_0 + \alpha_tx_1$. For FM-OT the scheduler is
\begin{equation}
    \alpha_t=t, \quad \sigma_t=1-t,
\end{equation}
and for FM/$v$-CS the scheduler is 
\begin{equation}
    \alpha_t=\sin\frac{\pi}{2}t, \quad \sigma_t=\cos\frac{\pi}{2}t.
\end{equation}
All our $\eps$-VP models were trained with noise prediction loss as derived in \cite{ho2020denoising} and \cite{song2020score},
\begin{equation}
    \gL_{\text{noise}}(\theta) = \E_{t, p_0(x_0), q(x_1)}\norm{\epsilon_t(x_t;\theta) -x_0}^2,
\end{equation}
where the VP scheduler is
\begin{equation}
    \quad \alpha_t= \xi_{1-t}, \quad \sigma_t = \sqrt{1-\xi_{1-t}^2},\quad \xi_{s}= e^{-\frac{1}{4}s^2(B-b) -\frac{1}{2}sb},
\end{equation}
and $B=20,\ b=0.1$. All models use U-Net architecture as in \citet{dhariwal2021diffusion}, and the hyper-parameters are listed in Table \ref{tab:training_hyper-params}.
\input{tables/training_hyper_parameters}

\newpage
\section{More results}
In this section we present more sampling results using RK2-Bespoke solvers, the RK2 baseline and Ground Truth samples (with DOPRI5).
\input{tables/afhq_appendix}
%
\input{tables/imagenet128_appendix}
\input{tables/imagenet64_appendix_fm_ot}
\input{tables/imagenet64_appendix_fm_cs}
\input{tables/imagenet64_appendix_eps_pred}

\input{tables/imagenet64_samples_DPM}
\end{document}

%% file: math_commands.tex
\usepackage{xcolor}

\newcommand{\norm}[1]{\left\Vert#1\right\Vert}

\newcommand{\abs}[1]{\left\vert#1\right\vert}

\newcommand{\set}[1]{\left\{#1\right\}}
\newcommand{\parr}[1]{\left (#1\right )}
\newcommand{\brac}[1]{\left [#1\right ]}

\newcommand{\Real}{\mathbb R}

\newcommand{\eps}{\varepsilon}
\newcommand{\too}{\rightarrow}

\definecolor{mygray}{gray}{0.95}

\newcommand{\FM}{\scriptscriptstyle \text{FM}}

\usepackage{xspace}
\newcommand*{\eg}{{\it e.g.}\@\xspace}
\newcommand*{\ie}{{\it i.e.}\@\xspace}

\usepackage{amsmath,amsfonts,bm}

\def\eqref#1{equation~\ref{#1}}

\def\1{\bm{1}}

\def\eps{{\epsilon}}

\DeclareMathAlphabet{\mathsfit}{\encodingdefault}{\sfdefault}{m}{sl}
\SetMathAlphabet{\mathsfit}{bold}{\encodingdefault}{\sfdefault}{bx}{n}

\def\gF{{\mathcal{F}}}

\def\gL{{\mathcal{L}}}

\def\gN{{\mathcal{N}}}

\def\gU{{\mathcal{U}}}

\newcommand{\E}{\mathbb{E}}

\newcommand{\R}{\mathbb{R}}

\usepackage{tabularx,ragged2e,booktabs,caption}
\newcolumntype{C}[1]{>{\Centering}m{#1}}

\newcolumntype{Z}[1]{>{\Left}m{#1}}

%% file: tables/ode_solver.tex
\begin{wrapfigure}[9]{R}{0.38\textwidth}
\vspace{-25pt}
    \begin{minipage}{0.37\textwidth}
      \begin{algorithm}[H]
        \caption{Numerical ODE solver.}\label{alg:odesolve}
        \begin{algorithmic}
            \Require $t_0,x_0$
            \For{$i=0,1,\ldots,n-1$}
            \State $(t_{i+1},x_{i+1}) = \odestep(t_i,x_i;u_t)$  
        \EndFor
        \State \Return $x_n$
        \end{algorithmic}
      \end{algorithm}
    \end{minipage}
  \end{wrapfigure}

%% file: tables/bespoke_pseudocode.tex
\begin{wrapfigure}[24]{R}{0.47\textwidth}
\vspace{-24pt}
    \begin{minipage}{0.45\textwidth}
      \begin{algorithm}[H]
        \caption{Bespoke training.}\label{alg:bes_training}
        \begin{algorithmic}
            \Require pre-trained $u_t$, number of steps $n$
            \State initialize $\theta\in \Real^p$         
            \While{not converged}
            \State $x_0 \sim p(x_0)$ \Comment{{\color{cyan}sample noise}}  
            \State $x(t) \gets$ solve ODE \ref{e:ode} \Comment{{\color{cyan}GT path}}  
            \State $\gL \gets 0$ \Comment{{\color{cyan}init loss}}
            \For{$i=0,...,n-1$} 
            \State $x^\theta_{i+1}\gets\odestep_x^{\theta}\parr{x^{\aux}_i(t_{i}),t_{i};u_t}$
            \State $\gL \hspace{-1pt} \mathrel{{+}{=}} \hspace{-1pt} M^\theta_{i+1}\norm{x_{i+1}^{\aux}(t_{i+1}) - x^\theta_{i+1}}$
            \EndFor
            \State $\theta \gets \theta -\gamma \nabla_\theta\gL$  \Comment{{\color{cyan}optimization step}}
            \EndWhile
        \State \Return $\theta$
        \end{algorithmic}
      \end{algorithm}
    \end{minipage}
    \begin{minipage}{0.45\textwidth}      
      \begin{algorithm}[H]
        \caption{Bespoke sampling.}\label{alg:bes_sampling}
        \begin{algorithmic}
            \Require pre-trained $u_t$, trained $\theta$ 
            \State $x_0 \sim p(x_0)$ \Comment{{\color{cyan}sample noise}} 
            \State $r_0\gets 0$,\ $\bar{x}_0\gets x_0$ \Comment{{\color{cyan}initial conditions}} 
            \For{$i=0,1,\ldots,n-1$}
            \State $(r_{i+1},\bar{x}_{i+1}) \gets \odestep(r_i,\bar{x}_i;\bar{u}^\theta_r)$  
        \EndFor
        \State \Return $\varphi^{-1}_1(\bar{x}_n)$
        \end{algorithmic}
      \end{algorithm}
    \end{minipage}
  \end{wrapfigure}

%% file: sections/previous_work.tex
\section{Previous Work}

Diffusion models \citep{sohl2015deep,ho2020denoising} are a powerful paradigm for generative models that for sampling require solving a Stochastic Differential Equation~(SDE), or its associated ODE, describing a (deterministic) flow process \citep{song2020denoising}. Diffusion models have been generalized to paradigms directly aiming to learn a deterministic flow \citep{lipman2022flow,albergo2022si,liu2022flow}. Flow-based models are efficient to train but costly to sample. Previous works had tackled the sample complexity of flow models by building \emph{dedicated solver schemes} and \emph{distillation}.  


\textbf{Dedicated Solvers.}
This line of works introduced specialized ODE solvers exploiting the structure of the sampling ODE. \citet{lu2022dpm,zhang2022fast} utilize the semi-linear structure of the score/$\eps$-based sampling ODE to adopt a method of exponential integrators. 
\citep{zhang2023improved} further introduced refined error conditions to fulfill desired order conditions and achieve better sampling, while \citet{lu2022dpm-pp} adapted the method to guided sampling. 
\citet{karras2022elucidating} suggested transforming the ODE to sample a different Gaussian Path for more efficient sampling, while also suggesting non-uniform time steps. 

In principle, all of these methods effectively proposed---based on intuition and heuristics---to apply a particular scale-time transformation to the sampling trajectories of the pre-trained model for more efficient sampling, while our Bespoke solvers search over the entire space of scale-time transformation for the \textit{optimal} transformation of a particular trained model.

Other works also aimed at learning the solver: \citet{dockhorn2022genie}~(GENIE) introduced a higher-order solver, and distilled the necessary JVP for their method; \citet{watson2021learning} (DDSS) optimized a perceptual loss considering a family of generalized Gaussian diffusion models; \citet{lam2021bilateral} improved the denoising process using bilateral filters, thereby indirectly affecting the efficiency of the ODE solver; \citet{duan2023optimal} suggested to learn a solver for diffusion models by replacing every other function evaluation by a linear subspace projection. 
Our Bespoke Solvers belong to this family of learnt solvers, however, they are consistent by construction~(Theorem ~\ref{thm:consistency}) and minimize a bound on the solution error~(for the appropriate Lipschitz constant parameter).

\textbf{Distillation.} 
Distillation techniques aim to simplify sampling from a trained model by fine-tuning or training a new model to produce samples with fewer function evaluations. \citet{luhman2021knowledge} directly regressed the trained model's samples, while \input{tables/rk1_rk2}\citet{salimans2022progressive,meng2023distillation} built a sequence of models each reducing the sampling complexity by a factor of 2. \citet{song2023consistency} distilled a consistency map that enables large time steps in the probability flow; \citet{liu2022flow} retrained a flow-based method based on samples from a previously trained flow. \citet{yang2023diffusion} used distillation to reduce model size while maintaining the quality of the generated images. The main drawbacks of distillation methods is their long training time \citep{salimans2022progressive}, and lack of consistency, \ie, they do not sample from the distribution of the pre-trained model. \vspace{-0pt}

%% file: tables/rk1_rk2.tex
\begin{wrapfigure}[10]{r}{0.26\textwidth}
  \begin{center}\vspace{-10pt}
    \includegraphics[width=0.25\textwidth]{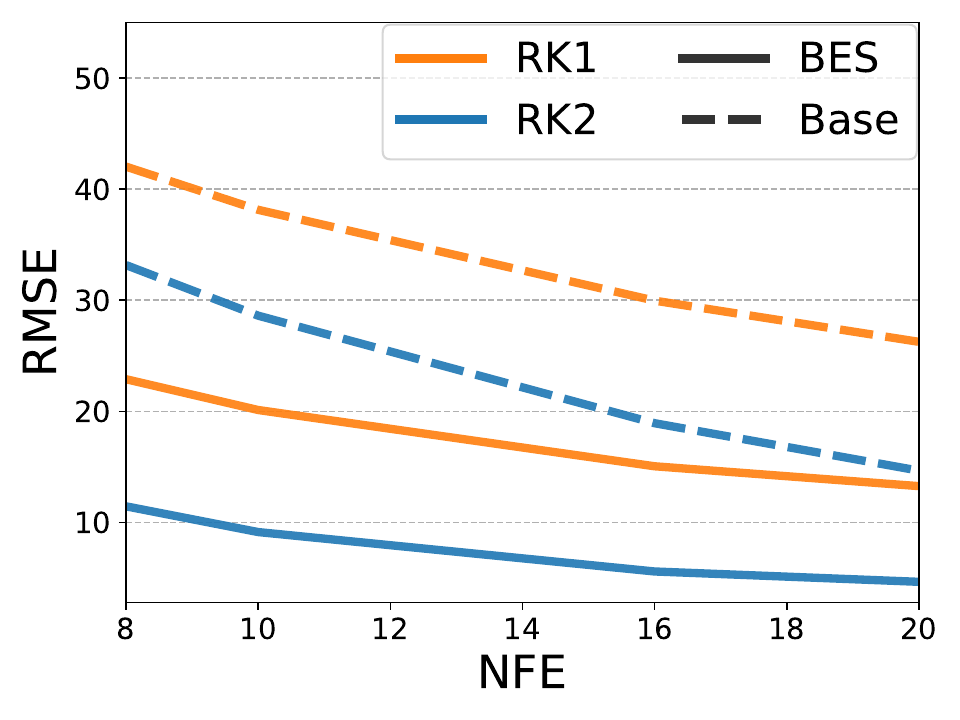}\vspace{-15pt}
  \end{center}
  \caption{Bespoke RK1/2 solvers on ImageNet-64 FM-OT.}
  \label{fig:euler_vs_midpoint}
\end{wrapfigure}

%% file: sections/experiments.tex
\section{Experiments}\vspace{-0pt}


\input{tables/cifar_edm}
\textbf{Models and datasets.} Our method works with pre-trained models: we use the pre-trained CIFAR10~\citep{krizhevsky2009learning} model of \citep{song2020score} with published weights from EDM \citep{karras2022elucidating}. Additionally, we trained diffusion/flow models on the datasets: CIFAR10, AFHQ-256~\citep{choi2020stargan} and ImageNet-64/128~\citep{deng2009imagenet}. Specifically, for ImageNet, as recommended by the authors \citep{imagenet_website} we used the official \emph{face-blurred} data (64$\times$64 downsampled using the open source preprocessing scripts from \citet{chrabaszcz2017downsampled}). For diffusion models, we used an $\eps$-Variance Preserving ($\eps$-VP) parameterization and schedule \citep{ho2020denoising,song2020score}. For flow models, we used Flow Matching~\citep{lipman2022flow} with Conditional Optimal Transport (FM-OT), and Flow Matching/$v$-prediction with Cosine Scheduling (FM/$v$-CS) \citep{salimans2022progressive,albergo2022si}. Note that Flow Matching  methods directly provide the velocity vector field $u_t(x)$, and we converted $\eps$-VP to a velocity field using the identity in \citet{song2020score}. For conditional sampling we apply classifier free guidance~\citep{ho2022classifier}, so each evaluation uses two forward passes.  

\input{tables/table_cifar_10} 
\textbf{Bespoke hyper-parameters and optimization.} 
As our base ODE solvers, we tested RK1 (Euler) and RK2 (Midpoint). Furthermore, we have two hyper-parameters $n$ -- the number of steps, and $L_\tau$ -- the Lipschitz constant from lemmas \ref{lem:bound_L_euler}, \ref{lem:bound_L_midpoint}. We train our models with $n \in \set{4,5,8,10}$ steps and fix $L_\tau=1$. Ground Truth (GT) sample trajectories, $x(t_i)$, are computed with an adaptive RK45 solver~\citep{shampine1986some}. We compute FID~\citep{heusel2017GANs} and validation RMSE (\eqref{e:rmse}) is computed on a set of 10K fresh noise samples $x_0\sim p(x_0)$; Figure \ref{fig:rmse_vs_iteration_imagenet} depicts an example of RMSE vs.~training iterations for different $n$ values. Unless otherwise stated, below we report results on best FID iteration and show samples on best RMSE validation iteration. Figures \ref{fig_a:scheme_imagenet128}, \ref{fig_a:scheme_imagenet64}, \ref{fig_a:scheme_cifar10} depict the learned Bespoke solvers' parameters $\theta$ for the experiments presented below; note the differences across the learned schemes for different models and datasets.  

\textbf{Bespoke RK1 vs.~RK2.} We compared RK1 and RK2 and their Bespoke versions on CIFAR10 and ImageNet-64 models (FM-OT and FM/$v$-CS). Figure \ref{fig:euler_vs_midpoint} and Figures \ref{fig_a:rk1_vs_rk2_cifar10}, \ref{fig_a:rk1_vs_rk2_imagenet64} show best validation RMSE (and corresponding PSNR). Using the same budget of function evaluations RK2/RK2-Bespoke produce considerably lower RMSE validation compared to RK1/RK1-Bespoke, respectively. We therefore opted for RK2/RK2-Bespoke for the rest of the experiments below. 

\textbf{CIFAR10.} We tested our method on the pre-trained CIFAR10 $\eps$-VP model \citep{song2020score} that was released by EDM~\citep{karras2022elucidating}. In Figure \ref{fig:bes_vs_edm}, we compare our RK2-Bespoke solver to the EDM method, which corresponds to a particular choice of scaling, $s_i$, and time step discretization, $t_i$. Euler and EDM curves computed as originally implemented in EDM, where the latter achieves FID=$3.05$ at 35 NFE, comparable to the result reported by EDM. Using our RK2-Bespoke Solver, we achieved an FID of $2.99$ with 20 NFE, providing a $42\%$ reduction in NFE. Additionally, we tested our method on three models we trained ourselves on CIFAR10, namely $\eps$-VP, FM/$v$-CS, and FM-OT. Table \ref{tab:cifar10_comprison} compares our best FID for each model with different baselines demonstrating superior generation quality for low NFE among all dedicated solvers; \eg, for NFE=10 we improve the FID of the runner-up by over 34\%  (from 4.17 to 2.73) using RK2-Bespoke FM-OT model. Table \ref{tab:cifar10_a} lists best FID values for different NFE, along with the ground truth FID for the model and the fraction of time Bespoke training took compared to the original model's training time; with 20 NFE, our RK2-Bespoke solvers achieved FID within 8\%, 1\%, 1\% (resp.)~of the GT solvers' FID. Although close, our Bespoke solver does not match distillation's performance, however our approach is much faster to train, requiring $\sim$1\% of the original GPU training time with our naive implementation that re-samples the model at each iteration. Figure \ref{fig:nfe_vs_fid_cifar10} shows FID/RMSE/PSNR vs.~NFE, where PSNR is computed w.r.t.~the GT solver's samples.

\input{tables/graphs_imagenet}

\input{tables/imagenet_table}
\textbf{ImageNet 64/128.}
We further experimented with the more challenging ImageNet-64$\times$64 / 128$\times$128 datasets. For ImageNet-64 we also trained 3 models as described above. For ImageNet-128, due to computational budget constraints, we only trained FM-OT (training requires nearly 2000 GPU days). Figure \ref{fig:nfe_vs_fid_main} compares RK2-Bespoke to various baselines including DPM $2^{\text{nd}}$ order~\citep{lu2022dpm}. As can be seen in the graphs, the Bespoke solvers improve both FID and RMSE. Interestingly, the Bespoke sampling takes all methods to similar RMSE levels, a fact that can be partially explained by Theorem \ref{thm:equivalence}. In Table \ref{tab:imagenet_64_and_128}, similar to Table \ref{tab:cifar10_a}, we report best FID per NFE for the Bespoke solvers we trained, the GT FID of the model, the \% from GT achieved by the Bespoke solver, and the fraction of GPU time (in \%) it took to train this Bespoke solver compared to training the original pre-trained model. Lastly, Figures \ref{fig:imagenet-64}, \ref{fig:imagenet-128}, \ref{fig:imagenet64_fm_ot_a}, \ref{fig:imagenet64_fm_cs_a}, \ref{fig:imagenet64_eps_pred_a}, \ref{fig:imagenet_128_a_1}, \ref{fig:imagenet_128_a_2} depict qualitative sampling examples for RK2-Bespoke and RK2 solvers. Note the significant improvement of fidelity in the Bespoke samples to the ground truth. 

\textbf{AFHQ-256.} We tested our method on the AFHQ dataset~\citep{choi2020starganv2} resized to 256$\times$256 where as pre-trained model we used a FM-OT model we trained as described above. Figure \ref{fig:nfe_vs_metrics_afhq} depicts PSNR/RMSE curves for the RK2-Bespoke solvers and baselines, and Figures \ref{fig:imagenet-128} and \ref{fig:afhq_a} show qualitative sampling examples for RK2-Bespoke and RK2 solvers. Notice the high fidelity of the Bespoke generation samples. 


\textbf{Ablations.} We conducted two ablation experiments. First, Figure \ref{fig_a:scale_time_ablation} shows the effect of training only time transform (keeping $s_r\equiv 1$) and scale transformation (keeping $t_r=r$). Note that although the time transform is more significant than scale transform, we find that incorporating scale improves RMSE for low NFE (which aligns with Theorem \ref{thm:consistency}), and improve FID. Second, Figure \ref{fig_a:transferred_ablation} shows application of RK2-Bespoke solver trained on ImageNet-64 applied to ImageNet-128. The transferred solver, while sub-optimal compared to the Bespoke solver, still considerably improves the RK2 baseline in RMSE and PSNR, and improves FID for higher NFE (16,20). Reusing Bespoke solvers can potentially be a cheap option to improve solvers. \vspace{-5pt}

    

\input{tables/imagenet64_samples}

\input{tables/imagenet128_afhq_samples}

%% file: tables/cifar_edm.tex
\begin{wrapfigure}[12]{r}{0.26\textwidth}
\vspace{-15pt}
  \begin{center}
    \includegraphics[width=0.25\textwidth]{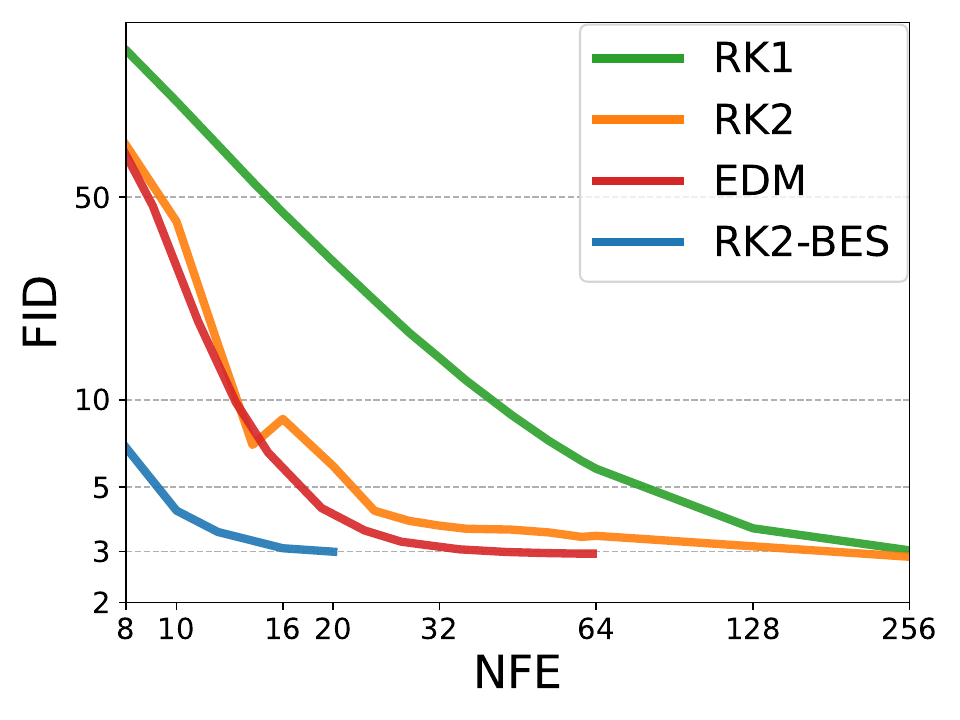}\vspace{-15pt}
  \end{center}
  \caption{Bespoke solver applied to EDM's \citep{karras2022elucidating} CIFAR10 published model. }
  \label{fig:bes_vs_edm}
\end{wrapfigure}

%% file: tables/table_cifar_10.tex
\begin{wraptable}[19]{r}{0.37\textwidth}\vspace{-15pt}
\setlength{\tabcolsep}{2.0pt}
\centering
\resizebox{0.35\textwidth}{!}{
\begin{tabular}{llcc}
\toprule
& Method                    &  NFE & FID   \\ \midrule 
\multirow{3}{*}{\raisebox{-0.9\height}{\rotatebox[origin=c]{90}{\textbf{{Distillation}}}}} &
\cite{zheng2023fast} & 1   & 3.78              \\[4pt] 
& \cite{luhman2021knowledge}      & 1   & 9.36              \\ [4pt]
& 
\begin{tabular}[t]{c}
    \cite{salimans2022progressive} \\  \\  
\end{tabular} 
&
\begin{tabular}[t]{c}
    1\\ 2\\ 8 
\end{tabular}     
& 
\begin{tabular}[t]{c}
    9.12\\ 4.51\\ 2.57
\end{tabular}    \\ \midrule 
\multirow{8}{*}{\raisebox{-2.1\height}{\rotatebox[origin=c]{90}{\textbf{Dedicated solvers}}}} & 
DDIM\citep{song2020denoising} & 
\begin{tabular}[t]{c}
    10 \\ 20  
\end{tabular}     
& 
\begin{tabular}[t]{c}
    13.36 \\ 6.84 
\end{tabular}    \\[13pt] 
& DPM~\citep{lu2022dpm} &
\begin{tabular}[t]{c}
    10  \\ 20  
\end{tabular}     
& 
\begin{tabular}[t]{c}
    4.7\\ 3.99
\end{tabular}  \\[13pt] 
& DEIS~\citep{zhang2022fast} &
\begin{tabular}[t]{c}
    10\\ 20
\end{tabular}     
& 
\begin{tabular}[t]{c}
    4.17\\ 2.86
\end{tabular}  \\[13pt] 
& GENIE~\citep{dockhorn2022genie} &
\begin{tabular}[t]{c}
    10\\  20
\end{tabular}     
& 
\begin{tabular}[t]{c}
    5.28\\ 3.94
\end{tabular}  \\[13pt]  
& DDSS~\citep{watson2021learning} &
\begin{tabular}[t]{c}
    10\\ 20
\end{tabular}     
& 
\begin{tabular}[t]{c}
    7.86\\ 4.72
\end{tabular}  \\[13pt] 

& \textbf{\textit{RK2-BES}} \qquad 
\begin{tabular}[t]{r}
     $\eps$-VP\\ $\eps$-VP
\end{tabular}
& 
\begin{tabular}[t]{c}
     10\\ 20
\end{tabular}
&
\begin{tabular}[t]{c}
     3.31\\ 2.75
\end{tabular}  \\[13pt] 
& \textbf{\textit{RK2-BES}} \qquad 
\begin{tabular}[t]{r}
     FM/$v$-CS\\FM/$v$-CS
\end{tabular}
&  
\begin{tabular}[t]{c}
     10\\ 20
\end{tabular}
&
\begin{tabular}[t]{c}
     2.89\\ 2.64
\end{tabular} \\[13pt] 
& \textbf{\textit{RK2-BES}} \qquad 
\begin{tabular}[t]{r}     
    FM-OT \\ FM-OT
\end{tabular}
& 
\begin{tabular}[t]{c}
     10\\ 20
\end{tabular}
&
\begin{tabular}[t]{c}
     \textbf{2.73}\\ \textbf{2.59}
\end{tabular} \\                      
\bottomrule
\end{tabular}
}
\vspace{-0.2cm}
\caption{CIFAR10 sampling.} 
\label{tab:cifar10_comprison}
\end{wraptable} 

%% file: tables/graphs_imagenet.tex
\begin{figure}[t]
    \centering
    \begin{tabular}{@{\hspace{0pt}}c@{\hspace{0pt}}c@{\hspace{0pt}}c@{\hspace{0pt}}c@{\hspace{0pt}}}
    {\quad \ \scriptsize ImageNet-64: $\eps$-pred} & {\quad \ \scriptsize ImageNet-64: FM/$v$-CS} & {\quad \ \scriptsize ImageNet-64: FM-OT}  & {\quad \ \scriptsize ImageNet-128: FM-OT} \\
    \includegraphics[width=0.25\textwidth]{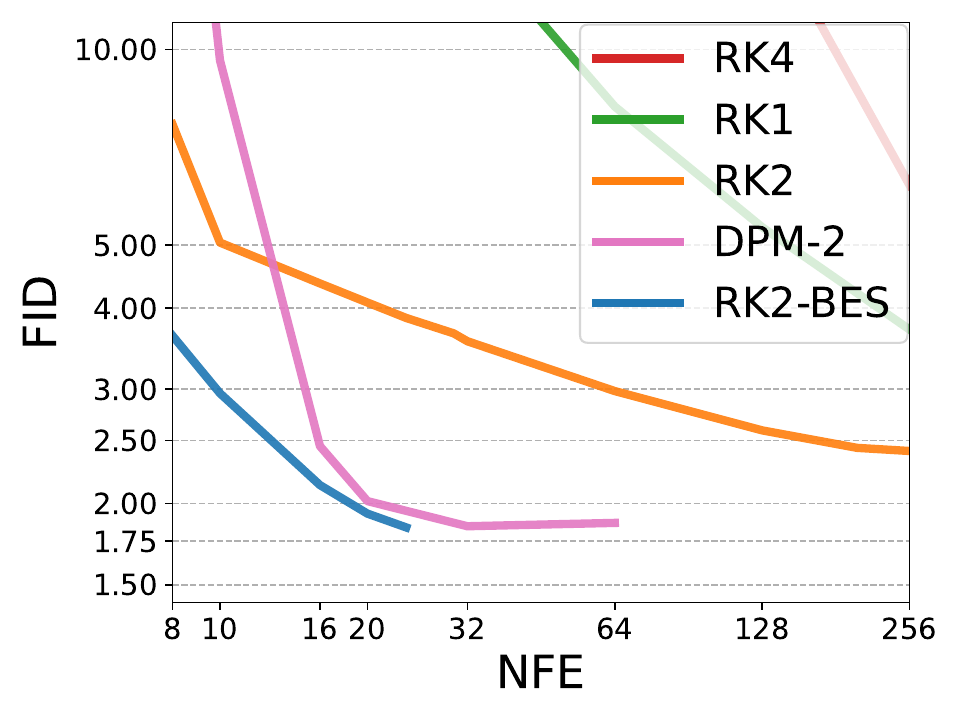} & \includegraphics[width=0.25\textwidth]{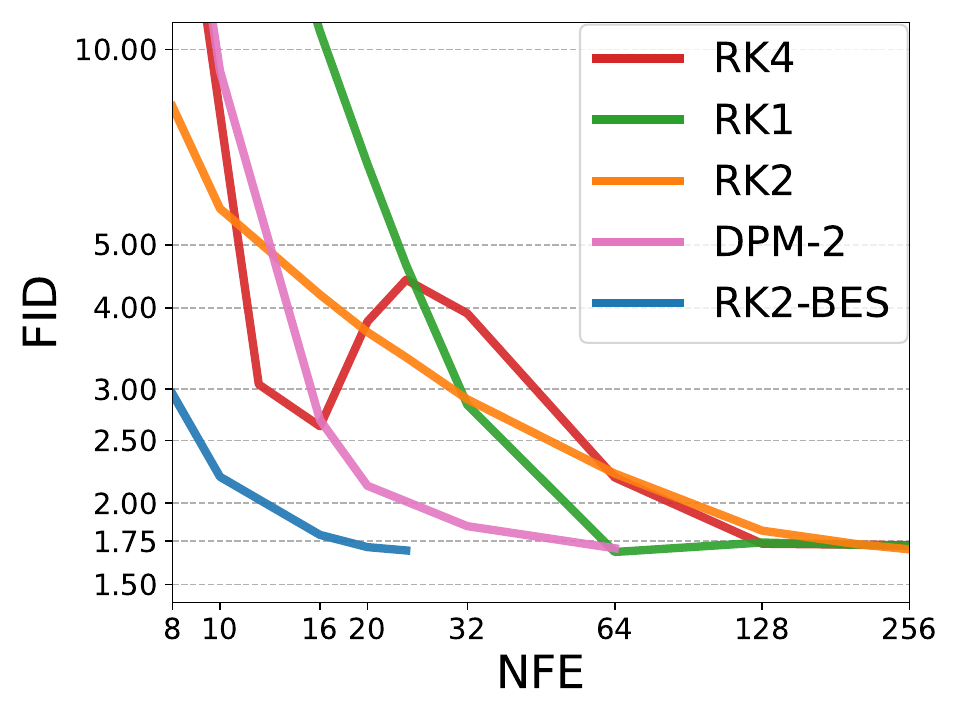} & \includegraphics[width=0.25\textwidth]{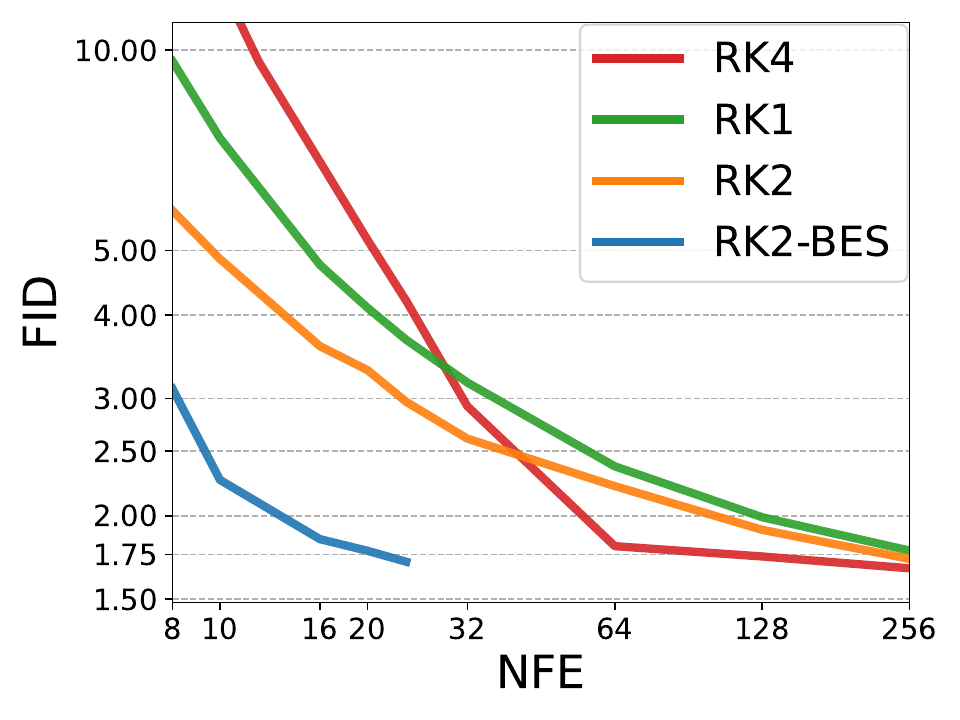} &
    \includegraphics[width=0.25\textwidth]{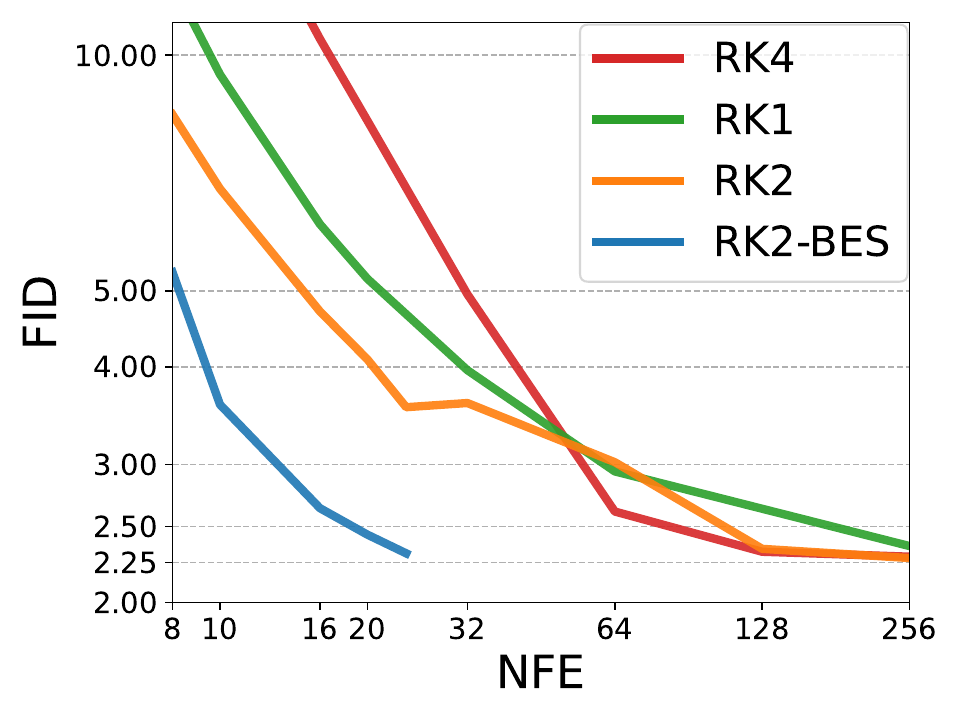}\\
     \includegraphics[width=0.25\textwidth]{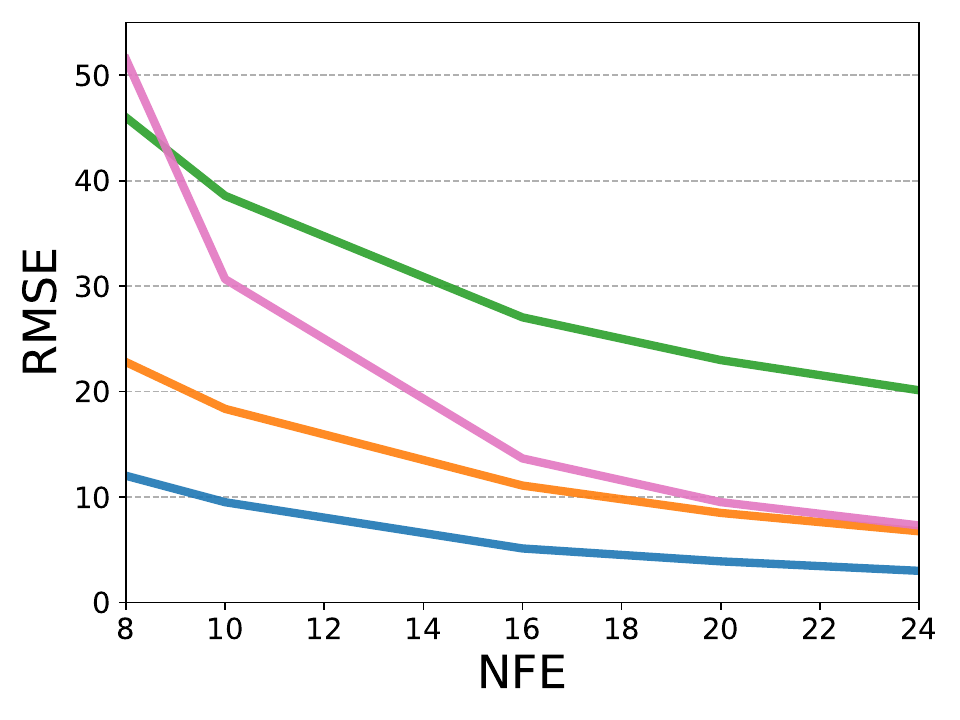}  & \includegraphics[width=0.25\textwidth]{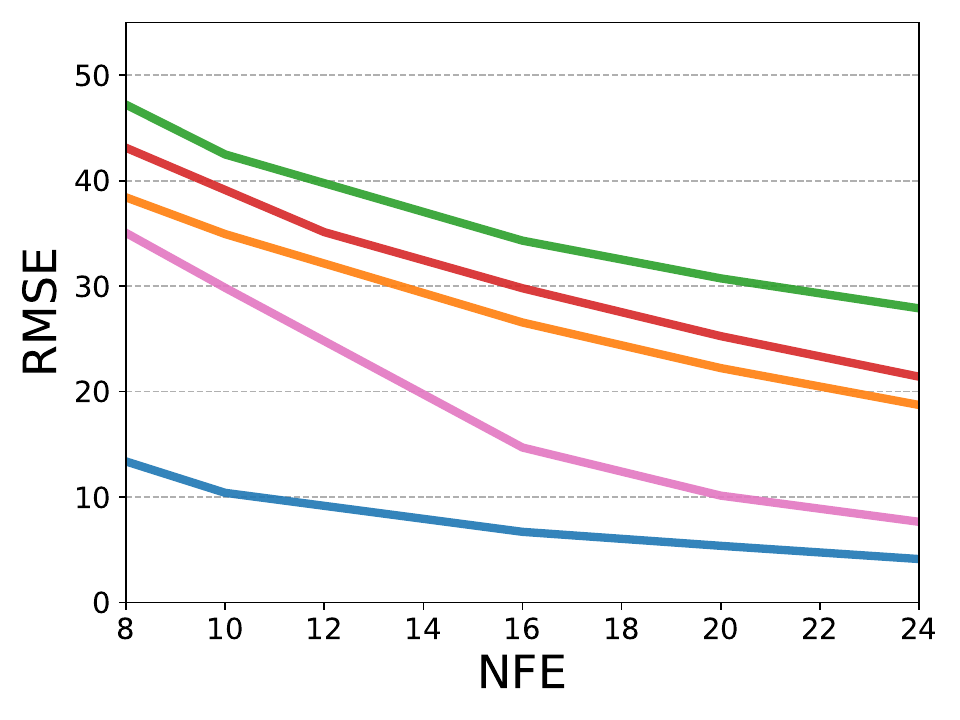} & \includegraphics[width=0.25\textwidth]{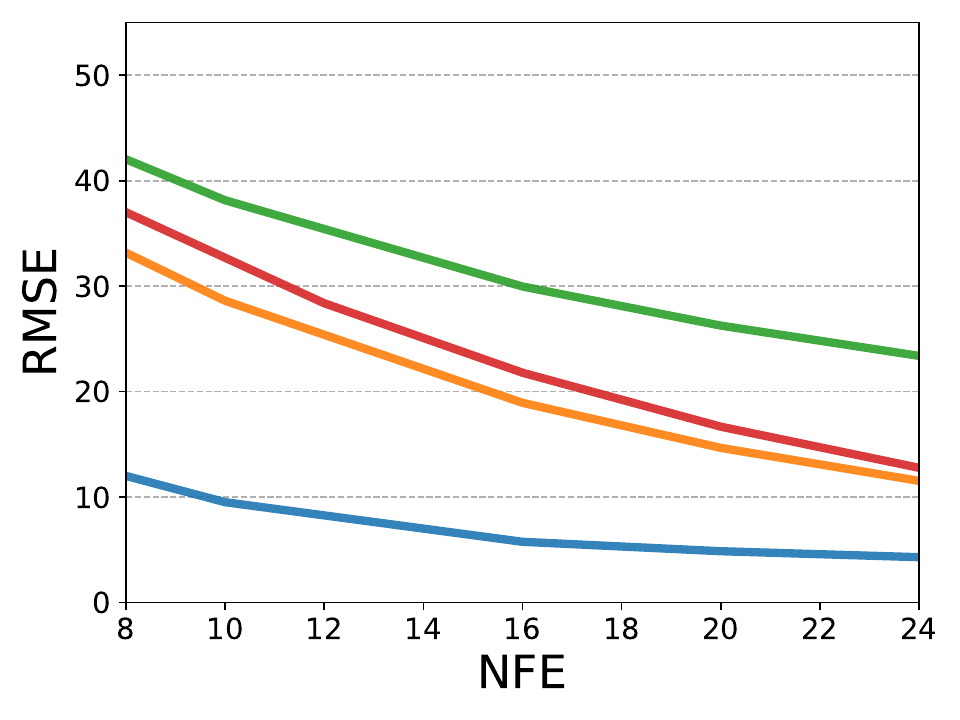} &
     \includegraphics[width=0.25\textwidth]{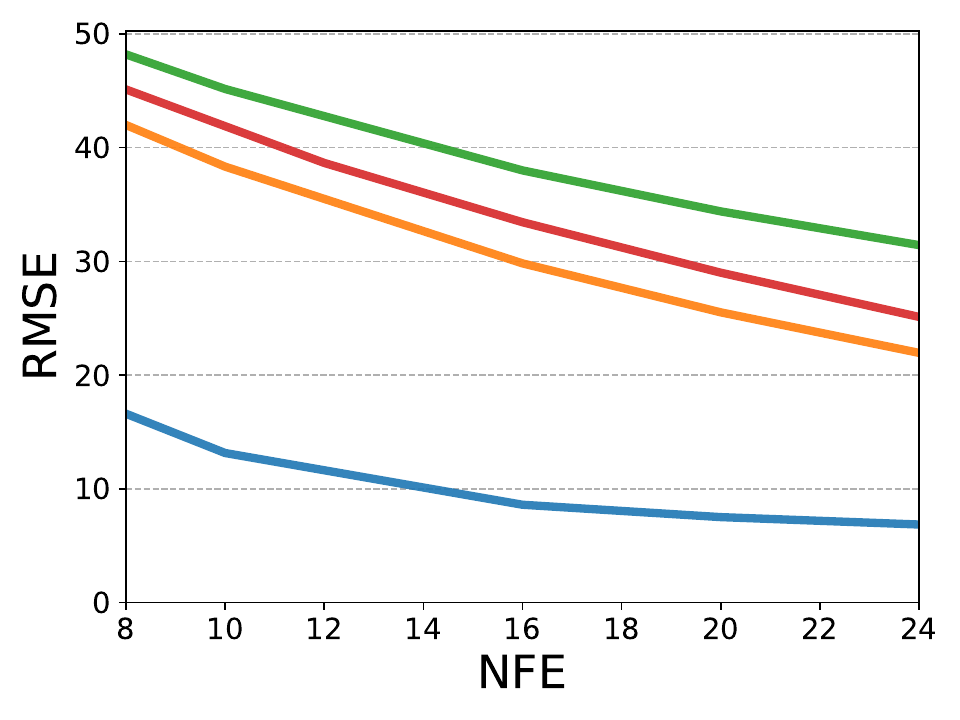}
     \vspace{-10pt}
\end{tabular}    
     \caption{Bespoke RK2 solvers vs.~RK1/2/4 solvers on CIFAR-10 ImageNet-64, and Image-Net128: FID vs.~NFE (top row), and RMSE vs.~NFE (bottom row). PSNR vs.~NFE is shown in Figure  \ref{fig:nfe_vs_psnr_imagenet}. \vspace{-10pt}} 
    \label{fig:nfe_vs_fid_main}
\end{figure}

%% file: tables/imagenet_table.tex
\begin{wraptable}[20]{r}{0.44\textwidth}\vspace{-10pt}
\setlength{\tabcolsep}{2.0pt}
\centering
\resizebox{0.43\textwidth}{!}{
\begin{tabular}{lcccc} 
 \toprule
  ImageNet-64 &  NFE & FID & GT-FID/\% & \%Time   \\ \midrule 
 \textbf{\textit{RK2-BES}} \ 
 \begin{tabular}[t]{r}
       $\eps$-VP \\$\eps$-VP \\ $\eps$-VP\\ $\eps$-VP\\ $\eps$-VP
 \end{tabular}
 & 
 \begin{tabular}[t]{c}
      8\\ 10\\ 16\\ 20\\ 24
 \end{tabular}
 &
 \begin{tabular}[t]{c}
      3.63\\ 2.96\\ 2.14\\ 1.93\\ 1.84
 \end{tabular}  
 &
 \begin{tabular}[t]{c}
      1.83 
 \end{tabular} / 
 \begin{tabular}[t]{c}
      229\\ 163\\ 120\\ 109\\ 101
 \end{tabular}  
 &
 \begin{tabular}[t]{c}
      3.5\\ 3.6\\ 3.6\\ 3.5\\ 3.6
 \end{tabular}  \\[35pt] 
  \textbf{\textit{RK2-BES}} \ 
 \begin{tabular}[t]{r}
       FM/$v$-CS \\ FM/$v$-CS \\ FM/$v$-CS \\ FM/$v$-CS \\ FM/$v$-CS 
 \end{tabular}
 & 
 \begin{tabular}[t]{c}
      8\\ 10\\ 16\\ 20\\ 24
 \end{tabular}
 &
 \begin{tabular}[t]{c}
      2.95\\ 2.20\\ 1.79\\ 1.71\\ 1.69
 \end{tabular}  
 &
 \begin{tabular}[t]{c}
      1.68
 \end{tabular} /
 \begin{tabular}[t]{c}
      176\\ 131\\ 107\\ 102\\ 101
 \end{tabular}  
 &
 \begin{tabular}[t]{c}
      1.4\\ 1.6\\ 1.8\\ 1.5\\ 2.0
 \end{tabular}  \\[35pt] 
 \textbf{\textit{RK2-BES}} \ 
 \begin{tabular}[t]{r}
       FM-OT \\ FM-OT \\ FM-OT \\ FM-OT \\ FM-OT 
 \end{tabular}
 & 
 \begin{tabular}[t]{c}
      8\\ 10\\ 16\\ 20\\ 24
 \end{tabular}
 &
 \begin{tabular}[t]{c}
     3.10\\ 2.26\\ 1.84\\ 1.77\\ 1.71
 \end{tabular}
 &
 \begin{tabular}[t]{c}
      1.68
 \end{tabular} / 
 \begin{tabular}[t]{c}
      185\\ 135\\ 110\\ 105\\ 102
 \end{tabular}  
 &
 \begin{tabular}[t]{c}
      1.6\\ 1.6\\ 1.7\\ 1.7\\ 1.8
 \end{tabular}  \\ 
  \toprule
  ImageNet-128 &  NFE & FID & GT-FID/\%  & \%Time  \\ \midrule 
\textbf{\textit{RK2-BES}} \ 
 \begin{tabular}[t]{r}
       FM-OT \\ FM-OT \\ FM-OT \\ FM-OT \\ FM-OT 
 \end{tabular}
 & 
 \begin{tabular}[t]{c}
      8\\ 10\\ 16\\ 20 \\ 24\\
 \end{tabular}
 &
 \begin{tabular}[t]{c}
      5.28\\ 3.58\\ 2.64\\ 2.45\\ 2.31\\
 \end{tabular}
 &
 \begin{tabular}[t]{c}
      2.30 
 \end{tabular}
 /
 \begin{tabular}[t]{c}
      230\\ 156\\ 115\\ 107\\ 101
 \end{tabular}  
 &
 \begin{tabular}[t]{c}
      1.1\\ 1.1\\ 1.2\\ 1.2\\ 1.2\\ 
 \end{tabular} \\ 
 \bottomrule  
\end{tabular}} \vspace{-5pt}
\caption{ImageNet Bespoke solvers. 
}\label{tab:imagenet_64_and_128}
\end{wraptable} 

%% file: tables/imagenet64_samples.tex
\begin{figure}[t]
\centering
\begin{tabular}{@{\hspace{2pt}}c@{\hspace{1pt}}c@{\hspace{1pt}}c@{\hspace{1pt}}c@{\hspace{2pt}}|@{\hspace{2pt}}c@{\hspace{1pt}}c@{\hspace{1pt}}c@{\hspace{1pt}}c@{\hspace{1pt}}c@{\hspace{2pt}}|@{\hspace{2pt}}c@{\hspace{1pt}}c@{\hspace{1pt}}c@{\hspace{1pt}}c@{\hspace{1pt}}c@{\hspace{0pt}}} 
{\scriptsize GT} & {\scriptsize NFE=20} & {\scriptsize NFE=10} & {\scriptsize NFE=8} & {\scriptsize GT} & {\scriptsize NFE=20} & {\scriptsize NFE=10} & {\scriptsize NFE=8} & {\scriptsize GT} & {\scriptsize NFE=20} & {\scriptsize NFE=10} & {\scriptsize NFE=8} \\ 
\rotatebox[origin=c]{90}{\scriptsize FM-OT}
\cincludegraphics[width=0.077 \textwidth]{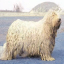} & 
    \begin{tabular}{@{\hspace{0pt}}c@{\hspace{0pt}}}    
        \includegraphics[width=0.077 \textwidth]{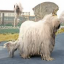} \\
        \includegraphics[width=0.077 \textwidth]{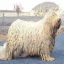}
    \end{tabular}
    &
    \begin{tabular}{@{\hspace{0pt}}c@{\hspace{0pt}}}   
        \includegraphics[width=0.077 \textwidth]{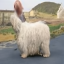} \\
        \includegraphics[width=0.077 \textwidth]{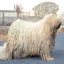}
    \end{tabular}
    &\begin{tabular}{@{\hspace{0pt}}c@{\hspace{0pt}}}  
        \includegraphics[width=0.077 \textwidth]{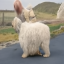} \\
        \includegraphics[width=0.077 \textwidth]{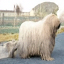}
    \end{tabular}
    &
\cincludegraphics[width=0.077 \textwidth]{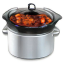} & 
    \begin{tabular}{@{\hspace{0pt}}c@{\hspace{0pt}}}   
        \includegraphics[width=0.077 \textwidth]{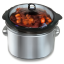} \\
        \includegraphics[width=0.077 \textwidth]{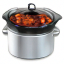}
    \end{tabular}
    &
    \begin{tabular}{@{\hspace{0pt}}c@{\hspace{0pt}}}
        \includegraphics[width=0.077 \textwidth]{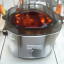} \\
        \includegraphics[width=0.077 \textwidth]{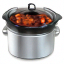}
    \end{tabular}
    &\begin{tabular}{@{\hspace{0pt}}c@{\hspace{0pt}}}   
        \includegraphics[width=0.077 \textwidth]{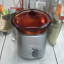} \\
        \includegraphics[width=0.077 \textwidth]{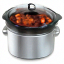}
    \end{tabular}
    &
\cincludegraphics[width=0.077 \textwidth]{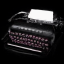} & 
    \begin{tabular}{@{\hspace{0pt}}c@{\hspace{0pt}}}    
        \includegraphics[width=0.077 \textwidth]{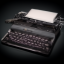} \\
        \includegraphics[width=0.077 \textwidth]{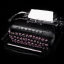}
    \end{tabular}
    &
    \begin{tabular}{@{\hspace{0pt}}c@{\hspace{0pt}}}   
        \includegraphics[width=0.077 \textwidth]{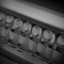} \\
        \includegraphics[width=0.077 \textwidth]{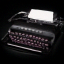}
    \end{tabular}
    &\begin{tabular}{@{\hspace{0pt}}c@{\hspace{0pt}}}  
        \includegraphics[width=0.077 \textwidth]{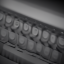} \\
        \includegraphics[width=0.077 \textwidth]{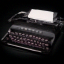}
    \end{tabular}
    &\begin{tabular}{@{\hspace{0pt}}c@{\hspace{0pt}}}   
        \raisebox{-0.5\height}{\rotatebox[origin=c]{90}{\scriptsize RK2}} \\
        \raisebox{-1.35\height}{\rotatebox[origin=c]{90}{\scriptsize RK2-BES}}
    \end{tabular}
\\ \midrule
\rotatebox[origin=c]{90}{\scriptsize FM/$v$-CS} \cincludegraphics[width=0.077 \textwidth]{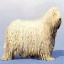} & 
    \begin{tabular}{@{\hspace{0pt}}c@{\hspace{0pt}}}    
        \includegraphics[width=0.077 \textwidth]{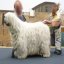} \\
        \includegraphics[width=0.077 \textwidth]{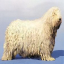}
    \end{tabular}
    &
    \begin{tabular}{@{\hspace{0pt}}c@{\hspace{0pt}}}   
        \includegraphics[width=0.077 \textwidth]{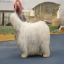} \\
        \includegraphics[width=0.077 \textwidth]{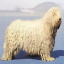}
    \end{tabular}
    &\begin{tabular}{@{\hspace{0pt}}c@{\hspace{0pt}}}  
        \includegraphics[width=0.077 \textwidth]{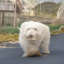} \\
        \includegraphics[width=0.077 \textwidth]{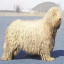}
    \end{tabular}
    &
\cincludegraphics[width=0.077 \textwidth]{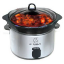} & 
    \begin{tabular}{@{\hspace{0pt}}c@{\hspace{0pt}}}   
        \includegraphics[width=0.077 \textwidth]{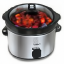} \\
        \includegraphics[width=0.077 \textwidth]{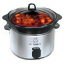}
    \end{tabular}
    &
    \begin{tabular}{@{\hspace{0pt}}c@{\hspace{0pt}}}
        \includegraphics[width=0.077 \textwidth]{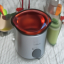} \\
        \includegraphics[width=0.077 \textwidth]{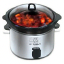}
    \end{tabular}
    &\begin{tabular}{@{\hspace{0pt}}c@{\hspace{0pt}}}   
        \includegraphics[width=0.077 \textwidth]{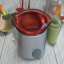} \\
        \includegraphics[width=0.077 \textwidth]{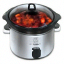}
    \end{tabular}
    &
\cincludegraphics[width=0.077 \textwidth]{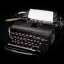} & 
    \begin{tabular}{@{\hspace{0pt}}c@{\hspace{0pt}}}    
        \includegraphics[width=0.077 \textwidth]{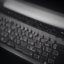} \\
        \includegraphics[width=0.077 \textwidth]{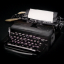}
    \end{tabular}
    &
    \begin{tabular}{@{\hspace{0pt}}c@{\hspace{0pt}}}   
        \includegraphics[width=0.077 \textwidth]{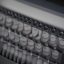} \\
        \includegraphics[width=0.077 \textwidth]{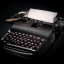}
    \end{tabular}
    &\begin{tabular}{@{\hspace{0pt}}c@{\hspace{0pt}}}  
        \includegraphics[width=0.077 \textwidth]{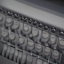} \\
        \includegraphics[width=0.077 \textwidth]{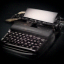}
    \end{tabular}
    &\begin{tabular}{@{\hspace{0pt}}c@{\hspace{0pt}}}   
        \raisebox{-0.5\height}{\rotatebox[origin=c]{90}{\scriptsize RK2}} \\
        \raisebox{-1.35\height}{\rotatebox[origin=c]{90}{\scriptsize RK2-BES}}
    \end{tabular}\vspace{-10pt}
\end{tabular}
\caption{Comparison of FM-OT and FM/$v$-CS ImageNet-64 samples with RK2 and bespoke-RK2 solvers. Comparison to DPM-2 samples are in Figure \ref{fig:imagenet-64-dpm2}. More examples are in Figures \ref{fig:imagenet64_fm_ot_a}, \ref{fig:imagenet64_fm_cs_a}, and \ref{fig:imagenet64_eps_pred_a}. The similarity of generated images across models can be explained by their identical noise-to-data coupling (Theorem \ref{thm:equivalence}). \vspace{-5pt} }\label{fig:imagenet-64}
\end{figure}

%% file: tables/imagenet128_afhq_samples.tex
\begin{figure}[t]
\centering
\begin{tabular}{@{\hspace{1pt}}c@{\hspace{1pt}}c@{\hspace{1pt}}c@{\hspace{1pt}}c@{\hspace{1pt}}c@{\hspace{2pt}}|@{\hspace{2pt}}c@{\hspace{1pt}}c@{\hspace{1pt}}c@{\hspace{1pt}}c@{\hspace{1pt}}c@{\hspace{2pt}}|@{\hspace{2pt}}c@{\hspace{1pt}}c@{\hspace{1pt}}c@{\hspace{1pt}}c@{\hspace{1pt}}c@{\hspace{0pt}}} 
& {\scriptsize GT} & {\scriptsize NFE=20} & {\scriptsize NFE=10} & {\scriptsize NFE=8} & {\scriptsize GT} & {\scriptsize NFE=20} & {\scriptsize NFE=10} & {\scriptsize NFE=8} & {\scriptsize GT} & {\scriptsize NFE=20} & {\scriptsize NFE=10} & {\scriptsize NFE=8} \\ 
\multirow{2}{*}{ \raisebox{-1.5\height}{\rotatebox[origin=c]{90}{\scriptsize ImageNet-128}}} & 
\cincludegraphics[width=0.077 \textwidth]{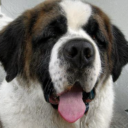} & 
    \begin{tabular}{@{\hspace{0pt}}c@{\hspace{0pt}}}    
        \includegraphics[width=0.077 \textwidth]{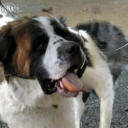} \\
        \includegraphics[width=0.077 \textwidth]{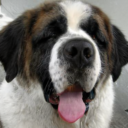}
    \end{tabular}
    &
    \begin{tabular}{@{\hspace{0pt}}c@{\hspace{0pt}}}   
        \includegraphics[width=0.077 \textwidth]{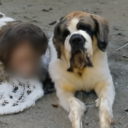} \\
        \includegraphics[width=0.077 \textwidth]{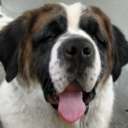}
    \end{tabular}
    &\begin{tabular}{@{\hspace{0pt}}c@{\hspace{0pt}}}  
        \includegraphics[width=0.077 \textwidth]{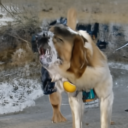} \\
        \includegraphics[width=0.077 \textwidth]{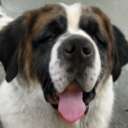}
    \end{tabular}
    &
\cincludegraphics[width=0.077 \textwidth]{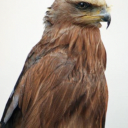} & 
    \begin{tabular}{@{\hspace{0pt}}c@{\hspace{0pt}}}   
        \includegraphics[width=0.077 \textwidth]{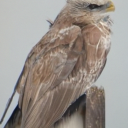} \\
        \includegraphics[width=0.077 \textwidth]{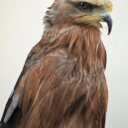}
    \end{tabular}
    &
    \begin{tabular}{@{\hspace{0pt}}c@{\hspace{0pt}}}
        \includegraphics[width=0.077 \textwidth]{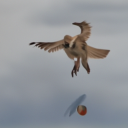} \\
        \includegraphics[width=0.077 \textwidth]{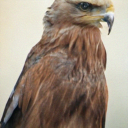}
    \end{tabular}
    &\begin{tabular}{@{\hspace{0pt}}c@{\hspace{0pt}}}   
        \includegraphics[width=0.077 \textwidth]{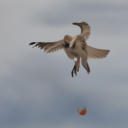} \\
        \includegraphics[width=0.077 \textwidth]{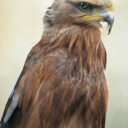}
    \end{tabular}
    &
\cincludegraphics[width=0.077 \textwidth]{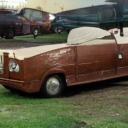} & 
    \begin{tabular}{@{\hspace{0pt}}c@{\hspace{0pt}}}    
        \includegraphics[width=0.077 \textwidth]{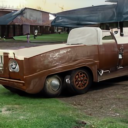} \\
        \includegraphics[width=0.077 \textwidth]{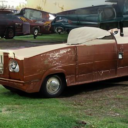}
    \end{tabular}
    &
    \begin{tabular}{@{\hspace{0pt}}c@{\hspace{0pt}}}   
        \includegraphics[width=0.077 \textwidth]{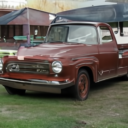} \\
        \includegraphics[width=0.077 \textwidth]{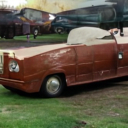}
    \end{tabular}
    &\begin{tabular}{@{\hspace{0pt}}c@{\hspace{0pt}}}  
        \includegraphics[width=0.077 \textwidth]{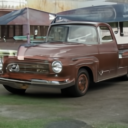} \\
        \includegraphics[width=0.077 \textwidth]{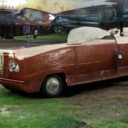}
    \end{tabular}
    &\begin{tabular}{@{\hspace{0pt}}c@{\hspace{0pt}}}   
        \raisebox{-0.5\height}{\rotatebox[origin=c]{90}{\scriptsize RK2}} \\
        \raisebox{-1.35\height}{\rotatebox[origin=c]{90}{\scriptsize RK2-BES}}
    \end{tabular}
\\ 
& \cincludegraphics[width=0.077 \textwidth]{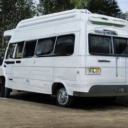} & 
    \begin{tabular}{@{\hspace{0pt}}c@{\hspace{0pt}}}    
        \includegraphics[width=0.077 \textwidth]{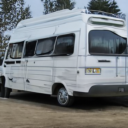} \\
        \includegraphics[width=0.077 \textwidth]{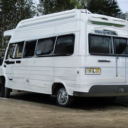}
    \end{tabular}
    &
    \begin{tabular}{@{\hspace{0pt}}c@{\hspace{0pt}}}   
        \includegraphics[width=0.077 \textwidth]{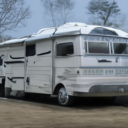} \\
        \includegraphics[width=0.077 \textwidth]{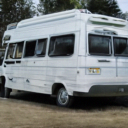}
    \end{tabular}
    &\begin{tabular}{@{\hspace{0pt}}c@{\hspace{0pt}}}  
        \includegraphics[width=0.077 \textwidth]{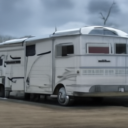} \\
        \includegraphics[width=0.077 \textwidth]{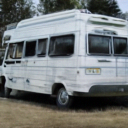}
    \end{tabular}
    &
\cincludegraphics[width=0.077 \textwidth]{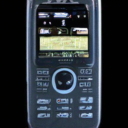} & 
    \begin{tabular}{@{\hspace{0pt}}c@{\hspace{0pt}}}   
        \includegraphics[width=0.077 \textwidth]{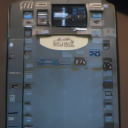} \\
        \includegraphics[width=0.077 \textwidth]{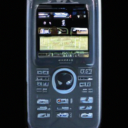}
    \end{tabular}
    &
    \begin{tabular}{@{\hspace{0pt}}c@{\hspace{0pt}}}
        \includegraphics[width=0.077 \textwidth]{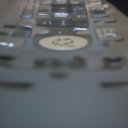} \\
        \includegraphics[width=0.077 \textwidth]{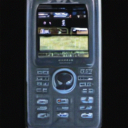}
    \end{tabular}
    &\begin{tabular}{@{\hspace{0pt}}c@{\hspace{0pt}}}   
        \includegraphics[width=0.077 \textwidth]{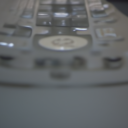} \\
        \includegraphics[width=0.077 \textwidth]{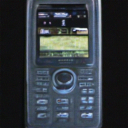}
    \end{tabular}
    &
\cincludegraphics[width=0.077 \textwidth]{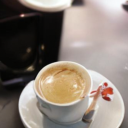} & 
    \begin{tabular}{@{\hspace{0pt}}c@{\hspace{0pt}}}    
        \includegraphics[width=0.077 \textwidth]{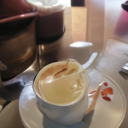} \\
        \includegraphics[width=0.077 \textwidth]{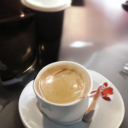}
    \end{tabular}
    &
    \begin{tabular}{@{\hspace{0pt}}c@{\hspace{0pt}}}   
        \includegraphics[width=0.077 \textwidth]{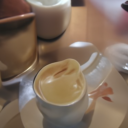} \\
        \includegraphics[width=0.077 \textwidth]{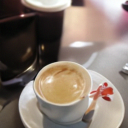}
    \end{tabular}
    &\begin{tabular}{@{\hspace{0pt}}c@{\hspace{0pt}}}  
        \includegraphics[width=0.077 \textwidth]{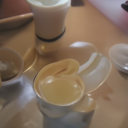} \\
        \includegraphics[width=0.077 \textwidth]{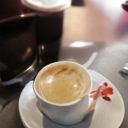}
    \end{tabular}
    &\begin{tabular}{@{\hspace{0pt}}c@{\hspace{0pt}}}   
        \raisebox{-0.5\height}{\rotatebox[origin=c]{90}{\scriptsize RK2}} \\
        \raisebox{-1.35\height}{\rotatebox[origin=c]{90}{\scriptsize RK2-BES}}
    \end{tabular}
\\ \midrule
\rotatebox[origin=c]{90}{\scriptsize AFHQ-256} & 
\cincludegraphics[width=0.077 \textwidth]{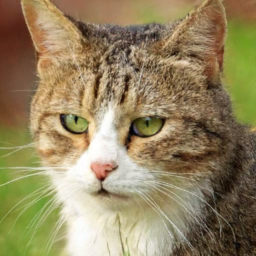} & 
    \begin{tabular}{@{\hspace{0pt}}c@{\hspace{0pt}}}    
        \includegraphics[width=0.077 \textwidth]{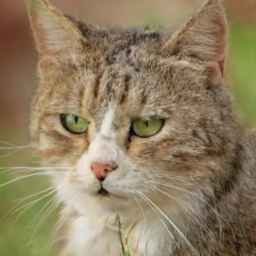} \\
        \includegraphics[width=0.077 \textwidth]{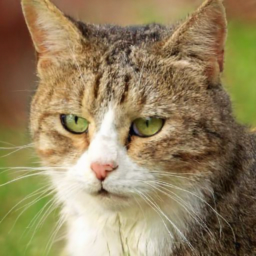}
    \end{tabular}
    &
    \begin{tabular}{@{\hspace{0pt}}c@{\hspace{0pt}}}   
        \includegraphics[width=0.077 \textwidth]{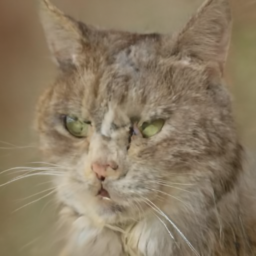} \\
        \includegraphics[width=0.077 \textwidth]{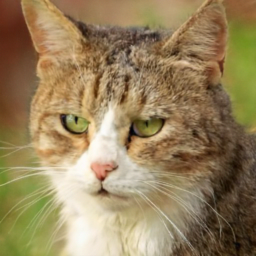}
    \end{tabular}
    &\begin{tabular}{@{\hspace{0pt}}c@{\hspace{0pt}}}  
        \includegraphics[width=0.077 \textwidth]{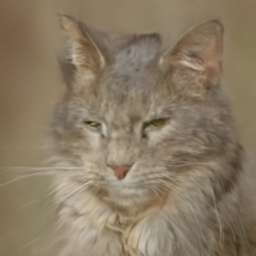} \\
        \includegraphics[width=0.077 \textwidth]{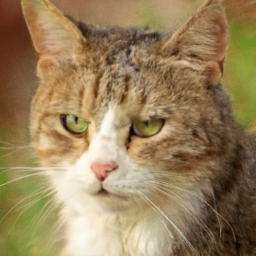}
    \end{tabular}
    &
\cincludegraphics[width=0.077 \textwidth]{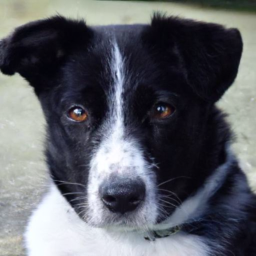} & 
    \begin{tabular}{@{\hspace{0pt}}c@{\hspace{0pt}}}   
        \includegraphics[width=0.077 \textwidth]{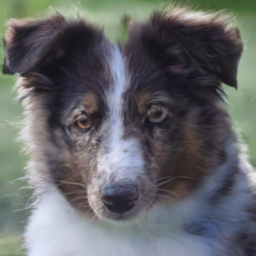} \\
        \includegraphics[width=0.077 \textwidth]{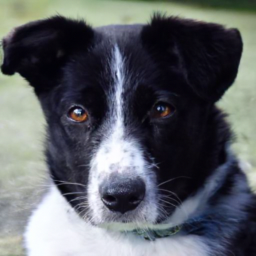}
    \end{tabular}
    &
    \begin{tabular}{@{\hspace{0pt}}c@{\hspace{0pt}}}
        \includegraphics[width=0.077 \textwidth]{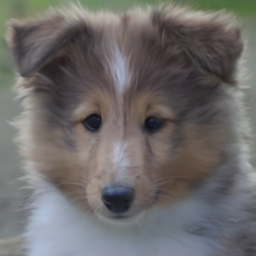} \\
        \includegraphics[width=0.077 \textwidth]{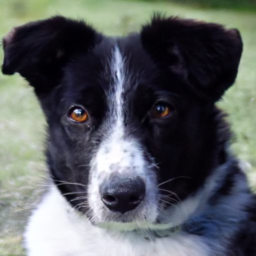}
    \end{tabular}
    &\begin{tabular}{@{\hspace{0pt}}c@{\hspace{0pt}}}   
        \includegraphics[width=0.077 \textwidth]{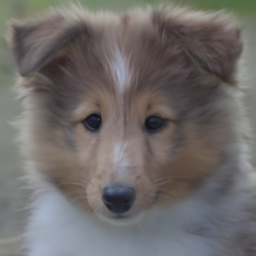} \\
        \includegraphics[width=0.077 \textwidth]{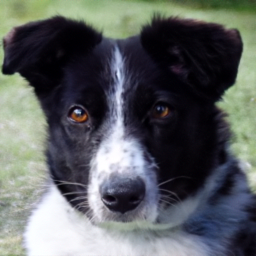}
    \end{tabular}
    &
\cincludegraphics[width=0.077 \textwidth]{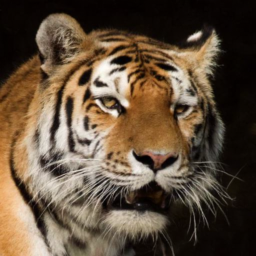} & 
    \begin{tabular}{@{\hspace{0pt}}c@{\hspace{0pt}}}    
        \includegraphics[width=0.077 \textwidth]{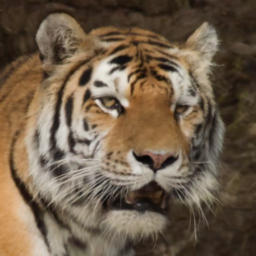} \\
        \includegraphics[width=0.077 \textwidth]{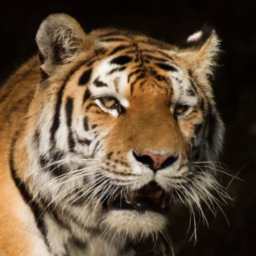}
    \end{tabular}
    &
    \begin{tabular}{@{\hspace{0pt}}c@{\hspace{0pt}}}   
        \includegraphics[width=0.077 \textwidth]{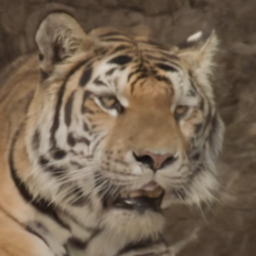} \\
        \includegraphics[width=0.077 \textwidth]{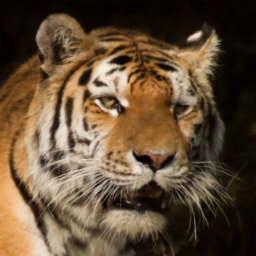}
    \end{tabular}
    &\begin{tabular}{@{\hspace{0pt}}c@{\hspace{0pt}}}  
        \includegraphics[width=0.077 \textwidth]{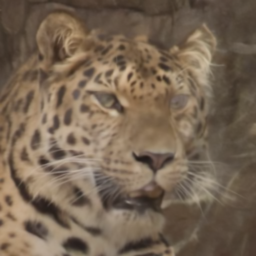} \\
        \includegraphics[width=0.077 \textwidth]{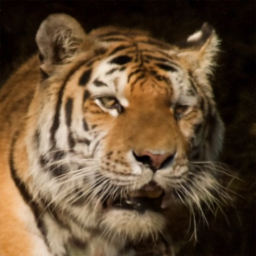}
    \end{tabular}
    &\begin{tabular}{@{\hspace{0pt}}c@{\hspace{0pt}}}   
        \raisebox{-0.5\height}{\rotatebox[origin=c]{90}{\scriptsize RK2}} \\
        \raisebox{-1.35\height}{\rotatebox[origin=c]{90}{\scriptsize RK2-BES}}
    \end{tabular}\vspace{-10pt}
    
\end{tabular}
\caption{FM-OT ImageNet-128 (top) and AFHQ-256 (bottom) samples with RK2 and bespoke-RK2 solvers. More examples are in Figures \ref{fig:imagenet_128_a_1}, \ref{fig:imagenet_128_a_2} and \ref{fig:afhq_a}.\vspace{-10pt}}
\label{fig:imagenet-128}
\end{figure}

%% file: tables/rk1_vs_rk2_appendix.tex
\begin{figure}[h!]
    \centering
    \begin{tabular}{@{\hspace{0pt}}c@{\hspace{0pt}}c@{\hspace{0pt}}}  \quad ImageNet 64: FM/$v$-CS & \quad ImageNet 64: FM-OT \\
     \includegraphics[width=0.335\textwidth]{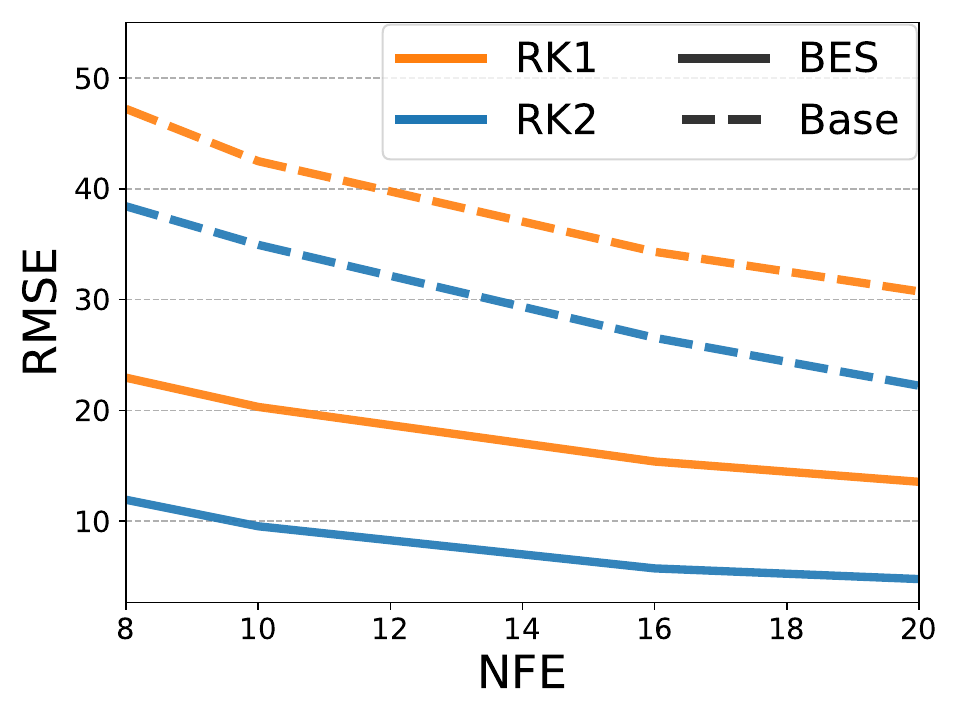} & \includegraphics[width=0.335\textwidth]{figures/rk1_vs_rk2/imagenet64_rk1_vs_rk2_fm_ot_rmse_vs_nfe_trans_vs_baseline.pdf} \\
      \includegraphics[width=0.335\textwidth]{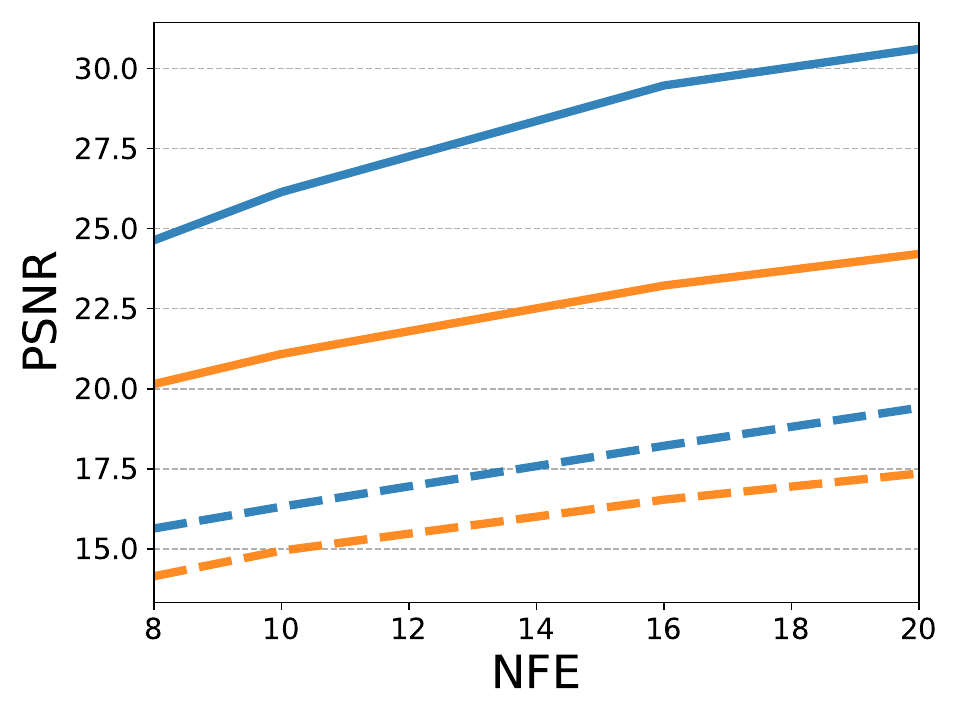} & \includegraphics[width=0.335\textwidth]{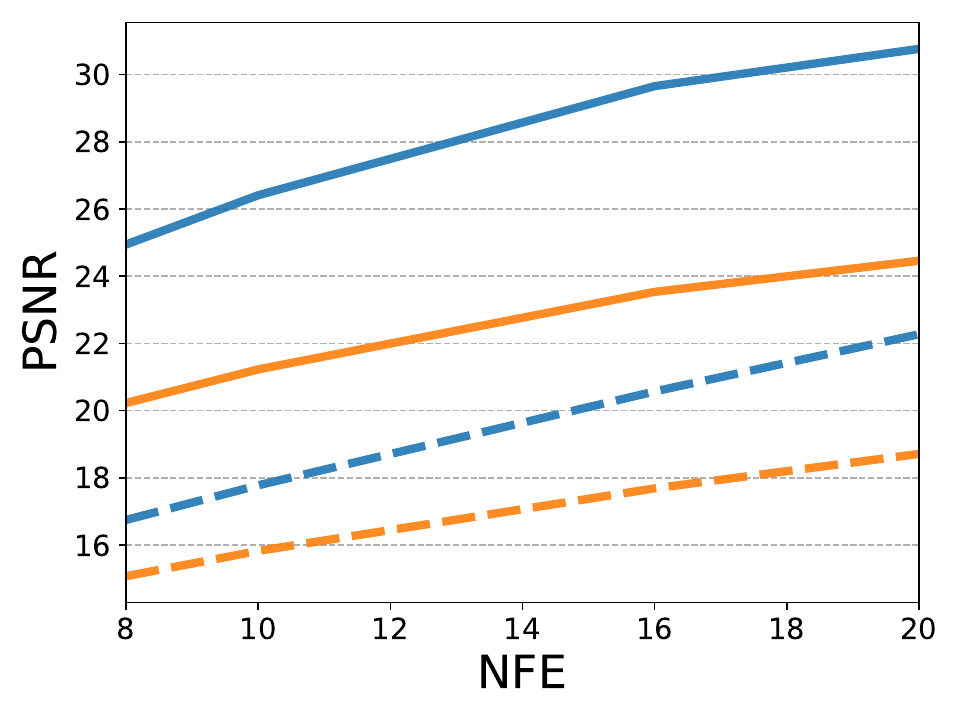} 
\end{tabular}    
    \caption{Bespoke RK1, Bespoke RK2, RK1, and RK2 solvers on ImageNet-64 models: RMSE vs.~NFE (top row), and PSNR vs.~NFE (bottom row).}
    \label{fig_a:rk1_vs_rk2_imagenet64}
\end{figure}

\begin{figure}[h!]
    \centering
    \begin{tabular}{@{\hspace{0pt}}c@{\hspace{0pt}}c@{\hspace{0pt}}c@{\hspace{0pt}}} \quad CIFAR10: $\eps$-VP & \quad CIFAR10: FM/$v$-CS & \quad CIFAR10: FM-OT \\
    \includegraphics[width=0.335\textwidth]{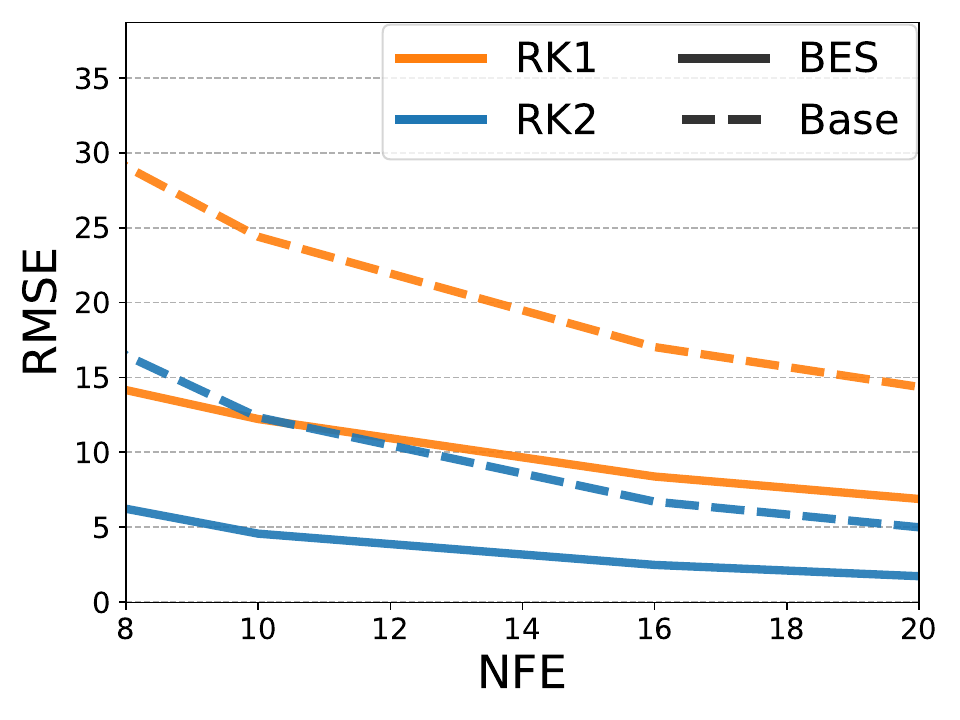} & \includegraphics[width=0.335\textwidth]{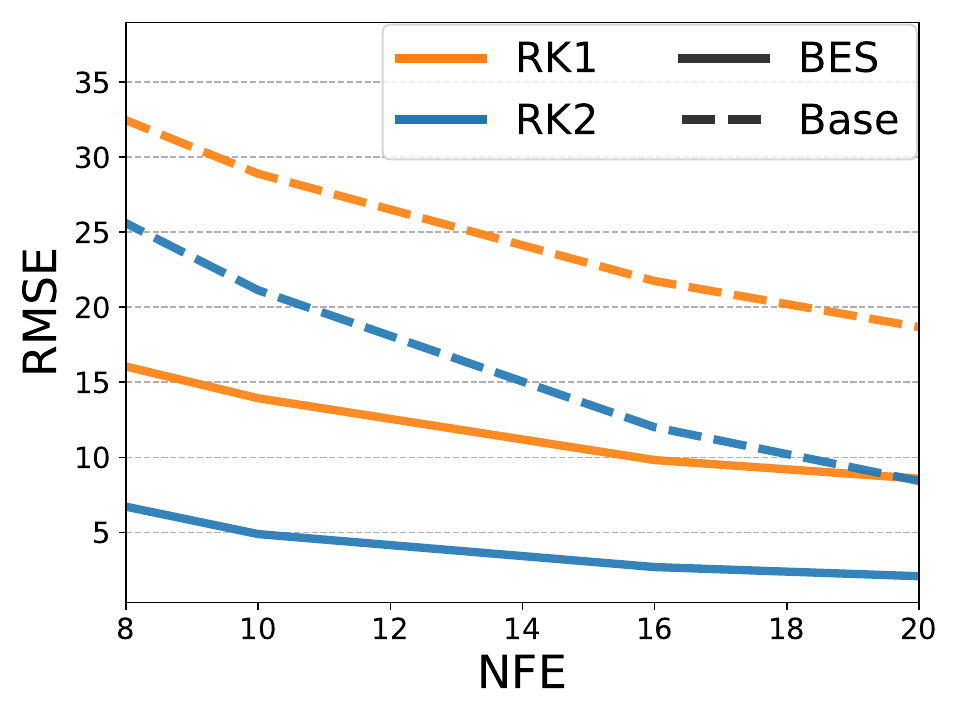} & \includegraphics[width=0.335\textwidth]{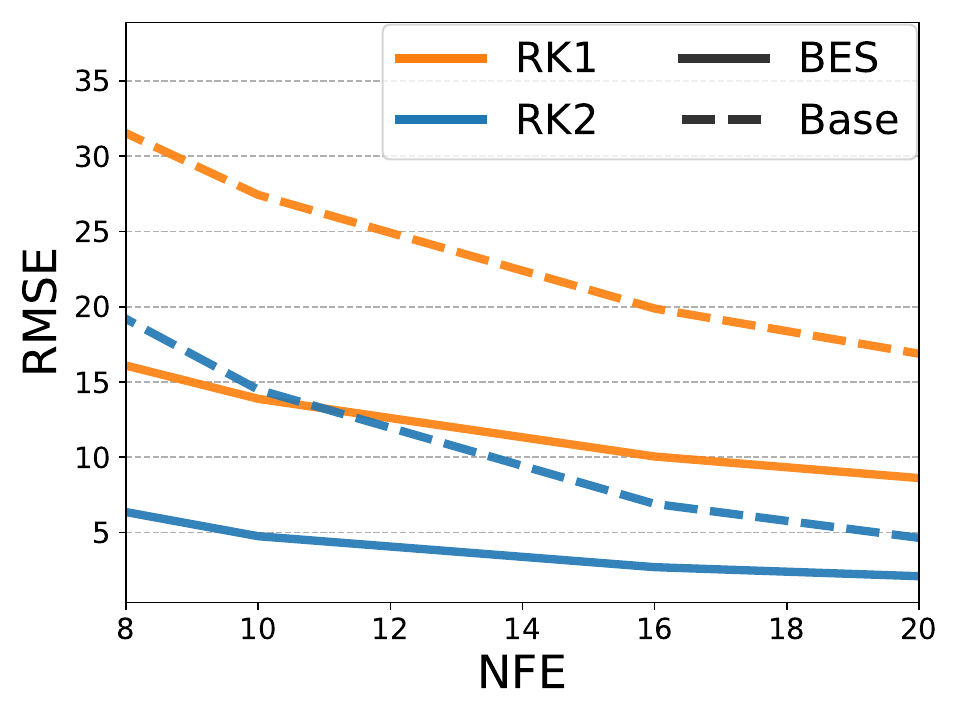} \\
     \includegraphics[width=0.335\textwidth]{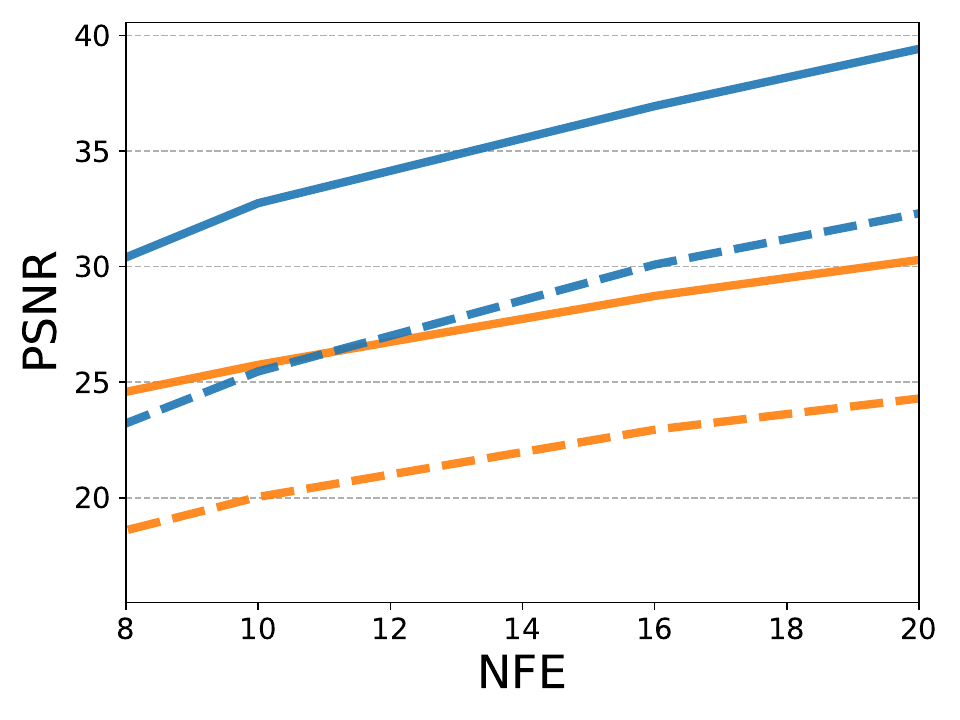}  & \includegraphics[width=0.335\textwidth]{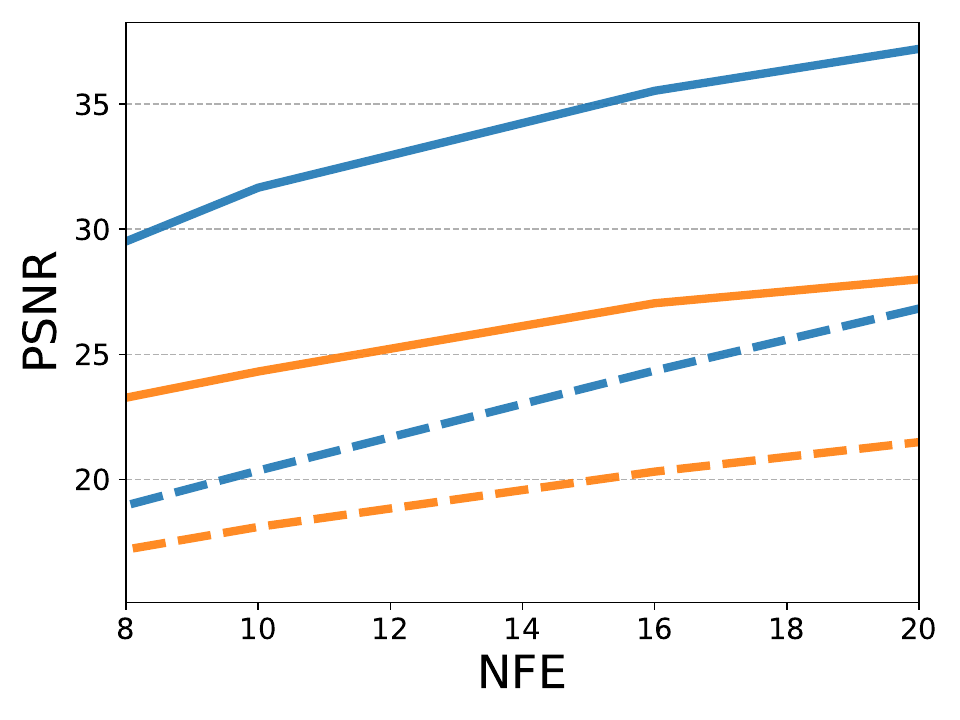} & \includegraphics[width=0.335\textwidth]{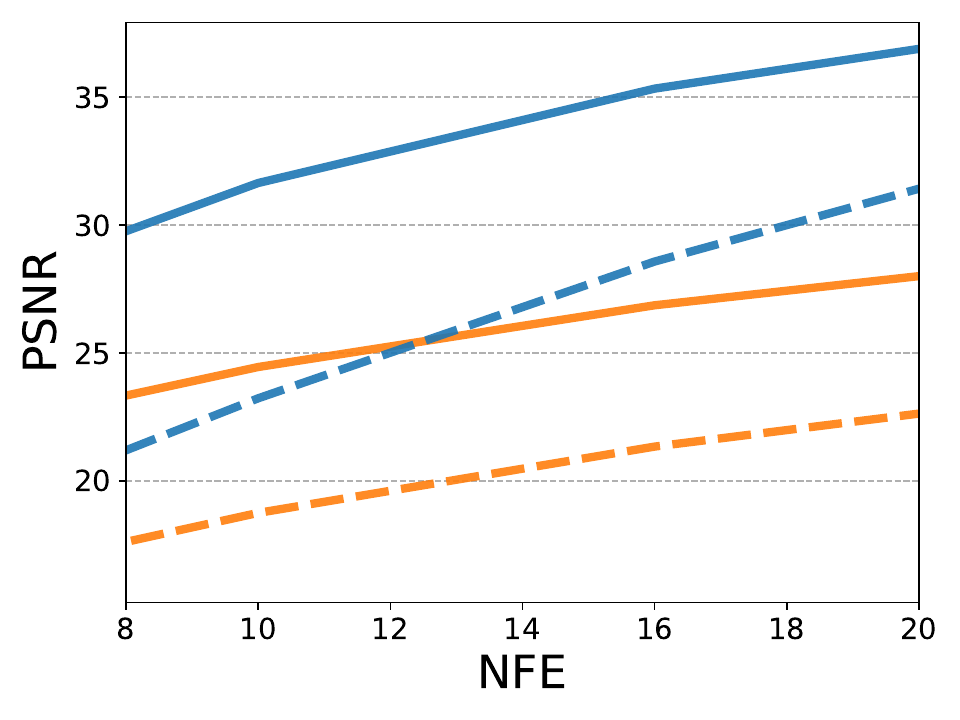} \\
\end{tabular}    
    \caption{Bespoke RK1, Bespoke RK2, RK1, and RK2 solvers on CIFAR10: RMSE vs.~NFE (top row), and PSNR vs.~NFE (bottom row).}
    \label{fig_a:rk1_vs_rk2_cifar10}
\end{figure}

%% file: tables/cifar10_table_appendix.tex

\begin{table}[h!]
    \centering
    \resizebox{0.66\textwidth}{!}{
    \begin{tabular}{lcccc} 
     \toprule
      CIFAR10 &  NFE & FID & GT-FID/\% & \%Time   \\ \midrule 
     \textbf{\textit{RK2-BES}} \ 
     \begin{tabular}[t]{r}
           $\eps$-VP \\$\eps$-VP \\ $\eps$-VP\\ $\eps$-VP
     \end{tabular}
     & 
     \begin{tabular}[t]{c}
         8\\ 10\\ 16\\ 20
     \end{tabular}
     &
     \begin{tabular}[t]{c}
          4.26\\ 3.31\\ 2.84\\ 2.75
     \end{tabular}  
     &
     \begin{tabular}[t]{c}
          2.54 
     \end{tabular} / 
     \begin{tabular}[t]{c}
          168\\ 130\\ 112\\ 108
     \end{tabular}  
     &
     \begin{tabular}[t]{c}
          1.4\\ 1.5\\ 1.5\\ 1.4
     \end{tabular}  \\[38pt] 
      \textbf{\textit{RK2-BES}} \ 
     \begin{tabular}[t]{r}
           FM/$v$-CS \\ FM/$v$-CS \\ FM/$v$-CS \\ FM/$v$-CS 
     \end{tabular}
     & 
     \begin{tabular}[t]{c}
          8\\ 10\\ 16\\ 20
     \end{tabular}
     &
     \begin{tabular}[t]{c}
          3.50\\ 2.89\\ 2.68\\ 2.64
     \end{tabular}  
     &
     \begin{tabular}[t]{c}
          2.61
     \end{tabular} /
     \begin{tabular}[t]{c}
          134\\ 111\\ 103\\ 101
     \end{tabular}  
     &
     \begin{tabular}[t]{c}
         0.5\\ 0.6\\ 0.6\\ 0.6
     \end{tabular}  \\[38pt] 
     \textbf{\textit{RK2-BES}} \ 
     \begin{tabular}[t]{r}
           FM-OT \\ FM-OT \\ FM-OT \\ FM-OT 
     \end{tabular}
     & 
     \begin{tabular}[t]{c}
          8\\ 10\\ 16\\ 20
     \end{tabular}
     &
     \begin{tabular}[t]{c}
         3.13\\ 2.73\\ 2.60\\ 2.59
     \end{tabular}
     &
     \begin{tabular}[t]{c}
          2.57
     \end{tabular} / 
     \begin{tabular}[t]{c}
          122\\ 106\\ 101\\ 101
     \end{tabular}  
     &
     \begin{tabular}[t]{c}
          0.5\\ 0.6\\ 0.6\\ 0.6
     \end{tabular}  \\ 
     \bottomrule 
    \end{tabular} } 
    \caption{CIFAR10 Bespoke solvers. We report best FID vs.~NFE, the ground truth FID (GT-FID) for the model and FID/GT-FID in \%, and the fraction of GPU time (in \%) required to train the bespoke solver w.r.t.~training the original model. }\label{tab:cifar10_a}
\end{table}

%% file: tables/graphs_imagenet_training.tex
\begin{figure}[h!]
    \centering
        \begin{tabular}{@{\hspace{0pt}}c@{\hspace{0pt}}c@{\hspace{0pt}}c@{\hspace{0pt}}c@{\hspace{0pt}}}
        {\quad \ \scriptsize ImageNet-64: $\eps$-pred} & {\quad \ \scriptsize ImageNet-64: FM/$v$-CS} & {\quad \ \scriptsize ImageNet-64: FM-OT}  & {\quad \ \scriptsize ImageNet-128: FM-OT} \\
         \includegraphics[width=0.25\textwidth]{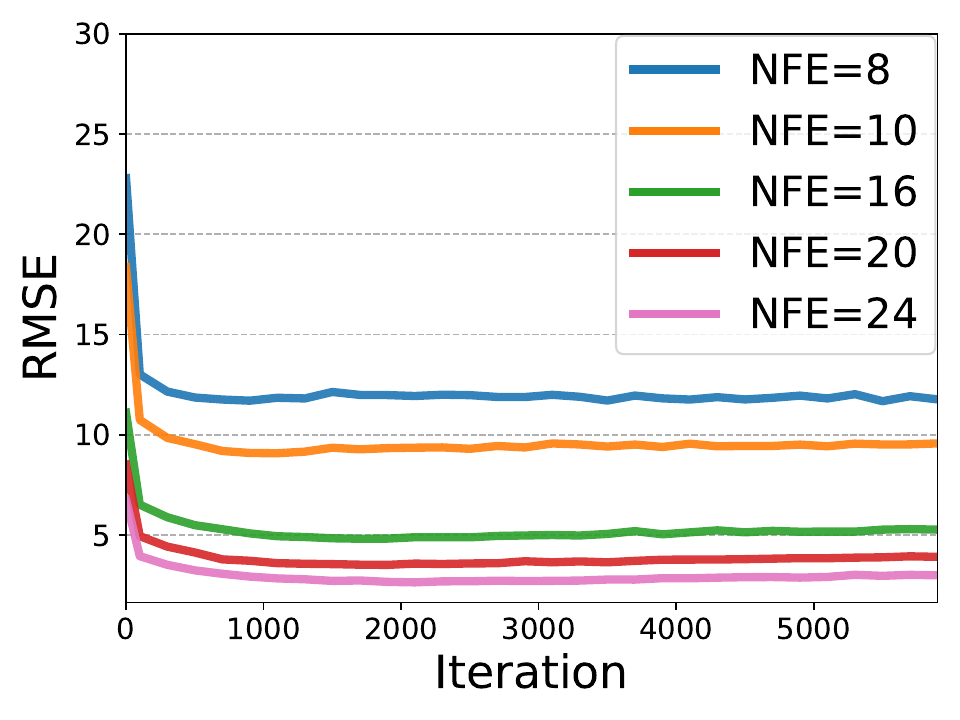}  & \includegraphics[width=0.25\textwidth]{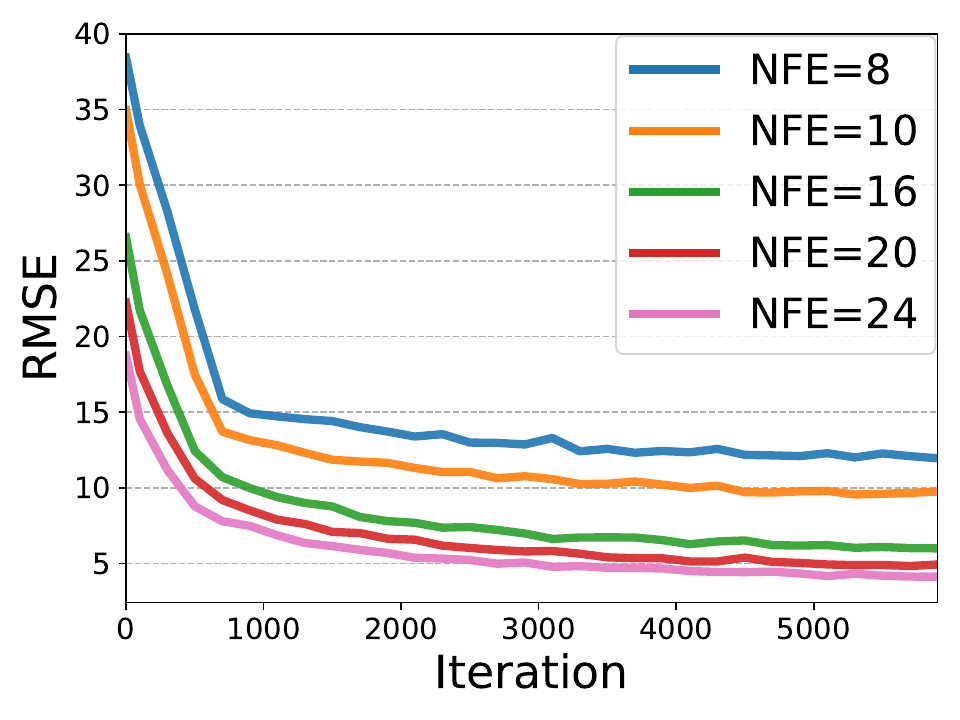} & \includegraphics[width=0.25\textwidth]{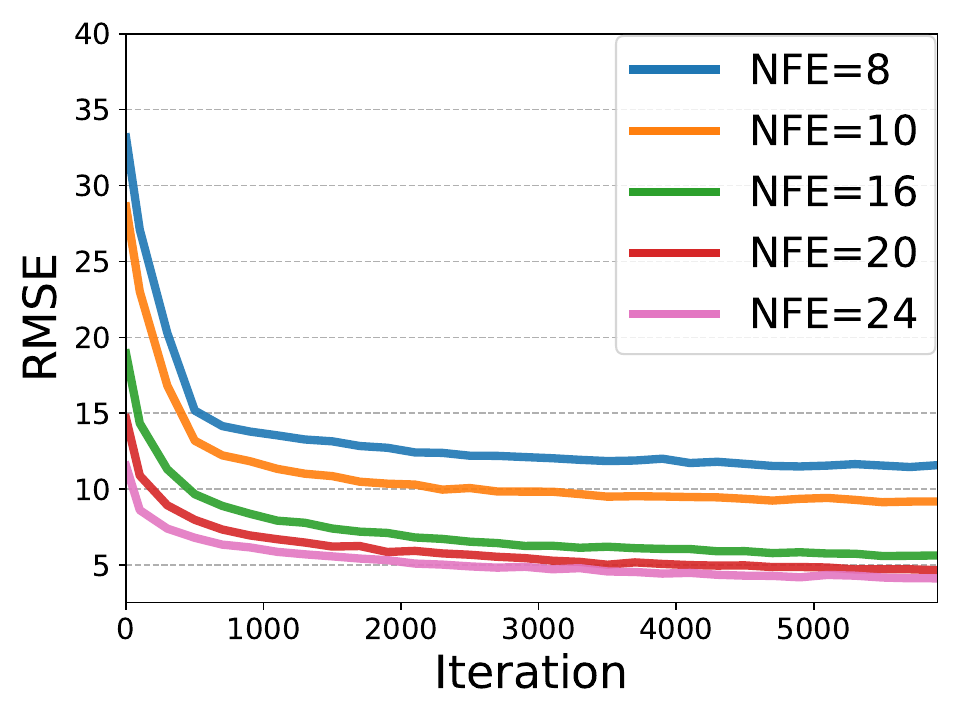} &
         \includegraphics[width=0.25\textwidth]{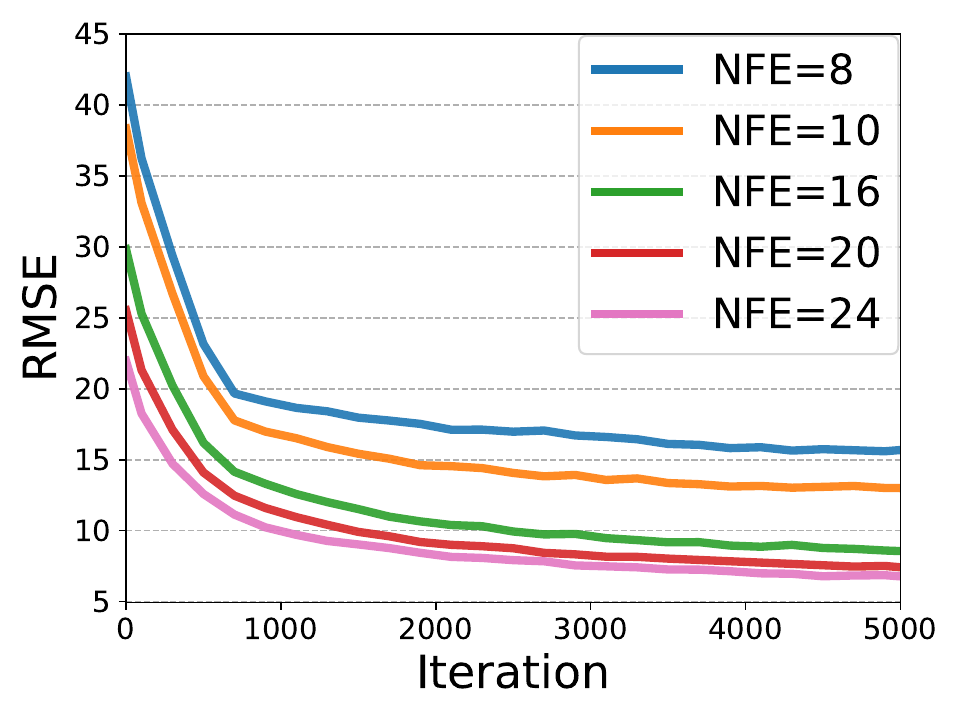}
        \end{tabular}    
    \caption{Validation RMSE vs.~training iterations of Bespoke RK2 solvers on ImageNet-64, and ImageNet-128.}
    \label{fig:rmse_vs_iteration_imagenet}
\end{figure}

%% file: tables/graphs_imagenet_appendix.tex
\begin{figure}[h!]
    \centering
        \begin{tabular}{@{\hspace{0pt}}c@{\hspace{0pt}}c@{\hspace{0pt}}c@{\hspace{0pt}}c@{\hspace{0pt}}}
        {\quad \ \scriptsize ImageNet-64: $\eps$-pred} & {\quad \ \scriptsize ImageNet-64: FM/$v$-CS} & {\quad \ \scriptsize ImageNet-64: FM-OT}  & {\quad \ \scriptsize ImageNet-128: FM-OT} \\
         \includegraphics[width=0.25\textwidth]{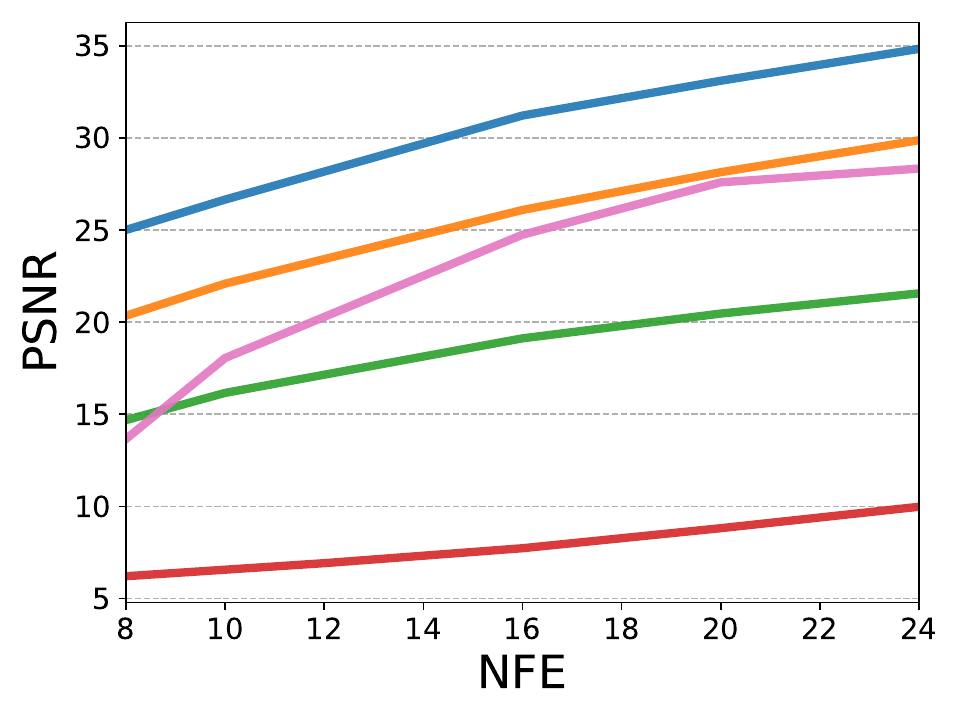}  & \includegraphics[width=0.25\textwidth]{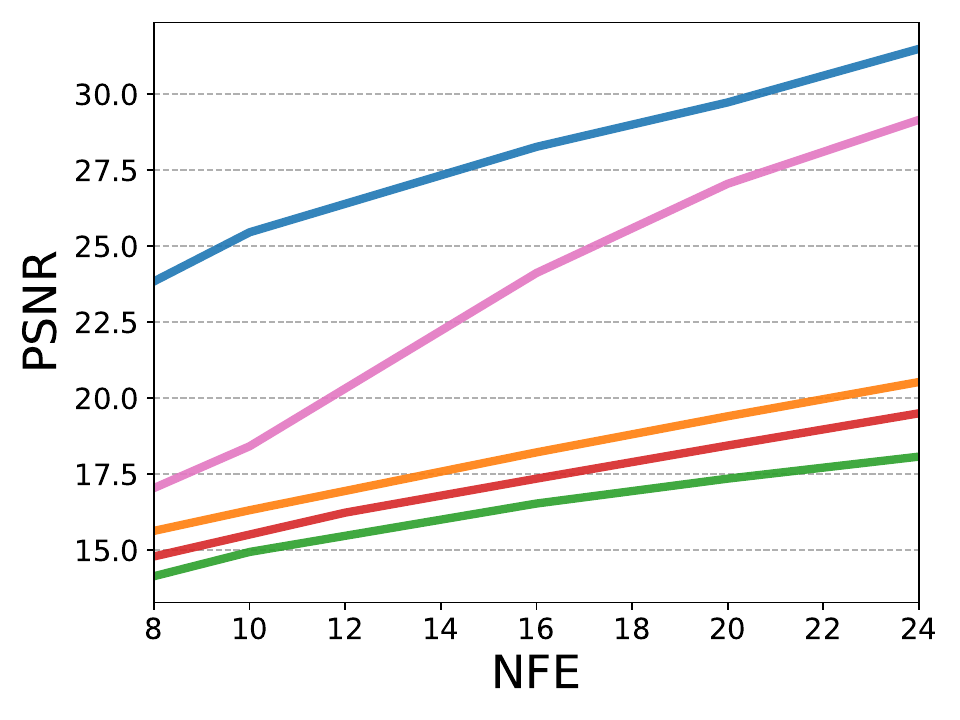} & \includegraphics[width=0.25\textwidth]{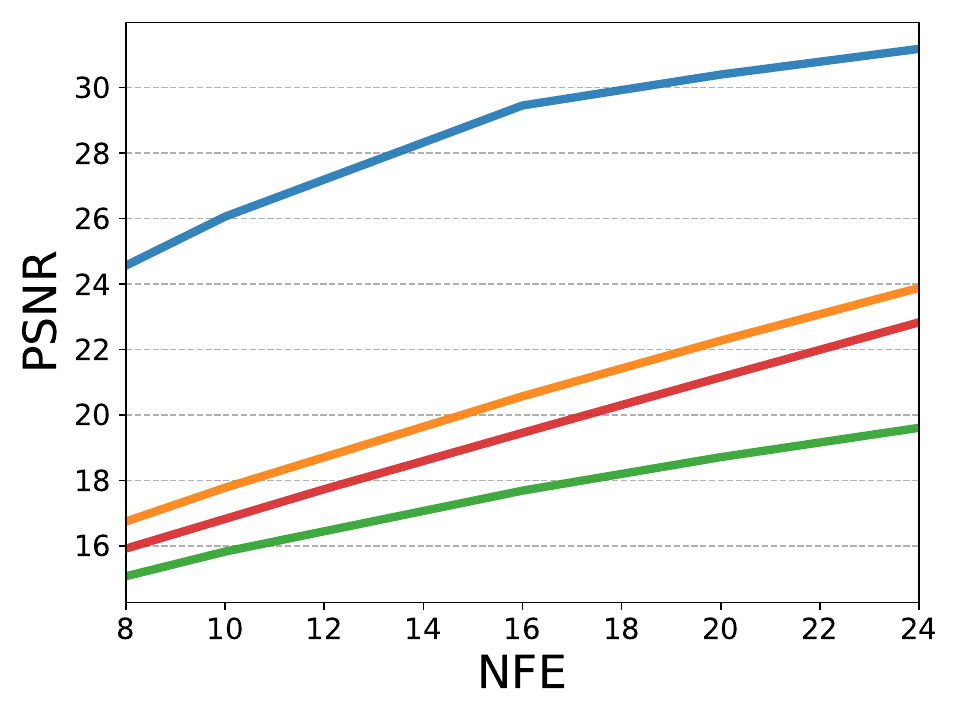} &
         \includegraphics[width=0.25\textwidth]{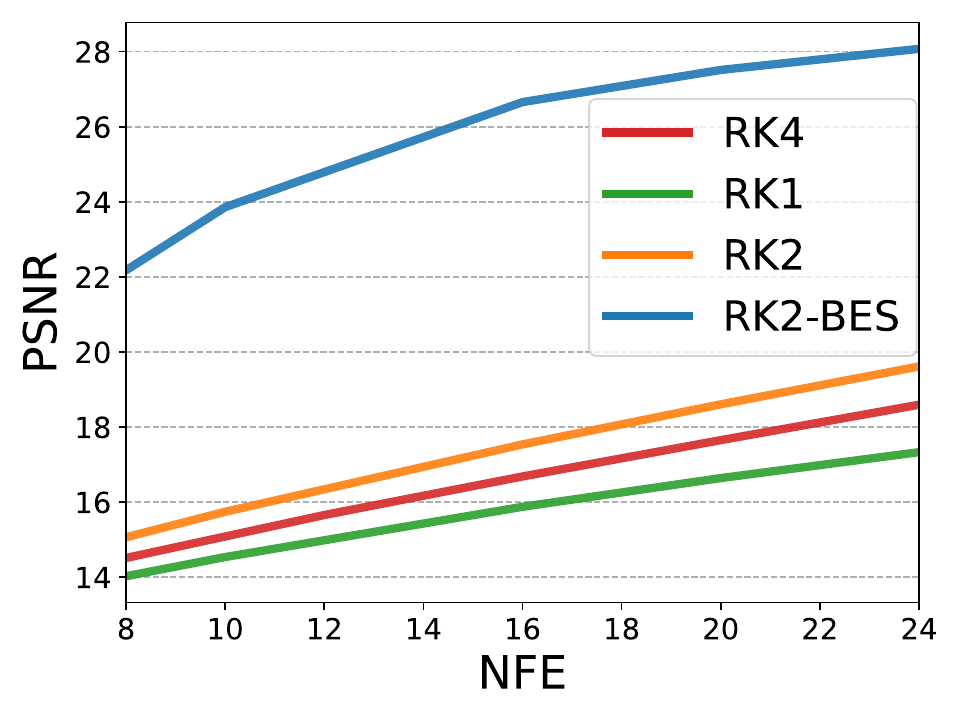}
        \end{tabular}    
    \caption{Bespoke RK2, RK1, RK2, and RK4 solvers on ImageNet-64, and ImageNet-128; PSNR vs.~NFE.  }
    \label{fig:nfe_vs_psnr_imagenet}
\end{figure}

%% file: tables/graphs_afhq_appendix.tex
\begin{figure}[h!]
    \centering
    \begin{tabular}{@{\hspace{0pt}}c@{\hspace{0pt}}c@{\hspace{0pt}}}     \includegraphics[width=0.25\textwidth]{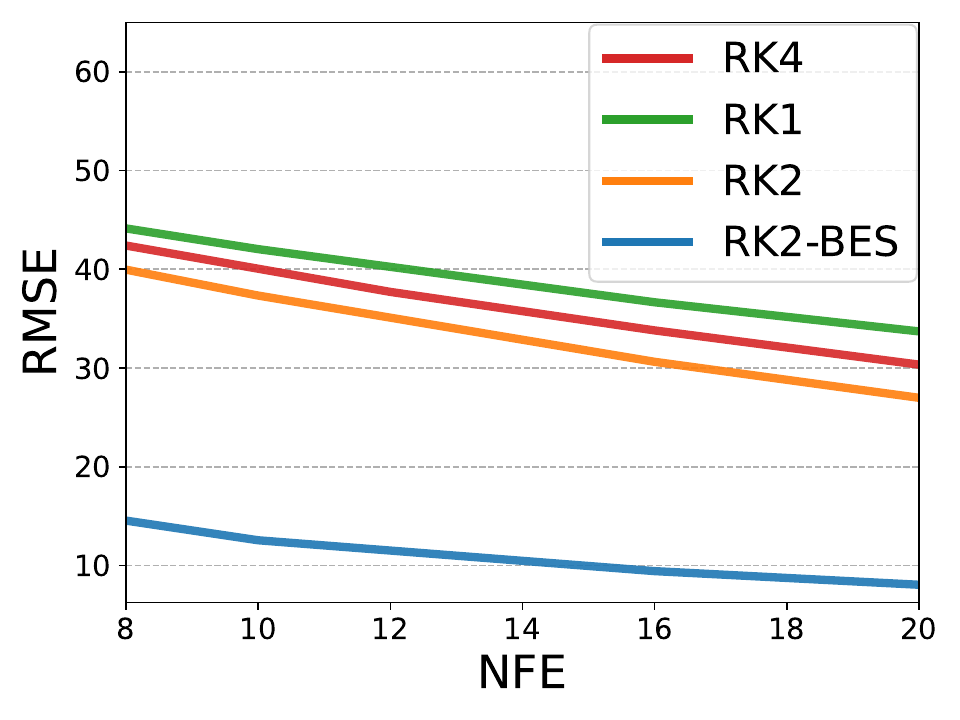} & \includegraphics[width=0.25\textwidth]{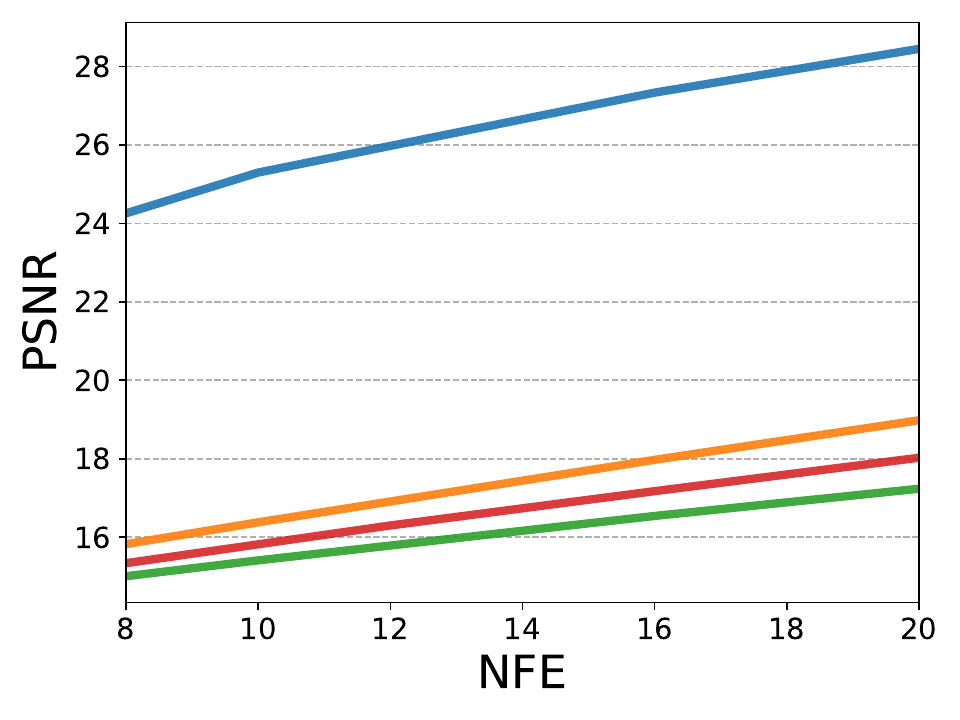} 
\end{tabular}    
    \caption{Bespoke RK2, RK1, RK2, and RK4 solvers on AFHQ-256; PSNR vs.~NFE (left), and RMSE vs.~NFE (right).}
    \label{fig:nfe_vs_metrics_afhq}
\end{figure}

%% file: tables/scale_time_ablation.tex
\begin{figure}[h!]
    \centering
    \begin{tabular}{@{\hspace{0pt}}c@{\hspace{0pt}}c@{\hspace{0pt}}c@{\hspace{0pt}}}
    \includegraphics[width=0.335\textwidth]{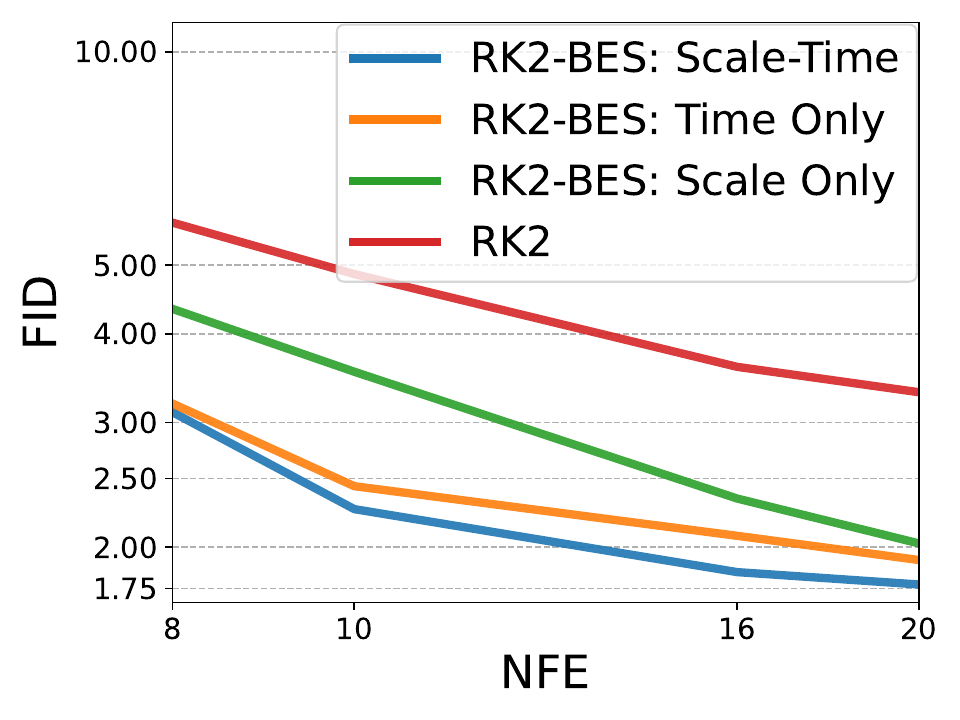} & \includegraphics[width=0.335\textwidth]{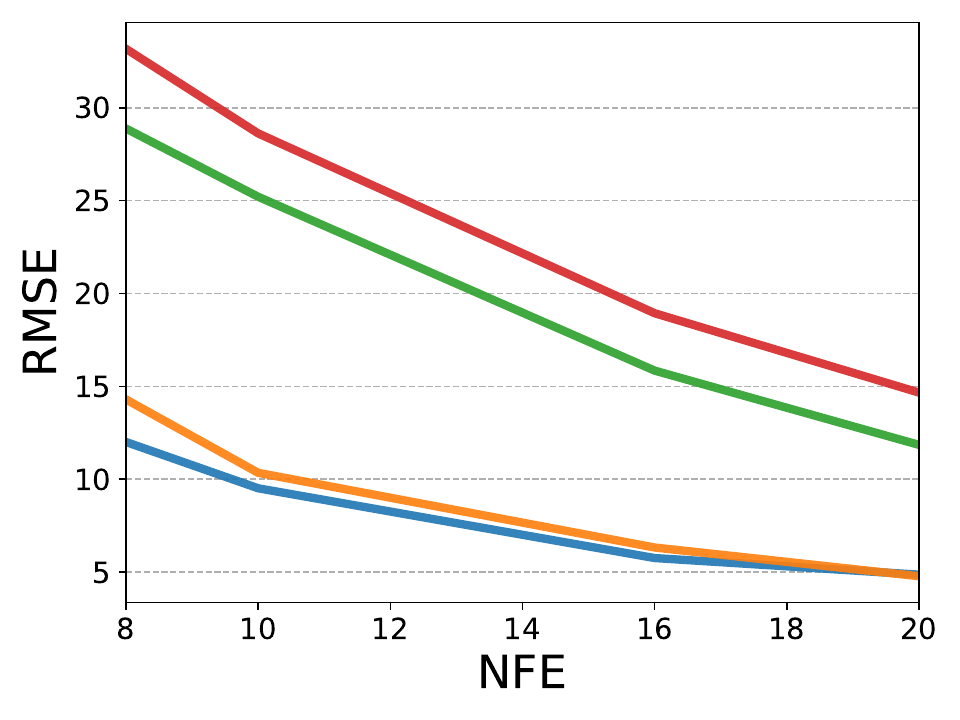} & \includegraphics[width=0.335\textwidth]{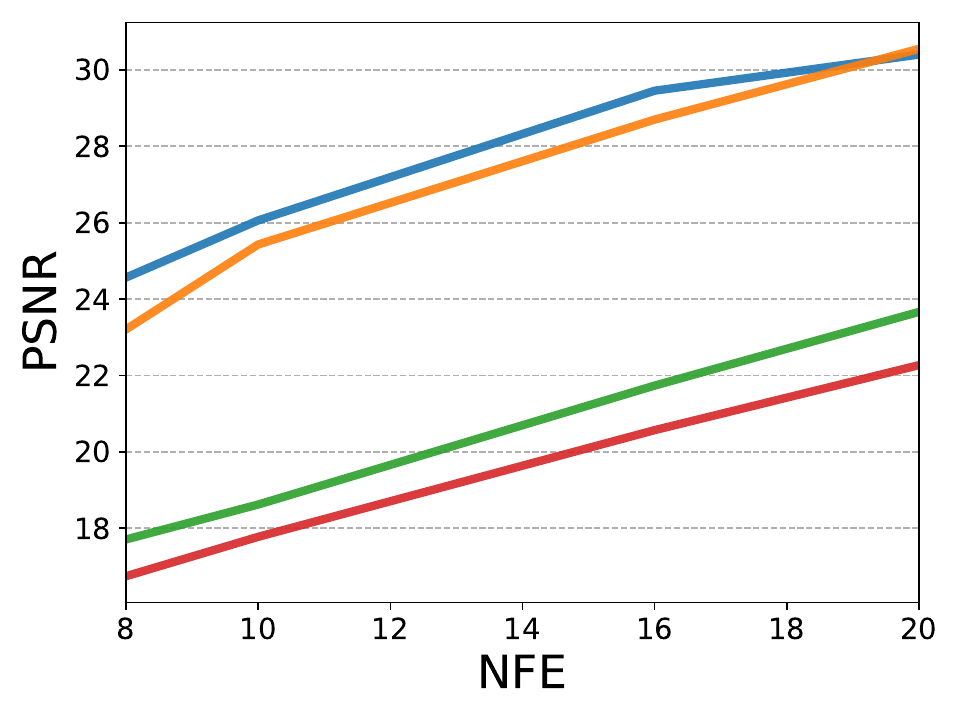} \\
\end{tabular}    
    \caption{Bespoke ablation I: RK2, Bespoke RK2 with full scale-time optimization, time-only optimization (keeping $s_r\equiv 1$ fixed), and scale-only optimization (keeping $t_r=r$ fixed) on FM-OT ImageNet-64: FID vs.~NFE (left), RMSE vs.~NFE (center), and PSNR vs.~NFE (right). Note that most improvement provided by time optimization where scale improves FID for all NFEs, and RMSE for $<20$ NFEs.}
    \label{fig_a:scale_time_ablation}
\end{figure}

%% file: tables/transferred_ablation.tex
\begin{figure}[h!]
    \centering
    \begin{tabular}{@{\hspace{0pt}}c@{\hspace{0pt}}c@{\hspace{0pt}}c@{\hspace{0pt}}}
    \includegraphics[width=0.335\textwidth]{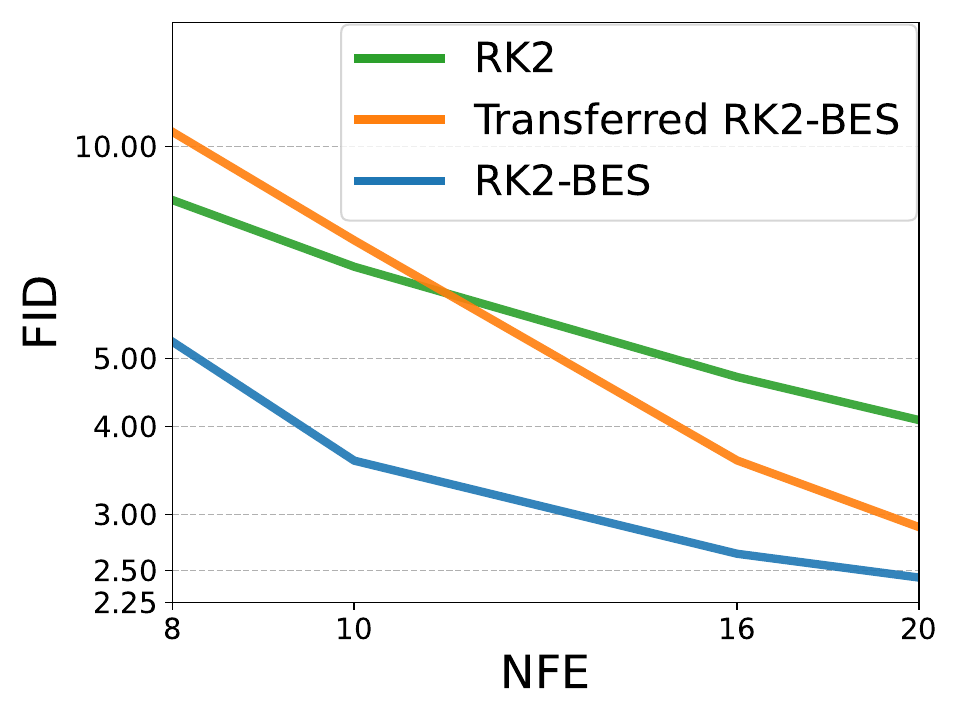} & \includegraphics[width=0.335\textwidth]{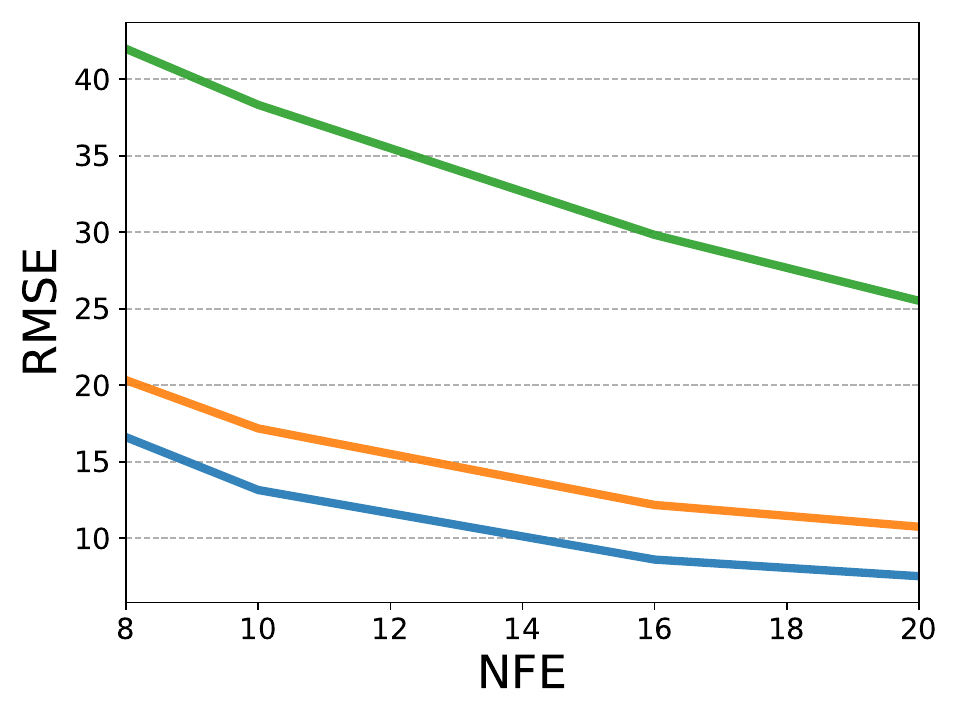} & \includegraphics[width=0.335\textwidth]{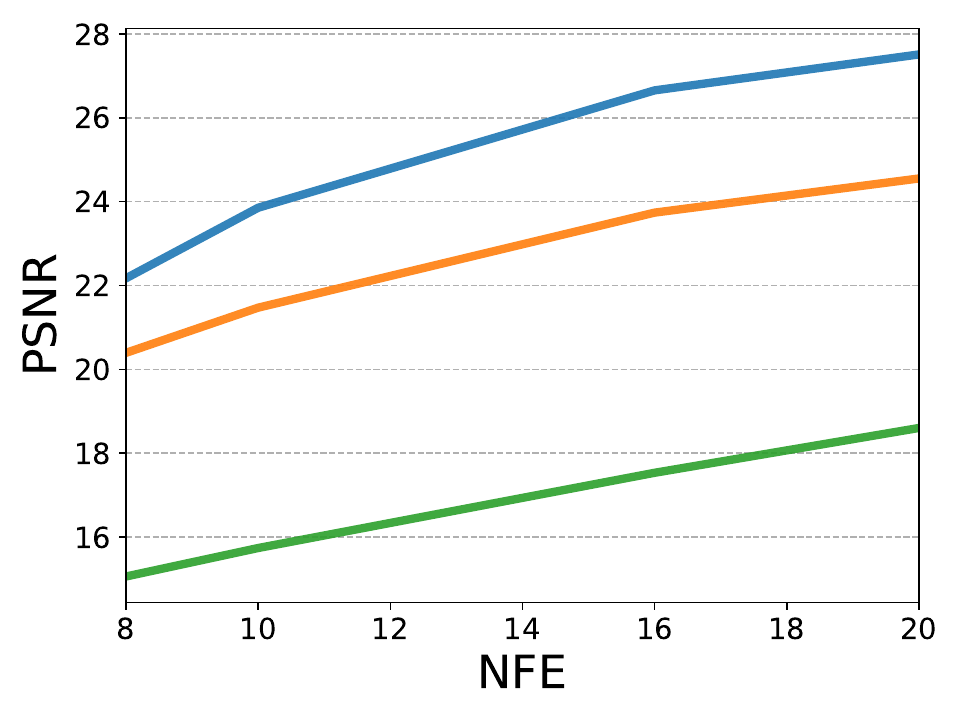} \\
\end{tabular}    
    \caption{Bespoke ablation II: RK2 evaluated on FM-OT ImageNet-128 model, Bespoke RK2 trained and evaluated on FM-OT ImageNet-128 model, and Bespoke RK2 trained on FM-OT ImageNet-64 and evaluated on FM-OT ImageNet-128 model (transferred): FID vs.~NFE (left), RMSE vs.~NFE (center), and PSNR vs.~NFE (right). Note that the transferred Bespoke solver is still inferior to the Bespoke solver but improves considerably RMSE and PSNR compared to the RK2 baseline. In FID the transferred solver improves over the baseline only for NFE=16,20. }
    \label{fig_a:transferred_ablation}
\end{figure}

%% file: tables/training_hyper_parameters.tex
\begin{table}
\centering
\resizebox{\textwidth}{!}{
\begin{tabular}{l c c c c c}
\toprule
 & CIFAR10 & CIFAR10  & ImageNet-64 & ImageNet-128 & AFHQ 256  \\
 & $\eps$-VP & FM-OT;FM/$v$-CS  & $\eps$-VP;FM-OT;FM/$v$-CS & FM-OT & FM-OT  \\
\midrule
Channels & 128 & 128 & 196 & 256 & 256 \\
Depth & 4 & 4  & 3  & 2 & 2 \\
Channels multiple & 2,2,2 & 2,2,2 & 1,2,3,4 & 1,1,2,3,4 & 1,1,2,2,4,4 \\
Heads & 1 & 1 & - & - & -\\
Heads Channels & - & - & 64 & 64 & 64 \\
Attention resolution & 16 & 16 & 32,16,8 & 32,16,8 & 64,32,16 \\
Dropout & 0.1 & 0.3 & 1.0 & 0.0 & 0.0  \\
Effective Batch size & 512 & 512 & 2048 & 2048 & 256  \\
GPUs & 8 & 8 & 64 & 64 & 64  \\
Epochs & 2000 & 3000 & 1600 & 1437 & 862\\ 
Iterations & 200k & 300k &  1M & 900k & 50k\\
Learning Rate & 5e-4 & 1e-4 & 1e-4 & 1e-4 & 1e-4 \\
Learning Rate Scheduler & constant & constant & constant  &Poly Decay &Polyn Decay\\
Warmup Steps & - & - & - & 5k & 5k\\
P-Unconditional & - & - & 0.2 & 0.2 & 0.2\\
\bottomrule
\end{tabular}
}
\caption{Pre-trained models' hyper-parameters. }
\label{tab:training_hyper-params}
\end{table}

%% file: tables/afhq_appendix.tex
\begin{figure}[h!]
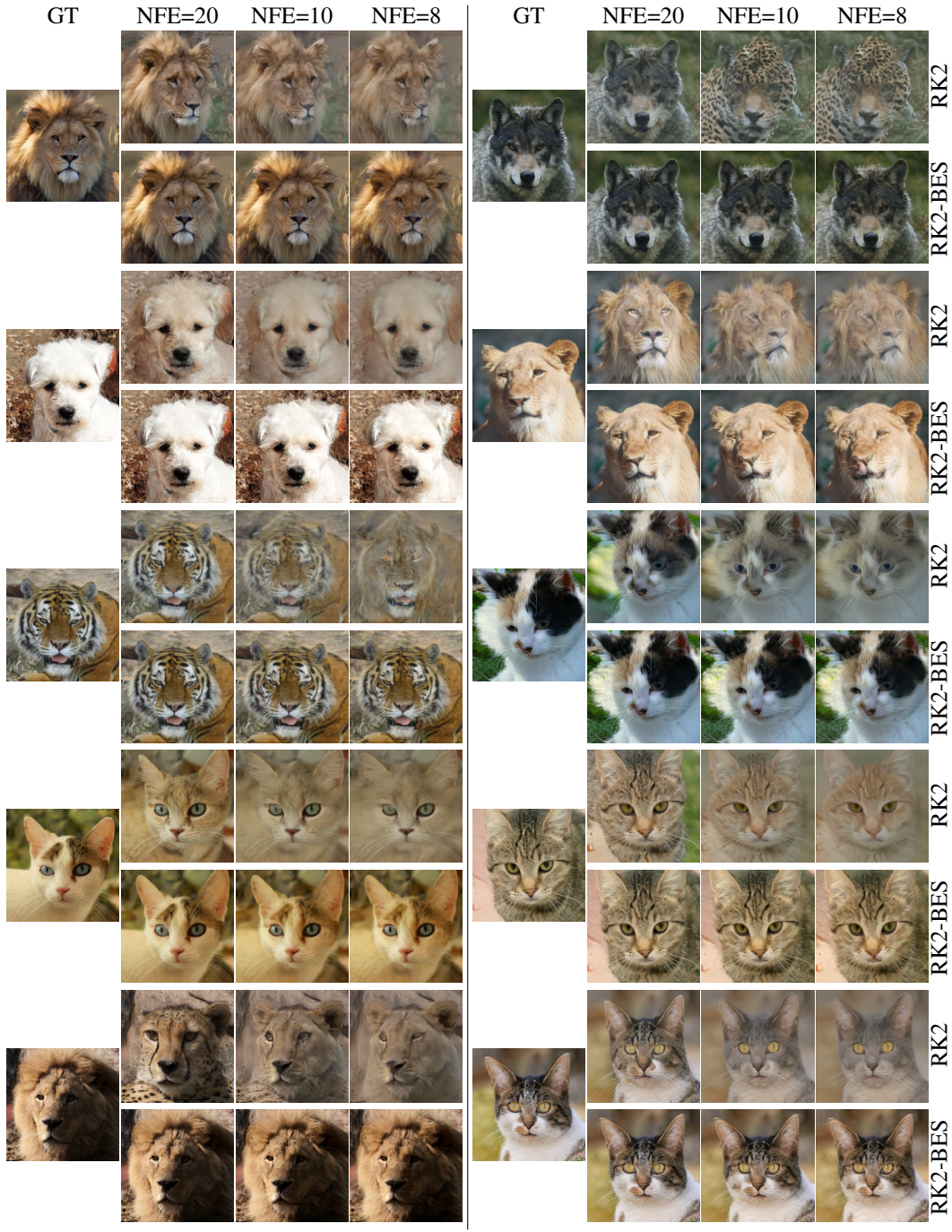

    \centering
 \\

    \caption{Comparison of FM-OT AFHQ-256 GT samples with RK2 and Bespoke-RK2 solvers.}\label{fig:afhq_a}
\end{figure}

%% file: tables/imagenet128_appendix.tex
\begin{figure}[h!]
    \centering
    %
 \\

    \caption{Comparison of FM-OT ImageNet-128 GT samples with RK2 and Bespoke-RK2 solvers.}\label{fig:imagenet_128_a_1}
\end{figure}
\begin{figure}[h]
    \centering
    %
 \\
    \caption{Comparison of FM-OT ImageNet-128 GT samples with RK2 and Bespoke-RK2 solvers.}\label{fig:imagenet_128_a_2}
\end{figure}

%% file: tables/imagenet64_appendix_fm_ot.tex
\begin{figure}[h!]
    \centering
    %
 \\
    
    \caption{Comparison of FM-OT ImageNet-64 GT samples with RK2 and Bespoke-RK2 solvers.}\label{fig:imagenet64_fm_ot_a}
\end{figure}

%% file: tables/imagenet64_appendix_fm_cs.tex
\begin{figure}[h!]
    \centering
    %
 \\
    
    \caption{Comparison of FM/$v$-CS ImageNet-64 GT samples with RK2 and Bespoke-RK2 solvers.}\label{fig:imagenet64_fm_cs_a}
\end{figure}

%% file: tables/imagenet64_appendix_eps_pred.tex
\begin{figure}[h!]
    \centering
    %
 \\
    
    \caption{Comparison of $\eps$-pred ImageNet-64 GT samples with RK2 and Bespoke-RK2 solvers.}\label{fig:imagenet64_eps_pred_a}
\end{figure}

%% file: tables/imagenet64_samples_DPM.tex
\begin{figure}[t]
\centering
%
\vspace{-10pt}
\end{tabular}
\caption{Comparison of $\eps$-VP and FM/$v$-CS ImageNet-64 samples with DPM-2 and bespoke-RK2 solvers.}\label{fig:imagenet-64-dpm2}
\end{figure}